\documentclass{article}

\usepackage[main, final]{neurips_2026}

\usepackage[utf8]{inputenc}
\usepackage[T1]{fontenc}
\usepackage{hyperref}
\usepackage{url}
\usepackage{booktabs}
\usepackage{tabularx}
\usepackage{amsfonts}
\usepackage{amsmath}
\usepackage{amssymb}
\usepackage{mathtools}
\usepackage{nicefrac}
\usepackage{microtype}
\usepackage[table]{xcolor}
\usepackage{graphicx}
\usepackage{multirow}
\usepackage{enumitem}
\usepackage{wrapfig}
\usepackage{amsthm}
\usepackage{tikz}
\usetikzlibrary{arrows.meta, positioning, shapes, calc, decorations.pathreplacing}
\usepackage{tcolorbox}
\tcbuselibrary{theorems}
\tcbuselibrary{breakable}

\graphicspath{{Fig/}}

\newtheorem{theorem}{Theorem}
\newtheorem{definition}{Definition}

\newtheorem{lemma}{Lemma}
\newtheorem{assumption}{Assumption}
\newtheorem{proposition}{Proposition}
\newtheorem{example}{Example}
\newtheorem{remark}{Remark}

\DeclareRobustCommand{\rev}[1]{\textcolor{black}{#1}}

\newtcolorbox{theoryblock}[1][]{
  breakable,
  colback=gray!15!white,
  colframe=gray!20!white,
  boxrule=0pt,
  boxsep=0pt,
  left=2pt,
  right=2pt,
  top=4pt,
  bottom=4pt,
  #1
}

\usepackage[rightcaption]{sidecap}
\sidecaptionvpos{figure}{t}
\usepackage{titletoc}
\titlecontents{section}[2.6em]
  {\vskip 4pt\bfseries}
  {\contentslabel[\thecontentslabel]{2.6em}}
  {\hspace*{-2.6em}}
  {\titlerule*[0.5pc]{.}\contentspage}
\titlecontents{subsection}[5.2em]
  {\vskip 1pt}
  {\contentslabel[\thecontentslabel]{3.2em}}
  {\hspace*{-3.2em}}
  {\titlerule*[0.5pc]{.}\contentspage}
\titlecontents{subsubsection}[8.4em]
  {\vskip 1pt}
  {\contentslabel[\thecontentslabel]{3.8em}}
  {\hspace*{-3.8em}}
  {\titlerule*[0.5pc]{.}\contentspage}
\newcommand\DoToC{%
  \startcontents
  \printcontents{}{1}{\hrulefill\vskip0pt}
  \vskip0pt \noindent\hrulefill
  }

\newcommand{\ours}{\textsc{CausalBind}}

\title{\ours{}: Causal Modeling and Learning for Protein-Molecule Virtual Screening}

\author{%
  \textbf{Loka Li}$^{1}$,~~\textbf{Jin Tian}$^{1}$,~~\textbf{Kun Zhang}$^{1,2}$\\
  $^1$Mohamed bin Zayed University of Artificial Intelligence\\
  $^2$Carnegie Mellon University
}

\begin{document}

\maketitle

\begin{abstract}
Protein-molecule virtual screening is increasingly cast as a problem of representation learning in a shared embedding space. Existing methods rely on dense holistic alignment, entangling invariant binding determinants with nuisance correlations and limiting transfer to new targets. 
It has been noted that binding in protein–molecule systems involves sparse cross-modality interactions: binding is governed by a small contact interface and a few decisive local interactions (e.g., hydrogen bonds, hydrophobic contacts, and salt bridges) rather than the global structures of the protein and molecule.
We hypothesize that uncovering and leveraging sparse interaction patterns is critical for generalization beyond the training data, as these patterns are reusable and expected to improve performance across different scenarios.
In this paper, we aim to identify and leverage sparse interaction patterns, and verify our hypothesis. Since the training data contain only observed binding pairs, we formalize this prior via a V-structure causal model under Heckman-style selection, and establish three theoretical results:
(i) when the compatibility between concepts is bilinear, the latent concepts of interacting proteins and molecules are not identifiable, even if the interaction is sparse;
(ii) with a non-bilinear compatibility function, these concepts and their sparse interaction support are component-wise identifiable under a sparse antichain condition; and
(iii) a low-rank relaxation yields subspace identifiability of the interaction modes, and applies to bilinear compatibility as well.
Inspired by these principles, we propose \ours{} with three implementation variants. Extensive experiments on DUD-E and LIT-PCBA benchmarks show that all variants consistently outperform strong retrieval baselines, with the largest gains on LIT-PCBA early enrichment, and further generalize to target- and scaffold-level out-of-distribution splits. Code and scripts are freely available at \url{https://github.com/lokali/CausalBind}.

\end{abstract}

\section{Introduction}
\label{sec:intro}

Modern drug discovery faces a fundamental challenge: identifying promising drug candidates from libraries containing millions to billions of candidate compounds~\citep{berdigaliyev2020overview}. Virtual screening~\citep{walters2020new}, which prioritizes protein-molecule pairs before costly experimental validation, has become a major bottleneck in this pipeline and a central problem in AI for drug discovery~\citep{jia2026deep}. An overview of the broader pipeline is illustrated in App.~\ref{app:pipeline}.

\begin{figure}[t]
\centering
\includegraphics[width=0.92\linewidth]{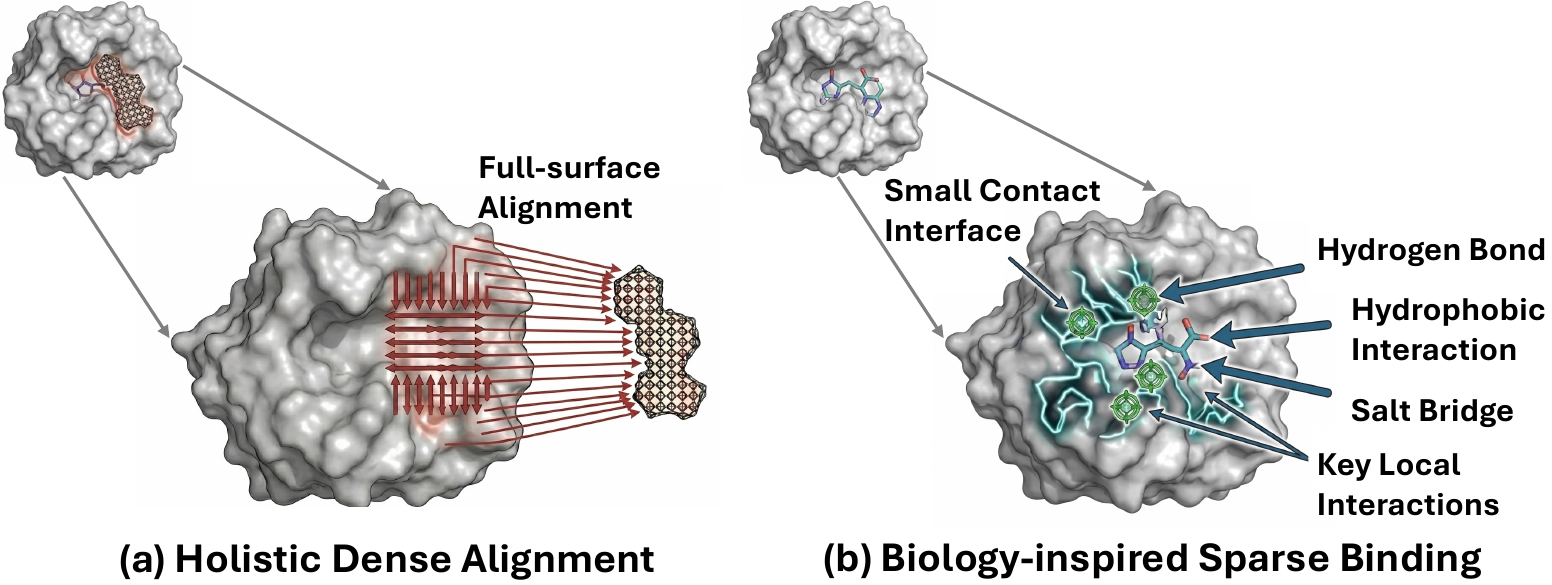}
% \vspace{-2em}
\setlength{\abovecaptionskip}{-0cm}
\caption{\textbf{Dense alignment versus sparse binding.} \textbf{(a) Holistic dense alignment.} Existing retrieval-based virtual screening methods compare every protein surface feature with every molecule feature in the shared embedding space; this ``full-surface'' alignment entangles binding-relevant signals with nuisance correlations. \textbf{(b) Biology-inspired sparse binding.} Protein-molecule binding is governed by a \emph{small contact interface} and a handful of \emph{key local interactions} (e.g., hydrogen bonds, hydrophobic contacts, salt bridges, and $\pi$-stacking), which together account for the majority of the binding affinity.} 
\label{fig:sparse_binding}
\vspace{-1em}
\end{figure}

Existing virtual-screening methods broadly fall into three regimes: physics-based docking systems~\citep{glide2004,vina2010} that score complexes via force fields; structure-aware machine-learning scoring models~\citep{gnina2021,rtmscore2023,genscore2023,equiscore2023,pignet2022,tankbind2022,bigbind2022} that learn to rank protein-molecule complexes from data; and retrieval-based representation learners~\citep{drugclip2023,drughash2025,ligunity2024,hypseek2025,s2drug2026} that embed both modalities into a shared space. The first two reason explicitly about three-dimensional poses and are valuable for mechanistic analysis, but they are typically better suited to rescoring a small candidate set than to ultra-large-scale retrieval. Retrieval-based methods, by contrast, embed both molecules and proteins into a shared embedding space and screen by similarity, scaling naturally to libraries with millions of compounds and now consistently dominating the empirical leaderboard. See App.~\ref{app:related} for extended related work.

Despite this success, retrieval models still learn holistic representations: they align full molecule and protein embeddings without separating binding-relevant determinants from nuisance factors. This reflects a structural mismatch with the biological process, as illustrated in Figure~\ref{fig:sparse_binding}, where the binding decision is governed by a small contact interface and a few decisive local interactions rather than the global structures of the protein and molecule. \rev{A classical example is the notion of ``hot spots'', where a small subset of interfacial contacts contributes the majority of binding free energy~\citep{clackson1995hot,bogan1998anatomy}.}
As a result, holistic models are prone to capturing spurious global correlations rather than the local structural features that actually drive binding~\citep{chen2019hidden,graber2025resolving}, which may hinder generalization to novel protein families and unseen chemical scaffolds. Additional details about biological observations are provided in App.~\ref{app:biological}.

These limitations motivate the core question to be answered in this paper: if protein-molecule binding is driven by a small number of local compatibilities, should one represent virtual-screening in terms of the sparse concept-level interactions, instead of potentially dense holistic alignment, and leverage this representation for better generalization ability? Our results strongly suggest that the answer is yes. To analyze this issue, we first develop a V-structure causal model of binding under Heckman-style selection~\citep{scholkopf2021causal,sedge2026}, and establish three identifiability results for the sparse protein-molecule interaction: dense cross-modal interaction is generally non-identifiable, sparse antichain interaction is component-wise identifiable up to trivial equivalences, and \rev{under a low-rank condition the interaction modes are subspace-identifiable}. \rev{Building on this theory, we propose \ours{} as a flexible structure-constrained representation framework with three different implementations: the main \ours{}$_{\mathrm{SP}}$ uses a \emph{concept extractor} and a \emph{learnable sparse concept-pair mask}, \ours{}$_{\mathrm{LR}}$ uses a low-rank factorization of the same concept-axis mask, and \ours{}$_{\mathrm{EMB}}$ applies the constrained mask after pooling in embedding space.}

Our main technical contributions are threefold.
(i)~Theoretically, we provide an identifiability analysis of representation learning for protein-molecule binding under binder-only selection, showing that \rev{bilinear interactions are not identifiable even when they are sparse, whereas sparse antichain-structured interactions with a non-bilinear compatibility function} are component-wise identifiable, and \rev{low-rank interactions are subspace-identifiable}.
(ii)~Methodologically, \rev{we instantiate these principles through three \ours{} implementations within a concept-based retrieval framework. \ours{}$_{\mathrm{SP}}$ routes binding evidence through a learnable sparse concept-pair mask, \ours{}$_{\mathrm{LR}}$ imposes an explicit low-rank concept-axis factorization, and \ours{}$_{\mathrm{EMB}}$ applies the constraint in pooled embedding space.}
(iii)~Empirically, on DUD-E and LIT-PCBA benchmarks, \rev{all three \ours{} implementations outperform reproduced baseline methods, with the largest relative gains on the property-debiased LIT-PCBA early enrichment.} Additional analyses, including multi-seed ablations, hyperparameter studies, antichain diagnostics of the masks, out-of-distribution evaluation on DEKOIS 2.0, an interpretability case study, and affinity-ranking results, all support the efficacy of our method.

\section{Theoretical Analysis: Sparse Binding under Heckman Selection}
\label{sec:theory}

This section develops a causal perspective on protein-molecule binding under binder-only observation. In $\S$~\ref{subsec:scm}, we model binding as a V-structure and study identifiability~\citep{daunhawer2023identifiability} under Heckman-style selection~\citep{heckman1979}. 
Our theorems in $\S$~\ref{subsec:nonid} proceed in three phases, from non-identifiability to component-wise and subspace identifiability.

\subsection{Causal Model: Binding as a Selected Collider}
\label{subsec:scm}

\begin{wrapfigure}{r}{0.43\textwidth}
\vspace{-2em}
\centering
\begin{tikzpicture}[
    scale=0.78,
    concept/.style={circle, draw=black!80, fill=white, minimum size=0.56cm, font=\scriptsize},
    mconcept/.style={concept, fill=blue!14, draw=blue!60},
    pconcept/.style={concept, fill=orange!18, draw=orange!70!black},
    bindvar/.style={rectangle, draw=red!60!black, fill=red!10, rounded corners=2pt, minimum height=0.36cm, minimum width=0.4cm, font=\tiny},
    offbind/.style={bindvar, fill=gray!18, draw=gray!55},
    finalbind/.style={rectangle, draw=red!75!black, fill=red!14, rounded corners=3pt, minimum height=0.46cm, minimum width=0.72cm, font=\scriptsize\bfseries},
    note/.style={font=\scriptsize, align=center},
    tinylabel/.style={font=\tiny, align=center, text=black!70},
    bindarrow/.style={-{Stealth[length=1.5mm]}, thick, green!60!black},
    blockarrow/.style={-{Stealth[length=1.5mm]}, thick, red!45, dashed},
    finalarrow/.style={-{Stealth[length=1.5mm]}, thick, gray!75},
    dep/.style={dash pattern=on 2pt off 1pt, thick, purple!70},
]

\node[mconcept] (m1) at (-3.0, 0.35) {$c_m^1$};
\node[mconcept] (m2) at (-1.8, 0.35) {$c_m^2$};
\node[mconcept] (m3) at (-0.6, 0.35) {$c_m^3$};

\node[pconcept] (p1) at (0.8, 0.35) {$c_p^1$};
\node[pconcept] (p2) at (2.0, 0.35) {$c_p^2$};
\node[pconcept] (p3) at (3.2, 0.35) {$c_p^3$};

\node[bindvar] (b1) at (-1.15, -1.2) {$1$};
\node[bindvar] (b2) at (0.25, -1.2) {$1$};
\node[offbind] (b3) at (1.65, -1.2) {$0$};
\node[finalbind] (B) at (0.25, -2.45) {$B{=}1$};

\draw[bindarrow] (m1) -- (b1);
\draw[bindarrow] (p1) -- (b1);
\draw[bindarrow] (m2) -- (b2);
\draw[bindarrow] (p2) -- (b2);
\draw[blockarrow] (m3) -- (b3);
\draw[blockarrow] (p3) -- (b3);

\draw[finalarrow] (b1) -- (B);
\draw[finalarrow] (b2) -- (B);
\draw[finalarrow] (b3) -- (B);

\draw[dep] (m1) .. controls (-3.0, 1.35) and (-0.6, 1.35) .. (m3);
\draw[dep] (p1) .. controls (0.8, 1.35) and (3.2, 1.35) .. (p3);
\end{tikzpicture}
\setlength{\abovecaptionskip}{-0cm}
\caption{\rev{Binder-selected sparse-interaction view. Observing only binders ($B{=}1$) induces cross-modal dependence through the collider $B$. Green edges denote active local interactions. Red dashed edges mark blocked interactions. Purple dashed curves mark allowed within-modality dependence.} }
\label{fig:selection_causal}
\vspace{-1em}
\end{wrapfigure}
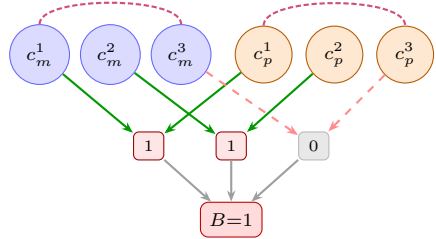
\rev{The causal model combines the biological sparsity motivation above with the binder-only nature of training data. In practice, we predominantly observe pairs with $B{=}1$; conditioning on this collider can make molecule and protein features statistically dependent even when they are marginally independent. Figure~\ref{fig:selection_causal} summarizes this collider structure of sparse interactions.}

Let $x^M\in\mathcal{X}^M$ and $x^P\in\mathcal{X}^P$ denote observed molecule and protein, generated from latent concepts
\begin{equation}
{\color{black}x^M = g^M(\mathbf{c}^M,\varepsilon^M), \qquad x^P = g^P(\mathbf{c}^P,\varepsilon^P),}
\label{eq:body-generators}
\end{equation}
\rev{where $\mathbf{c}^M = (c^M_1,\ldots,c^M_{K_M})$ and $\mathbf{c}^P = (c^P_1,\ldots,c^P_{K_P})$ are modality-specific atomic features, and $\varepsilon^M,\varepsilon^P$ denote modality-specific observation noise or nuisance variation. Here $K_M$ and $K_P$ count the number of concepts, respectively.}  Binding arises through cross-modal interactions.
Motivated by existing literature~\citep{hypseek2025,ligunity2024}, we model binding as a Bernoulli collider:
\begin{equation}
{\color{black}B \sim \mathrm{Bernoulli}\!\left(\sigma\!\left(\sum_{i=1}^{K_M}\sum_{j=1}^{K_P}\mathcal{M}_{ij}\,\phi(c^M_i, c^P_j)\right)\right),}
\quad
{\color{black}x^M \leftarrow \mathbf{c}^M \to B \leftarrow \mathbf{c}^P \to x^P,}
\label{eq:body-vstruct}
\end{equation}
\rev{where $\mathcal{M}\in\mathbb R_{\ge0}^{K_M\times K_P}$ is a (typically sparse) interaction mask; its scale is identifiable only jointly with $\phi$. In the theory, each atomic concept is a scalar latent variable, i.e., $c_i^M,c_j^P\in\mathbb{R}$, and $\phi:\mathbb{R}\times\mathbb{R}\to\mathbb{R}$ maps one molecule concept and one protein concept to a scalar logit-level compatibility score; $\sigma(\cdot)$ is the logistic sigmoid. The nuisance variables affect only the observed modalities and are marginalized in $p(x^M,x^P\mid B{=}1)$; the binding variable depends on the recovered concepts.}  Before selection, $\mathbf{c}^M \perp\!\!\!\perp \mathbf{c}^P$. After conditioning on $B{=}1$, the two modalities become dependent. Here we denote the distribution of the observed training data as: $p^\mathcal{D}(x^M, x^P) \;=\; p(x^M, x^P \mid B=1)$, 
which generally induces dependence between $x^M$ and $x^P$. Furthermore, by Bayes' rule, we have
\begin{equation}
{\color{black}p^\mathcal{D}(\mathbf{c}^M,\mathbf{c}^P) \;\propto\; p(\mathbf{c}^M)\,p(\mathbf{c}^P)\,\sigma\!\Big(\textstyle\sum_{i=1}^{K_M}\sum_{j=1}^{K_P}\mathcal{M}_{ij}\phi(c^M_i, c^P_j)\Big).}
\label{eq:body-select-factor}
\end{equation}
This induced dependence is the only signal available for recovering $(\mathcal{M},\phi)$.
\rev{Below we give the definition of identifiability used in this paper, followed by some common assumptions.}

\begin{definition}[Identifiability]\label{def:identifiability}
\rev{Identifiability is concerned with whether the latent concepts and interaction structure are uniquely determined by the observed binder-selected distribution $p^\mathcal{D}(x^M,x^P)$, modulo the unavoidable equivalences of the chosen representation.}  \rev{In the theory, the concept dimensions $K_M$ and $K_P$ are treated as given model dimensions.}  \rev{We use two levels of identifiability.}
\begin{enumerate}[leftmargin=1.6em, itemsep=2pt]
    \item \rev{\textbf{Component-wise identifiability.} The concepts and sparse interaction support are component-wise identifiable if any two parameterizations $(g^M,g^P,\mathcal{M},\phi)$ and $(\tilde g^M,\tilde g^P,\tilde{\mathcal{M}},\tilde\phi)$ inducing the same $p^\mathcal{D}(x^M,x^P)$ differ only by separate molecule-side and protein-side concept-index permutations, component-wise invertible transformations of individual concepts, and the trivial joint rescaling $(\mathcal M,\phi)\mapsto(a\mathcal M,a^{-1}\phi)$ for $a>0$.}
    \item \rev{\textbf{Subspace identifiability.} Suppose the interaction kernel has a minimal rank-$r$ factorization $S(\mathbf c^M,\mathbf c^P)=\mathbf u(\mathbf c^M)^\top\mathbf v(\mathbf c^P)$. The interaction is subspace-identifiable if every other minimal rank-$r$ factorization $\tilde{\mathbf u}^\top\tilde{\mathbf v}$ inducing the same $p^\mathcal D(x^M,x^P)$ satisfies}
    \[
    {\color{black}
    \mathrm{span}\{\tilde u_1,\ldots,\tilde u_r\}
    =
    \mathrm{span}\{u_1,\ldots,u_r\},\qquad
    \mathrm{span}\{\tilde v_1,\ldots,\tilde v_r\}
    =
    \mathrm{span}\{v_1,\ldots,v_r\}.
    }
    \]
    \rev{Equivalently, the ordered mode basis may change by an invertible transformation in $\mathrm{GL}(r)$, but the molecule-side and protein-side interaction subspaces are fixed.}
\end{enumerate}
\end{definition}

\begin{assumption}[Concept recoverability]\label{ass:concept-recoverability}
\rev{The latent concepts $\mathbf{c}^M$ and $\mathbf{c}^P$ can be written as functions of the corresponding observations $\mathbf{x}^M$ and $\mathbf{x}^P$, respectively.}
\end{assumption}

\begin{assumption}[Prior independence]\label{ass:marginal-independence}
Before conditioning on $B$, the two modalities are independent, i.e., $p(\mathbf{c}^M,\mathbf{c}^P)=p(\mathbf{c}^M)p(\mathbf{c}^P)$.
\end{assumption}

\begin{assumption}[Positivity / coverage]\label{ass:positivity-coverage}
We have $P(B{=}1)>0$, \rev{the supports of $p(\mathbf c^M)$ and $p(\mathbf c^P)$ are open sets,} and every atomic concept appears with positive probability under the binder-selected distribution $p^\mathcal{D}$.
\end{assumption}

\begin{assumption}[\rev{Log-linear selection}]\label{ass:loglinear}
\rev{On the data support, the binding probability is an exponential tilt of the binding score: $P(B{=}1\mid\mathbf c^M,\mathbf c^P)=\kappa\exp(S(\mathbf c^M,\mathbf c^P))$ for a constant $\kappa>0$, where $S(\mathbf c^M,\mathbf c^P):=\sum_{i,j}\mathcal M_{ij}\phi(c_i^M,c_j^P)$.}
\end{assumption}

In Definition~\ref{def:identifiability}, component-wise identifiability is stronger because it recovers each concept, whereas subspace identifiability recovers only a low-dimensional subspace spanned by several concepts.
Assumptions~\ref{ass:concept-recoverability}--\ref{ass:positivity-coverage} are shared by \rev{all} theorems.  They ensure that (i) concept information is preserved from the observed modalities, (ii) cross-modal dependence arises only through selection, and (iii) sufficient coverage is available for recovery, respectively. \rev{Assumption~\ref{ass:loglinear}, used by the two positive results, is the low-hit-rate form of the logistic model, under which our contrastive objective is well specified~\citep{infonce2018}.} Each of the four assumptions is discussed in more detail in App.~\ref{app:theory:assumptions}.

\subsection{Identifiability Results}
\label{subsec:nonid}

We next characterize, \rev{in a progressively structured way}, when the latent interaction mechanism is identifiable from binder-selected data. First, binder-only selection by itself is insufficient: without suitable structural constraints, interacting protein and molecule concepts are not identifiable. Second, such concepts and sparse interactions are component-wise identifiable under proper structural sparsity conditions. Third, \rev{low-rank interaction structure yields subspace identifiability when only low-dimensional interaction subspaces can be determined. We give the three formal statements below.}

\begin{theoryblock}
\begin{theorem}[\rev{Non-identifiability of bilinear interactions}]\label{thm:main:nonid}
\rev{Under the causal model in Section~\ref{subsec:scm}, suppose Assumptions~\ref{ass:concept-recoverability}--\ref{ass:positivity-coverage} hold and the compatibility function is bilinear, $\phi(x,y)=xy$. Then for any training distribution $p^\mathcal{D}(x^M,x^P)$ there exist infinitely many distinct tuples $(g^M,g^P,\mathcal{M},\phi)$ inducing the same $p^\mathcal{D}(x^M,x^P)$. Moreover, if $K_M\ge2$ and $\mathcal M\neq0$, some of these tuples keep the same mask, $\widetilde{\mathcal M}=\mathcal M$, while their concepts mix at least two of the original concepts. Hence the concepts are \emph{non-identifiable} here, and constraining the mask, e.g.\ to be sparse with antichain supports, does not remove the ambiguity.}
\end{theorem}
\end{theoryblock}

\paragraph{Remark 1.}
\rev{Theorem~\ref{thm:main:nonid} gives the baseline impossibility result: conditioning on $B{=}1$ induces cross-modal concept dependence, but this dependence alone does not determine which concepts or mask generated it. With a bilinear $\phi$, infinitely many parameterizations explain the same distribution, and the second part shows this cannot be removed by constraining the mask. Theorems~\ref{thm:main:ident} and~\ref{thm:main:lowrank} specify what additional structure makes identification possible. The complete proof is provided in App.~\ref{app:theory:non-id}.}

\begin{theoryblock}
\begin{theorem}[Component-wise identifiability under sparse interaction]\label{thm:main:ident}
\rev{Suppose Assumptions~\ref{ass:concept-recoverability}--\ref{ass:loglinear} hold, and let $(\tilde g^M,\tilde g^P,\widetilde{\mathcal M},\tilde\phi)$, with $\tilde\phi\in C^2$, be any other such tuple inducing the same $p^\mathcal{D}(x^M,x^P)$, with concepts $\tilde{\mathbf c}^M,\tilde{\mathbf c}^P$. Suppose further:}
\begin{enumerate}[label=\roman*, leftmargin=1.8em, itemsep=2pt]
    \item [2.i]\underline{(Sparse antichain interaction)} \rev{Letting}
    \[
    {\color{black}
    \begin{aligned}
    R_i&:=\mathrm{supp}(\mathcal{M}_{i,\cdot})=\{j\in[K_P]:\mathcal{M}_{ij}\neq 0\},\\
    C_j&:=\mathrm{supp}(\mathcal{M}_{\cdot,j})=\{i\in[K_M]:\mathcal{M}_{ij}\neq 0\},
    \end{aligned}
    }
    \]
    \rev{every row support and every column support is nonempty, and the row supports $\{R_i\}_{i=1}^{K_M}$ and column supports $\{C_j\}_{j=1}^{K_P}$ form antichains:}
    \[
    {\color{black}
    R_i\not\subseteq R_{i'}\quad \forall i\neq i',\qquad
    C_j\not\subseteq C_{j'}\quad \forall j\neq j'.
    }
    \]
    \item [2.ii] \underline{\rev{(Non-bilinear compatibility)}} \rev{$\phi$ is real-analytic on $\mathbb R^2$ and $\partial_1^2\partial_2^2\phi\not\equiv0$; equivalently, the cross-derivative $\partial_1\partial_2\phi(x,y)$ is not of the form $\alpha(x)+\beta(y)$.}
    \item [2.iii] \underline{\rev{(Linear ambiguity and sparsity minimality)}} \rev{The two concept systems are related by invertible affine maps, $\tilde{\mathbf c}^M=T_M\mathbf c^M+\mathbf t_M$ and $\tilde{\mathbf c}^P=T_P\mathbf c^P+\mathbf t_P$ with $T_M\in\mathrm{GL}(K_M)$, $T_P\in\mathrm{GL}(K_P)$; and $\|\widetilde{\mathcal M}\|_0\le\|\mathcal M\|_0$.}
\end{enumerate}
\rev{Then there exist permutations $\pi$ of $[K_M]$ and $\varsigma$ of $[K_P]$ such that each $\tilde c^M_{\pi(i)}$ is an affine function of $c^M_i$ alone, each $\tilde c^P_{\varsigma(j)}$ is an affine function of $c^P_j$ alone, and $\mathrm{supp}(\widetilde{\mathcal M})=\{(\pi(i),\varsigma(j)):(i,j)\in\mathrm{supp}(\mathcal M)\}$; the concepts and sparse support of $\mathcal M$ are thus component-wise identifiable (Definition~\ref{def:identifiability}).}
\end{theorem}
\end{theoryblock}

\paragraph{Remark 2.}
\rev{Theorem~\ref{thm:main:ident} resolves that ambiguity, and shows which ingredient does the work. Condition~2.iii restricts the competitors to linear reparameterizations, what contrastive learning already provides~\citep{roeder2021linear}. Condition~2.ii makes $\mathcal M_{ij}\partial_1\partial_2\phi$ vary over the support, so matching the two distributions yields a \emph{family} of matrix equations rather than the single one a bilinear $\phi$ gives. Condition~2.i gives each concept a distinct combinatorial footprint, and the minimality clause is the population counterpart of the sparsity penalty. App.~\ref{app:theory:main} has the full proof and the counterexamples.}

\begin{theoryblock}
\begin{theorem}[Subspace identifiability under low-rank interaction]\label{thm:main:lowrank}
Under the causal model above, if the observed binder-selected distribution matches the true joint distribution $p^\mathcal{D}(x^M,x^P)$, and Assumptions~\ref{ass:concept-recoverability}--\rev{\ref{ass:loglinear}} hold, and:
\begin{enumerate}[label=\roman*, leftmargin=1.8em, itemsep=2pt]
    \item [3.i]\underline{(Low-rank interaction kernel)} \rev{Let $S$ be the binding score of Assumption~\ref{ass:loglinear}, fix reference values $\mathbf a_0,\mathbf b_0$ in the supports of $\mathbf c^M,\mathbf c^P$, and let $S^\circ(\mathbf c^M,\mathbf c^P):=S(\mathbf c^M,\mathbf c^P)-S(\mathbf c^M,\mathbf b_0)-S(\mathbf a_0,\mathbf c^P)+S(\mathbf a_0,\mathbf b_0)$ be the centred kernel.}
    \rev{There exists a minimal rank $r<\min\{K_M,K_P\}$ and linearly independent mode functions $u_1,\ldots,u_r\in L^2(p^\mathcal D(\mathbf c^M))$ and $v_1,\ldots,v_r\in L^2(p^\mathcal D(\mathbf c^P))$ such that}
    \[
    {\color{black}
    S^\circ(\mathbf c^M,\mathbf c^P)
    =\sum_{\ell=1}^{r}u_\ell(\mathbf c^M)v_\ell(\mathbf c^P)
    =\mathbf u(\mathbf c^M)^\top \mathbf v(\mathbf c^P),
    }
    \]
    \rev{where $\mathbf u=(u_1,\ldots,u_r)^\top$ and $\mathbf v=(v_1,\ldots,v_r)^\top$, and no decomposition with fewer than $r$ separable terms exists.} 
    \item [3.ii]\underline{(Low-rank sufficient variability)} \rev{There exist probe values $\mathbf a_1,\ldots,\mathbf a_{R_M}\in\mathrm{int}(\mathrm{supp}(\mathbf c^M))$ and $\mathbf b_1,\ldots,\mathbf b_{R_P}\in\mathrm{int}(\mathrm{supp}(\mathbf c^P))$, with $R_M,R_P\ge r$, such that the mode-evaluation matrices}
    \[
    {\color{black}
    \mathbf U_A
    :=
    \begin{bmatrix}
    u_1(\mathbf a_1)&\cdots&u_r(\mathbf a_1)\\[-1mm]
    \vdots&\ddots&\vdots\\[-1mm]
    u_1(\mathbf a_{R_M})&\cdots&u_r(\mathbf a_{R_M})
    \end{bmatrix},
    \qquad
    \mathbf V_B
    :=
    \begin{bmatrix}
    v_1(\mathbf b_1)&\cdots&v_r(\mathbf b_1)\\[-1mm]
    \vdots&\ddots&\vdots\\[-1mm]
    v_1(\mathbf b_{R_P})&\cdots&v_r(\mathbf b_{R_P})
    \end{bmatrix}
    }
    \]
    \[
    {\color{black}
    \operatorname{rank}(\mathbf U_A)=\operatorname{rank}(\mathbf V_B)=r.
    }
    \]
    \rev{Thus both modality-side mode families can be independently excited across the selected concept support.} 
    \item[] \rev{Then the interaction subspaces}
    \[
    {\color{black}
    \mathcal U_M:=\mathrm{span}\{u_1,\ldots,u_r\}\subset L^2(p^\mathcal D(\mathbf c^M)),
    \qquad
    \mathcal V_P:=\mathrm{span}\{v_1,\ldots,v_r\}\subset L^2(p^\mathcal D(\mathbf c^P))
    }
    \]
    \rev{are subspace-identifiable from $p^\mathcal{D}(x^M,x^P)$ in the sense of Definition~\ref{def:identifiability}.}  
\end{enumerate}
\end{theorem}
\end{theoryblock}

\paragraph{Remark 3.}
\rev{Under Condition~3.i the interaction kernel is low-rank, so binding is determined by rank-$r$ interaction subspaces rather than by separately identifiable concept pairs. The kernel is centred because the selected distribution determines $S$ only up to additive functions of a single modality. Unlike Theorem~\ref{thm:main:ident}, this result does \emph{not} need a non-bilinear $\phi$: for $\phi(x,y)=xy$ the kernel has rank $\mathrm{rank}(\mathcal M)$, so a low-rank mask suffices, the regime our implementation occupies. Proof in App.~\ref{app:theory:lowrank}.}

\paragraph{From theory to implementation.}
\rev{Identifiability matters here because these interaction patterns are reusable and expected to improve generalizability across different scenarios.} These results pin down the inductive bias that our method should approximate: (i) concept encoders rich enough to preserve local binding factors, (ii) \rev{sparse interaction support or low-rank interaction subspaces}, and (iii) exploitation of the collider-induced dependence under $B{=}1$. The next section will translate the proposed theoretical model into concrete neural architectures, with different implementations. \rev{The causal formulation enters the implementation through the choice and regularization of an identifiable representation class, rather than through a separate intervention-, propensity-, or counterfactual-based training loss: \ours{} retains the contrastive/ranking supervision of retrieval-based methods, and replaces dense holistic interaction with the sparse-antichain or low-rank parameterizations identified above, together with the corresponding structural regularizer. Accordingly, the empirical gains arise from this theory-derived structural inductive bias, while the causal analysis establishes when these interaction classes are identifiable under binder-only selection; we do not claim that \ours{} estimates interventional effects of individual concepts.}
\rev{Since our $\phi$ is an inner product, hence bilinear, the guarantee that applies is Theorem~\ref{thm:main:lowrank} rather than Theorem~\ref{thm:main:ident} (App.~\ref{app:theory:gap}).}

\section{\ours{}: Implementation}
\label{sec:method}

We implement the sparse causal principle of $\S$~\ref{sec:theory}. \rev{We develop three variants: \ours{}$_{\mathrm{SP}}$, \ours{}$_{\mathrm{LR}}$, and \ours{}$_{\mathrm{EMB}}$. Due to space, we focus on \ours{}$_{\mathrm{SP}}$, the sparse concept-pair implementation; supplements for \ours{}$_{\mathrm{SP}}$, \ours{}$_{\mathrm{LR}}$, and \ours{}$_{\mathrm{EMB}}$ are given in App.~\ref{app:impl_sp}, App.~\ref{app:impl_lr}, and App.~\ref{app:impl_emb}.} Inspired by HypSeek~\citep{hypseek2025}, we use three modalities: molecule structure, pocket structure, and protein sequence, where the sequence branch provides complementary global context beyond local pocket geometry. Instead of aligning holistic embeddings, we first decompose each modality into atomic concepts, and then route protein-molecule matching through sparse concept-level interaction operators (see Figure~\ref{fig:causal_graph}).

\paragraph{Notation.}
\rev{Let $\mathbf{H}_{\mathrm{mol}}\in\mathbb{R}^{N_{\mathrm{mol}}\times d_{\mathrm{mol}}}$, $\mathbf{H}_{\mathrm{poc}}\in\mathbb{R}^{N_{\mathrm{poc}}\times d_{\mathrm{poc}}}$, and $\mathbf{H}_{\mathrm{seq}}\in\mathbb{R}^{N_{\mathrm{seq}}\times d_{\mathrm{seq}}}$ denote token representations from Uni-Mol~\citep{zhou2023uni} (molecule and pocket) and ESM2~\citep{esm2} (sequence). These project to a per-concept dimension $d_c$. The theory in $\S$~\ref{sec:theory} uses scalar concepts, corresponding to the special case $d_c=1$. In implementation, we allow vector-valued concept tokens with $d_c\ge 1$: for each modality $v\in\{\mathrm{mol},\mathrm{poc},\mathrm{seq}\}$, $\mathbf{C}_v\in\mathbb{R}^{K\times d_c}$ and the $i$-th row $\mathbf C_v[i,:]$ represents one concept group. The theory allows different concept counts across modalities ($K_M$ and $K_P$); for simplicity, we use a common concept budget $K$ for molecule, pocket, and sequence streams.}
\rev{To align notation with the theoretical mask $\mathcal M$, we denote the two implementation masks by $\mathcal M_{\mathrm{mol},\mathrm{poc}}$ and $\mathcal M_{\mathrm{mol},\mathrm{seq}}$. Their raw learnable parameters are $\widehat{\mathcal M}_{\mathrm{mol},\mathrm{poc}}$ and $\widehat{\mathcal M}_{\mathrm{mol},\mathrm{seq}}$, and the effective sparse masks $\mathcal M_{\mathrm{mol},\mathrm{poc}},\mathcal M_{\mathrm{mol},\mathrm{seq}}\in\mathbb{R}^{K\times K}$ are obtained after activation and thresholding.}

\begin{SCfigure}[0.8][!t]
\centering
\begin{tikzpicture}[xscale=0.72, yscale=0.84,
    inputbox/.style={trapezium, trapezium left angle=70, trapezium right angle=110, draw=blue!50, fill=blue!8, minimum height=0.5cm, minimum width=1.0cm, align=center, font=\scriptsize},
    embedbox/.style={trapezium, trapezium left angle=70, trapezium right angle=110, draw=purple!50, fill=purple!10, minimum height=0.4cm, minimum width=0.8cm, align=center, font=\scriptsize},
    encoderbox/.style={rectangle, draw=gray!60, fill=gray!15, rounded corners=2pt, minimum height=0.6cm, minimum width=1.2cm, align=center, font=\scriptsize},
    conceptorbox/.style={rectangle, draw=green!60!black, fill=green!12, rounded corners=2pt, minimum height=0.6cm, minimum width=1.2cm, align=center, font=\scriptsize, thick},
    concept/.style={circle, draw=black!80, fill=white, minimum size=0.36cm, font=\footnotesize},
    molconcept/.style={concept, fill=blue!15, draw=blue!60},
    pocconcept/.style={concept, fill=orange!15, draw=orange!60},
    seqconcept/.style={concept, fill=green!12, draw=green!55!black},
    bindvar/.style={rectangle, draw=red!60!black, fill=red!10, rounded corners=2pt, minimum height=0.35cm, minimum width=0.35cm, font=\tiny},
    arrow/.style={-{Stealth[length=1.5mm]}, thick, gray!70},
    bindarrow/.style={-{Stealth[length=1.5mm]}, thick, green!60!black},
    spurious/.style={-{Stealth[length=1.5mm]}, thick, red!40, dashed},
    label/.style={font=\small},
    sidelabel/.style={font=\small\bfseries},
]
\node[sidelabel, orange!70!black] at (-4.5, 2.8) {Pocket};
\node[inputbox] (poc_input) at (-4.5, 2) {3D Structure};
\node[encoderbox] (poc_enc) at (-4.5, 1) {Pocket Encoder};
\node[embedbox] (poc_emb) at (-4.5, 0) {$\mathbf{h}_{\mathrm{poc}}$};
\node[conceptorbox] (poc_ce) at (-4.5, -1) {Concept\\Extractor};
\node[pocconcept] (p1) at (-5.9, -2.75) {$c_{\mathrm{poc}}^1$};
\node[pocconcept] (p2) at (-4.5, -2.75) {$c_{\mathrm{poc}}^2$};
\node[pocconcept] (p3) at (-3.1, -2.75) {$c_{\mathrm{poc}}^3$};
\draw[arrow] (poc_input) -- (poc_enc);
\draw[arrow] (poc_enc) -- (poc_emb);
\draw[arrow] (poc_emb) -- (poc_ce);
\draw[arrow] (poc_ce) -- (-5.35, -1.55) -- (p1);
\draw[arrow] (poc_ce) -- (-4.5, -1.55) -- (p2);
\draw[arrow] (poc_ce) -- (-3.65, -1.55) -- (p3);

\node[sidelabel, blue!70] at (0, 2.8) {Molecule};
\node[inputbox] (mol_input) at (0, 2) {3D Structure};
\node[encoderbox] (mol_enc) at (0, 1) {Molecule Encoder};
\node[embedbox] (mol_emb) at (0, 0) {$\mathbf{h}_{\mathrm{mol}}$};
\node[conceptorbox] (mol_ce) at (0, -1) {Concept\\Extractor};
\node[molconcept] (m1) at (-1.4, -2.75) {$c_{\mathrm{mol}}^1$};
\node[molconcept] (m2) at (0, -2.75) {$c_{\mathrm{mol}}^2$};
\node[molconcept] (m3) at (1.4, -2.75) {$c_{\mathrm{mol}}^3$};
\draw[arrow] (mol_input) -- (mol_enc);
\draw[arrow] (mol_enc) -- (mol_emb);
\draw[arrow] (mol_emb) -- (mol_ce);
\draw[arrow] (mol_ce) -- (-0.9, -1.55) -- (m1);
\draw[arrow] (mol_ce) -- (0, -1.55) -- (m2);
\draw[arrow] (mol_ce) -- (0.9, -1.55) -- (m3);

\node[sidelabel, green!55!black] at (4.5, 2.8) {Sequence};
\node[inputbox] (seq_input) at (4.5, 2) {Protein Sequence};
\node[encoderbox] (seq_enc) at (4.5, 1) {Sequence Encoder};
\node[embedbox] (seq_emb) at (4.5, 0) {$\mathbf{h}_{\mathrm{seq}}$};
\node[conceptorbox] (seq_ce) at (4.5, -1) {Concept\\Extractor};
\node[seqconcept] (s1) at (3.1, -2.75) {$c_{\mathrm{seq}}^1$};
\node[seqconcept] (s2) at (4.5, -2.75) {$c_{\mathrm{seq}}^2$};
\node[seqconcept] (s3) at (5.9, -2.75) {$c_{\mathrm{seq}}^3$};
\draw[arrow] (seq_input) -- (seq_enc);
\draw[arrow] (seq_enc) -- (seq_emb);
\draw[arrow] (seq_emb) -- (seq_ce);
\draw[arrow] (seq_ce) -- (3.65, -1.55) -- (s1);
\draw[arrow] (seq_ce) -- (4.5, -1.55) -- (s2);
\draw[arrow] (seq_ce) -- (5.35, -1.55) -- (s3);

\node[bindvar] (pm1) at (-3.65, -4.8) {$1$};
\node[bindvar] (pm2) at (-2.25, -4.8) {$1$};
\node[bindvar, fill=gray!20, draw=gray!50] (pm3) at (-0.85, -4.8) {$0$};
\node[bindvar] (ms1) at (0.85, -4.8) {$1$};
\node[bindvar] (ms2) at (2.25, -4.8) {$1$};
\node[bindvar, fill=gray!20, draw=gray!50] (ms3) at (3.65, -4.8) {$0$};

\draw[bindarrow] (p1) -- (pm1);
\draw[bindarrow] (m1) -- (pm1);
\draw[bindarrow] (p2) -- (pm2);
\draw[bindarrow] (m2) -- (pm2);
\draw[spurious] (p3) -- (pm3);
\draw[spurious] (m3) -- (pm3);

\draw[bindarrow] (m1) -- (ms1);
\draw[bindarrow] (s1) -- (ms1);
\draw[bindarrow] (m2) -- (ms2);
\draw[bindarrow] (s2) -- (ms2);
\draw[spurious] (m3) -- (ms3);
\draw[spurious] (s3) -- (ms3);

\node[label] at (-2.25, -5.6) {\textcolor{green!60!black}{$\rightarrow$} $B_{ij}=1$: Active binding};
\node[label] at (2.25, -5.6) {\textcolor{red!40}{$\dashrightarrow$} $B_{ij}=0$: Blocked};
\end{tikzpicture}
\caption{Implementation-level view of \ours{}. Each modality-specific encoder maps its input to \rev{atomic concepts $\{c_v^k\}$,} \rev{each with dimension $d_c\ge1$}. The indicators $B_{ij}$ at the bottom visualize the \rev{branch-wise} cross-modal mask $\mathcal{M}_{a,b}$ \rev{for $(a,b)\in\{(\mathrm{mol},\mathrm{poc}),(\mathrm{mol},\mathrm{seq})\}$}: $B_{ij}{=}1$ when the $(i,j)$ entry of $\mathcal{M}_{a,b}$ survives the forward threshold \rev{($\tau=0.5$)} (green solid edges, active pair), and $B_{ij}{=}0$ otherwise (red dashed edges, blocked pair).  Active pairs form the contrastive score, while $\mathcal{L}_\mathrm{spa}$ encourages the interaction sparsity levels.}
\label{fig:causal_graph}
\end{SCfigure}
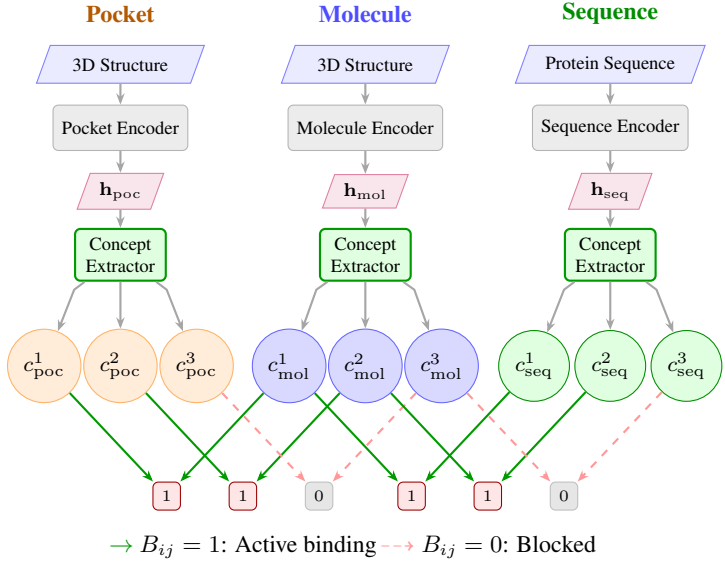

\paragraph{Concept Extractor.}\label{subsec:concept-factorize}
Each modality is decomposed into $K$ latent concepts using a Perceiver-style cross-attention block~\citep{perceiver2021}:
\begin{equation}
\begin{aligned}
\widetilde{\mathbf C}^{(\ell+1)}_v
&=
\mathbf{C}^{(\ell)}_v
+\mathrm{MHA}\!\bigl(\mathrm{LN}(\mathbf{C}^{(\ell)}_v),\,\mathrm{LN}(\widehat{\mathbf H}_v),\,\mathrm{LN}(\widehat{\mathbf H}_v)\bigr),\\
\mathbf{C}^{(\ell+1)}_v
&=
\widetilde{\mathbf C}^{(\ell+1)}_v
+\mathrm{FFN}\!\bigl(\mathrm{LN}(\widetilde{\mathbf C}^{(\ell+1)}_v)\bigr),
\end{aligned}
\end{equation}
where $\mathbf{C}^{(0)}_v = \mathbf{Q}_v\in\mathbb{R}^{K\times d_c}$ are learnable concept queries and $\widehat{\mathbf H}_v=\mathbf H_v W_v\in\mathbb{R}^{N_v\times d_c}$ is the projected token sequence. After $L$ layers, we obtain $\mathbf{C}_v\in\mathbb{R}^{K\times d_c}$, which serves as the practical counterpart of the latent concepts in $\S$~\ref{sec:theory}. If a backbone exposes only a pooled vector, it corresponds to the special case $N_v=1$; in the main structural branches, token-level features provide nontrivial cross-attention keys and values.

\paragraph{Sparsity Mask.}\label{subsec:variants}
Following $\S$~\ref{sec:theory}, we model binding through sparse cross-modal interactions rather than dense alignment. For each branch $(a,b)\in\{(\mathrm{mol},\mathrm{poc}),(\mathrm{mol},\mathrm{seq})\}$, we compute
\begin{equation}
\begin{aligned}
\mathbf Z_{a,b}&=\Gamma(\widehat{\mathcal{M}}_{a,b}),&
\mathcal{M}_{a,b}&=H_{\tau}^{\mathrm{ST}}\!\left(\mathbf Z_{a,b}\right),\\
\mathbf{A}_{a,b}&=\phi(\mathbf{C}_{a},\mathbf{C}_{b}),&
u_{a,b}&=\mathrm{Agg}(\mathcal{M}_{a,b}\odot \mathbf{A}_{a,b}),
\end{aligned}
\end{equation}
\rev{where, under the common-$K$ implementation, $\widehat{\mathcal M}_{a,b}\in\mathbb{R}^{K\times K}$ is an unconstrained parameter matrix and $\mathcal M_{a,b}$ is the effective sparse mask used for interaction. The activation $\Gamma(\widehat{\mathcal M})=\tanh(\widehat{\mathcal M})+1$ first produces a continuous mask $\mathbf Z_{a,b}\in[0,2]^{K\times K}$, and the straight-through threshold operator $H_{\tau}^{\mathrm{ST}}$ is applied element-wise with mask threshold $\tau=0.5$: the forward pass uses $[H_{\tau}(Z)]_{ij}=Z_{ij}\mathbf{1}\{Z_{ij}>\tau\}$, while the backward pass sends gradients through the continuous $Z_{ij}$. Thus thresholding occurs before aggregation, and $\mathrm{Agg}$ receives only the thresholded masked compatibility matrix $\mathcal M_{a,b}\odot\mathbf A_{a,b}$. Each nonzero entry $\mathcal{M}_{a,b}[i,j]$ represents the strength of interaction between the $i$-th concept token of modality $a$ and the $j$-th concept token of modality $b$.}
\rev{Here $\mathcal{M}_{a,b}$ is the implementation counterpart of the theoretical interaction mask $\mathcal{M}$ in $\S$~\ref{sec:theory}. For notational consistency, we use the same compatibility symbol $\phi$: in the scalar theory $\phi:\mathbb R\times\mathbb R\to\mathbb R$, while in the vector-valued implementation $\phi:\mathbb R^{d_c}\times\mathbb R^{d_c}\to\mathbb R$. Together, $\mathcal M_{a,b}$ and $\phi$ instantiate the gated interaction term $\mathcal{M}_{ij}\phi(c_i^M,c_j^P)$ at the token level.} 
\rev{The matrix $\mathbf{A}_{a,b}\in\mathbb{R}^{K\times K}$ collects all concept-pair compatibilities, where each entry is computed by $\phi$ (we use a simple inner product between vector concept tokens). The operator $\odot$ denotes element-wise multiplication, and $\mathrm{Agg}(\mathbf B)=K^{-1}\sum_{i,j}B_{ij}$ aggregates the retained interactions into a branch-level score in our implementation. Thus, the masks $\mathcal M_{\mathrm{mol},\mathrm{poc}}$ and $\mathcal M_{\mathrm{mol},\mathrm{seq}}$ act as cross-modal bottlenecks, enforcing the interaction sparsity.}

\paragraph{Training Loss.}\label{subsec:training}
We retain retrieval-based supervision applied after concept decomposition and sparse interaction. The scalar $u_{a,b}$ defines the pairwise score for a protein-molecule pair, and over a mini-batch of size $N$, these scores form matrices $\mathbf{U}^{a,b} \in \mathbb{R}^{N \times N}$. \rev{The contrastive loss is} 
\begin{equation}
{\color{black}
\ell_{\mathrm{cont}}(\mathbf U^{a,b}) =
-\frac{1}{2N}\sum_{i=1}^{N}\left[
\log\frac{\exp(U^{a,b}_{ii}/\tau_{\mathrm{c}})}{\sum_{j=1}^{N}\exp(U^{a,b}_{ij}/\tau_{\mathrm{c}})} +
\log\frac{\exp(U^{a,b}_{ii}/\tau_{\mathrm{c}})}{\sum_{j=1}^{N}\exp(U^{a,b}_{ji}/\tau_{\mathrm{c}})}
\right],
}
\end{equation}
\rev{where $\tau_{\mathrm{c}}>0$ is the contrastive temperature. The combined contrastive objective is}
\begin{equation}
{\color{black}
\mathcal{L}_{\mathrm{cont}}
=
\ell_{\mathrm{cont}}(\mathbf U^{\mathrm{poc,mol}})
+\ell_{\mathrm{cont}}(\mathbf U^{\mathrm{seq,mol}}).
}
\end{equation}
\rev{The theory models the binder-selected positive distribution; the off-diagonal batch entries here are retrieval negatives used by the contrastive surrogate, not additional observed $B{=}0$ labels.}

We further add sparsity and geometric regularization terms~\citep{hypseek2025}. \rev{The sparsity loss is the entry-wise $L_1$ penalty on the continuous activated masks, which preserves gradients to the raw mask parameters:}
\begin{equation}
\rev{\mathcal{L}_{\mathrm{spa}}
=
\|\mathbf Z_{\mathrm{mol},\mathrm{poc}}\|_1
+\|\mathbf Z_{\mathrm{mol},\mathrm{seq}}\|_1.}
\label{eq:lspa}
\end{equation}
\rev{The full training objective is} 
\begin{equation}
{\color{black}
\mathcal{L}
=
\mathcal{L}_{\mathrm{cont}}
+\lambda_{\mathrm{spa}}\mathcal{L}_{\mathrm{spa}}
+\lambda_{\mathrm{reg}}\mathcal{L}_{\mathrm{reg}}.
}
\end{equation}
\rev{Here $\mathcal{L}_{\mathrm{reg}}$ collects the geometric regularizers inherited from HypSeek~\citep{hypseek2025}.}

\begin{table}[t!]
\caption{Virtual-screening comparison on two benchmarks between baselines and \ours{}$_{\mathrm{SP}}$ ($K{=}128,d_c{=}1024$), \ours{}$_{\mathrm{LR}}$ ($K{=}64,r{=}1$), and \ours{}$_{\mathrm{EMB}}$ ($K{=}64,d_c{=}256$), and \ours{}$_{\mathrm{ATOM}}$ ($K{=}128,d_c{=}1024$). $^\dagger$S$^2$Drug results are paper-reported~\citep{s2drug2026}.}
\label{tab:main_results}
\centering
\setlength{\tabcolsep}{5pt}
\small
\resizebox{0.93\textwidth}{!}{%
\begin{tabular}{l|ccc|ccc}
\toprule
\multirow{2}{*}{\textbf{Method}} & \multicolumn{3}{c|}{\textbf{DUD-E} ($n{=}102$)} & \multicolumn{3}{c}{\textbf{LIT-PCBA} ($n{=}15$)} \\
& AUROC & BEDROC$_{80.5}$ & EF$_{1\%}$ & AUROC & BEDROC$_{80.5}$ & EF$_{1\%}$ \\
\midrule
Glide-SP~\citep{glide2004} & 0.767 & 0.407 & 16.18 & 0.532 & 0.040 & 3.41 \\
DeepDTA~\citep{deepdta2018} & 0.584 & 0.051 & 2.28 & 0.563 & 0.025 & 1.47 \\
Gnina~\citep{gnina2021} & 0.782 & 0.299 & 17.73 & 0.609 & 0.054 & 4.63 \\
RTMScore~\citep{rtmscore2023} & 0.753 & 0.434 & 27.10 & 0.525 & 0.039 & 2.94 \\
TankBind~\citep{tankbind2022} & 0.751 & 0.330 & 13.00 & 0.597 & 0.039 & 2.90 \\
BigBind~\citep{bigbind2022} & 0.501 & 0.024 & 1.18 & 0.628 & 0.050 & 3.79 \\
DrugCLIP~\citep{drugclip2023} & 0.809 & 0.505 & 31.89 & 0.572 & 0.062 & 5.51 \\
GenScore~\citep{genscore2023} & 0.816 & 0.473 & 28.53 & 0.596 & 0.065 & 5.14 \\
EquiScore~\citep{equiscore2023} & 0.776 & 0.432 & 17.68 & 0.568 & 0.049 & 3.51 \\
DrugHash~\citep{drughash2025} & 0.837 & 0.572 & 37.18 & 0.546 & 0.071 & 6.14 \\
% LigUnity$_\text{poc}$~\citep{ligunity2024} & 0.892 & 0.653 & 42.63 & 0.599 & 0.113 & 6.47 \\
% HypSeek~\citep{hypseek2025} (reported) & \textbf{0.944} & \textbf{0.789} & \textbf{51.44} & 0.621 & \textbf{0.120} & 6.81 \\
LigUnity~\citep{ligunity2024} & 0.897 & 0.674 & 44.20 & 0.599 & 0.075 & 6.50 \\
HypSeek~\citep{hypseek2025} & 0.909 & 0.605 & 37.83 & 0.603 & 0.061 & 4.62 \\
S$^2$Drug$^\dagger$~\citep{s2drug2026} & 0.925 & \textbf{0.793} & 43.06 & 0.582 & 0.087 & 7.38 \\
% \midrule
% \ours{}-V0 (concept-only ablation, seed=1) & 0.932 & 0.704 & 44.48 & 0.627 & 0.061 & 4.67 \\
\rowcolor{green!7}\textbf{\ours{}$_{\mathrm{SP}}$ (ours)} & 0.935 & 0.744 & 47.69 & \textbf{0.639} & 0.097 & 8.56 \\
\rowcolor{green!7}\textbf{\ours{}$_{\mathrm{LR}}$ (ours)} & 0.935 & 0.704 & 44.87 & 0.632 & \textbf{0.103} & 8.84 \\
\rowcolor{green!7}\textbf{\ours{}$_{\mathrm{EMB}}$ (ours)} & 0.939 & 0.754 & 48.56 & 0.636 & 0.102 & \textbf{9.60} \\
\rowcolor{green!7}\textbf{\ours{}$_{\mathrm{ATOM}}$ (ours)} & \textbf{0.943} & 0.768 & \textbf{49.66} & 0.621 & 0.079 & 6.25 \\

\bottomrule
\end{tabular}}
\end{table}

Instead of letting the contrastive objective operate on dense embeddings, we require it to act through sparse concept-level interactions. This enforces that binding-relevant dependence is captured through a small subset of cross-modal interactions. \rev{The other variants change only how the interaction is constrained: \ours{}$_{\mathrm{LR}}$ factorizes the raw mask as $\widehat{\mathcal{M}}_{a,b}=\mathbf{U}_{a,b}\mathbf{V}_{a,b}^{\top}$ with $\mathbf{U}_{a,b},\mathbf{V}_{a,b}\in\mathbb{R}^{K\times r}$ before the same activation and gate, \ours{}$_{\mathrm{EMB}}$ applies the constrained mask to the pooled embedding instead of the concept axis, and \ours{}$_{\mathrm{ATOM}}$ feeds the full per-atom Uni-Mol token sequences to the concept extractor.} \rev{Further implementation details for the main \ours{}$_{\mathrm{SP}}$ model, including the explicit form of the geometric regularizers in $\mathcal{L}_{\mathrm{reg}}$, and a side-by-side comparison of the three main implementations are provided in App.~\ref{app:impl_sp} and App.~\ref{app:impl_compare}.}

\section{Experiments}
\label{sec:experiments}

\paragraph{Experimental setup.} We evaluate on two complementary benchmarks: {DUD-E}~\citep{dude2012}, which covers broad target diversity, and {LIT-PCBA}~\citep{litpcba2020,pcba2015}, which suppresses property-based shortcuts and is the stricter test of genuine binding discrimination. We report standard metrics: AUROC, BEDROC$_{80.5}$, and EF@1\%, with the latter two emphasizing early enrichment. We adopt the same curated assay-level training dataset as LigUnity~\citep{ligunity2024} and HypSeek~\citep{hypseek2025}, constructed from ChEMBL~\citep{mendez_chembl_2018}, BindingDB~\citep{bindingdb2016}, and PDBBind~\citep{liu2017pdbbind}. \rev{As in these baselines, we remove every ChEMBL/BindingDB assay whose target UniProt ID appears in DUD-E, LIT-PCBA, or DEKOIS 2.0, reducing this portion from 26{,}729 to 18{,}316 assays, while the PDBBind pairs are not filtered by UniProt ID (App.~\ref{app:implementation}). The evaluation is therefore a cross-dataset generalization protocol rather than a random split of the training pool, with about 476K training molecules and 130K pockets.} All models are trained for 50 epochs on four NVIDIA A100 GPUs under the same pipeline for a fair comparison, and the main tables report evaluation results from the last checkpoint, which avoids tuning checkpoint selection to the held-out test benchmarks in the cross-dataset generalization setting. 

\paragraph{Main Comparison on Virtual Screening.}
\label{subsec:main_results}
Table~\ref{tab:main_results} compares \ours{} with three families of virtual-screening methods: classical docking systems, structure-aware scoring models on protein-ligand complexes, and retrieval-based representation learners. Retrieval methods are strongest in this large-scale regime, so the key comparison is against recent retrieval baselines. \ours{}$_{\mathrm{SP}}$ improves over LigUnity and HypSeek on all six metrics, with the largest gains on LIT-PCBA early enrichment, where property shortcuts are least useful: EF@1\% rises from $6.50$ to $8.56$ over LigUnity and from $4.62$ to $8.56$ over HypSeek, while BEDROC$_{80.5}$ improves from $0.075$ and $0.061$ to $0.097$. \rev{Against the paper-reported results of the concurrent S$^2$Drug~\citep{s2drug2026}, which uses a heavier two-stage sequence--structure pipeline, S$^2$Drug attains the best DUD-E BEDROC$_{80.5}$, while all three \ours{} variants achieve higher DUD-E AUROC and EF@1\% and higher values on all three LIT-PCBA metrics; a matched EF@0.5\% comparison is given in App.~\ref{app:s2drug}.} \rev{\ours{}$_{\mathrm{LR}}$ and \ours{}$_{\mathrm{EMB}}$ further show that the same constrained-interaction principle remains strong under explicit low-rank and embedding-level implementations. We also report \ours{}$_{\mathrm{ATOM}}$, an atom-attentive version of \ours{}$_{\mathrm{SP}}$ with the same configuration, whose concept extractor attends over the full per-atom token sequences of the Uni-Mol molecule and pocket encoders rather than their pooled tokens, so that concepts can be traced to individual atoms (App.~\ref{app:thrb_case}). It attains the best DUD-E AUROC and EF@1\%, but weaker LIT-PCBA early enrichment than the other variants.} Overall, sparse or low-rank interaction helps when holistic alignment confronts nuisance correlations. 

\begin{table}[t]
\caption{Ablation and diagnostic study on DUD-E and LIT-PCBA. \ours{}$_{\mathrm{SP}}$ is the sparse mask, \ours{}$_{\mathrm{LR}}$ is the low-rank mask, and \ours{}$_{\mathrm{EMB}}$ is the pooled embedding-space mask. $-L_{\mathrm{spa}}$ removes the sparsity penalty from \ours{}$_{\mathrm{SP}}$, while $+L_{\mathrm{spa}}$ tests an optional sparsity penalty on \ours{}$_{\mathrm{LR}}$ without changing its rank-one parameterization. $\mathbf M_{\mathrm{emb}}^{r=1}$ replaces the pooled embedding mask with a rank-one mask. Rows report mean$_{\pm\,\text{std}}$ over five runs.}
\label{tab:ablation_results}
\centering
\setlength{\tabcolsep}{5pt}
\small
\resizebox{\textwidth}{!}{%
\begin{tabular}{l|ccc|ccc}
\toprule
\multirow{2}{*}{\textbf{Method}} & \multicolumn{3}{c|}{\textbf{DUD-E} ($n{=}102$)} & \multicolumn{3}{c}{\textbf{LIT-PCBA} ($n{=}15$)} \\
& AUROC & BEDROC$_{80.5}$ & EF$_{1\%}$ & AUROC & BEDROC$_{80.5}$ & EF$_{1\%}$ \\
\midrule
LigUnity~\citep{ligunity2024} & 0.896 \tiny{$\pm$ 0.005} & 0.660 \tiny{$\pm$ 0.018} & 43.14 \tiny{$\pm$ 1.42} & 0.586 \tiny{$\pm$ 0.010} & 0.077 \tiny{$\pm$ 0.001} & 6.46 \tiny{$\pm$ 0.18} \\
HypSeek~\citep{hypseek2025} & 0.914 \tiny{$\pm$ 0.013} & 0.690 \tiny{$\pm$ 0.081} & 44.12 \tiny{$\pm$ 6.14} & 0.603 \tiny{$\pm$ 0.009} & 0.067 \tiny{$\pm$ 0.007} & 5.62 \tiny{$\pm$ 0.71} \\
\midrule
\rowcolor{green!7}\ours{}$_{\mathrm{SP}}$ & \textbf{0.939} \tiny{$\pm$ 0.004} & 0.746 \tiny{$\pm$ 0.004} & 47.99 \tiny{$\pm$ 0.39} & \textbf{0.635} \tiny{$\pm$ 0.006} & 0.088 \tiny{$\pm$ 0.007} & 7.55 \tiny{$\pm$ 0.73} \\
\ours{}$_{\mathrm{SP}}$ $-\,L_\mathrm{spa}$ & 0.935 \tiny{$\pm$ 0.004} & 0.732 \tiny{$\pm$ 0.014} & 46.93 \tiny{$\pm$ 1.07} & 0.623 \tiny{$\pm$ 0.005} & 0.084 \tiny{$\pm$ 0.004} & 7.17 \tiny{$\pm$ 0.36} \\
% \ours{}$_{\mathrm{SP}}$ $-\,L_\mathrm{spa}-\mathbf{M}$ & 0.938 \tiny{$\pm$ 0.003} & 0.735 \tiny{$\pm$ 0.010} & 47.02 \tiny{$\pm$ 1.04} & 0.622 \tiny{$\pm$ 0.003} & \textbf{0.089} \tiny{$\pm$ 0.006} & \textbf{7.77} \tiny{$\pm$ 0.97} \\
% \ours{}$_{\mathrm{SP}}$\,$-L_\mathrm{spa}-\mathbf{M}-$\,extractor & 0.924 \tiny{$\pm$ 0.013} & 0.690 \tiny{$\pm$ 0.081} & 44.12 \tiny{$\pm$ 6.14} & 0.603 \tiny{$\pm$ 0.009} & 0.067 \tiny{$\pm$ 0.007} & 5.62 \tiny{$\pm$ 0.71} \\
% \ours{}$_{\mathrm{SP}}$\,$-L_\mathrm{spa}-L_\mathrm{cont}$ & 0.914 \tiny{$\pm$ 0.013} & 0.690 \tiny{$\pm$ 0.081} & 44.12 \tiny{$\pm$ 6.14} & 0.603 \tiny{$\pm$ 0.009} & 0.067 \tiny{$\pm$ 0.007} & 5.62 \tiny{$\pm$ 0.71} \\
\midrule
\rowcolor{green!7}\ours{}$_{\mathrm{LR}}$ & 0.937 \tiny{$\pm$ 0.004} & 0.732 \tiny{$\pm$ 0.011} & 46.87 \tiny{$\pm$ 0.91} & 0.630 \tiny{$\pm$ 0.006} & 0.085 \tiny{$\pm$ 0.006} & 7.22 \tiny{$\pm$ 0.31} \\
\ours{}$_{\mathrm{LR}}$\,$+L_\mathrm{spa}$ & 0.935 \tiny{$\pm$ 0.003} & 0.727 \tiny{$\pm$ 0.007} & 46.39 \tiny{$\pm$ 0.58} & 0.623 \tiny{$\pm$ 0.007} & 0.083 \tiny{$\pm$ 0.005} & 7.51 \tiny{$\pm$ 0.94} \\
\midrule
\rowcolor{green!7}\ours{}$_{\mathrm{EMB}}$ & 0.938 \tiny{$\pm$ 0.001} & \textbf{0.752} \tiny{$\pm$ 0.005} & \textbf{48.47} \tiny{$\pm$ 0.44} & 0.634 \tiny{$\pm$ 0.008} & \textbf{0.101} \tiny{$\pm$ 0.007} & \textbf{9.52} \tiny{$\pm$ 0.96} \\
\ours{}$_{\mathrm{EMB}}$ $+\mathbf M_{\mathrm{emb}}^{r=1}$ & 0.723 \tiny{$\pm$ 0.022} & 0.319 \tiny{$\pm$ 0.013} & 20.96 \tiny{$\pm$ 0.63} & 0.539 \tiny{$\pm$ 0.003} & 0.018 \tiny{$\pm$ 0.001} & 1.38 \tiny{$\pm$ 0.10} \\
\bottomrule
\end{tabular}}
\end{table}

\begin{figure}[t]
\centering
\includegraphics[width=\textwidth]{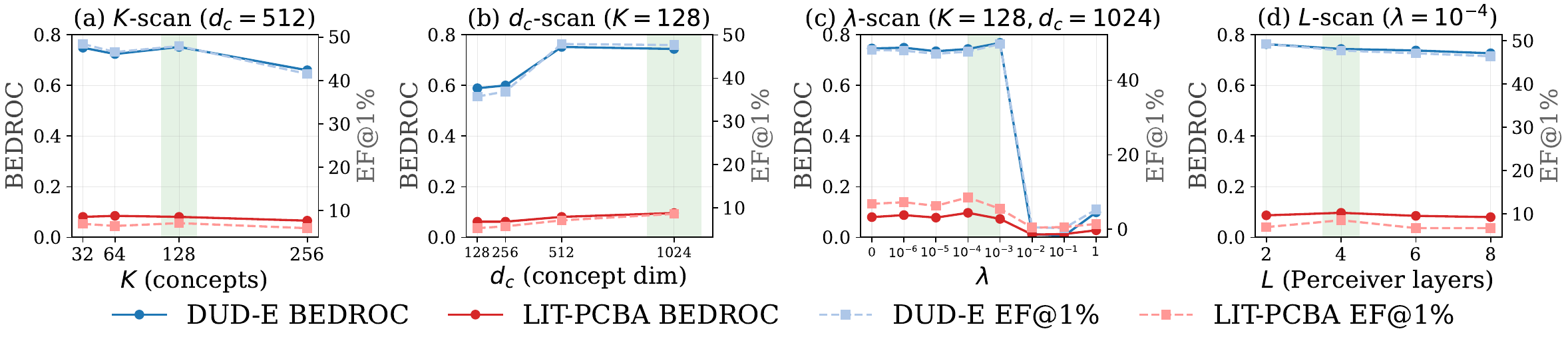}
\setlength{\abovecaptionskip}{-0.4cm}
\caption{Hyperparameter sweeps for \ours{}$_{\mathrm{SP}}$ on DUD-E and LIT-PCBA. Solid lines show BEDROC$_{80.5}$ and dashed lines show EF@1\%; shaded bands mark the selected hyperparameters.}
\label{fig:hparam_scans}
\end{figure}

\paragraph{Ablation Diagnostics.}
\label{subsec:ablation}
\rev{Table~\ref{tab:ablation_results} reports both component ablations and implementation diagnostics. For \ours{}$_{\mathrm{SP}}$, removing $L_\mathrm{spa}$ weakens early enrichment,}  \rev{suggesting that the concept-pair mask is most useful when regularized toward a simple interaction structure. For \ours{}$_{\mathrm{LR}}$, the rank-one factorization is retained in both rows; adding $L_\mathrm{spa}$ has limited effect, so the gain should be attributed to the explicit low-rank parameterization rather than to sparsity regularization. For \ours{}$_{\mathrm{EMB}}$, imposing a rank-one pooled embedding mask sharply degrades performance, showing that rank compression must act directly on the concept axis rather than after pooling.} 

\paragraph{Hyperparameter Study.}
\label{subsec:hparam}
We sweep $K$ (concept count), $d_c$ (concept dimension), $\lambda$ (sparsity weight $\lambda_{\mathrm{spa}}$), and $L$ (Perceiver depth). Figure~\ref{fig:hparam_scans} and Table~\ref{tab:lambda_mask} show robustness around our setting ($K{=}128, d_c{=}1024, \lambda{=}10^{-4}, L{=}4$), with failures: $K{=}256$ over-parameterizes, $d_c \leq 256$ underfits, and $\lambda \geq 10^{-2}$ collapses both benchmarks. In the useful regime, $\lambda{=}10^{-4}$ favors LIT-PCBA and $\lambda{=}10^{-3}$ favors DUD-E, so the sparsity level is benchmark-dependent. Overall, a soft interaction bottleneck helps, while hard-zero masking does not. \rev{The mask statistics in Table~\ref{tab:lambda_mask} describe the structure the model reaches: $\lambda{=}10^{-4}$ has no hard zeros but concentrates its energy in one singular direction, the low-rank regime of Condition~3.i; $\lambda{=}10^{-3}$ gives an exact sparse-antichain support; and $\lambda{=}10^{-2}$ keeps the antichain property but over-sparsifies, so that property alone is insufficient. As $\phi$ is bilinear here, these are structural diagnostics rather than certificates of component-wise identifiability.} 

\begin{table}[t]
\caption{$\lambda_{\mathrm{spa}}$-scan for \ours{}$_{\mathrm{SP}}$ with paired cross-modal mask statistics, averaged over the two branches. Small $\lambda_{\mathrm{spa}}$ preserves a low-rank regime, while large $\lambda_{\mathrm{spa}}$ drives hard-zero collapse. `Anti.\ viol.' counts antichain violations, i.e., pairs of row (or column) supports in which one contains the other, over $32{,}512$ pairwise tests. Highlighted rows are the two operating points used in the text.}
\label{tab:lambda_mask}
\centering\small
\setlength{\tabcolsep}{4pt}
\resizebox{\textwidth}{!}{%
\begin{tabular}{c|ccc|ccc|ccccc}
\toprule
\multirow{2}{*}{$\lambda_{\mathrm{spa}}$} & \multicolumn{3}{c|}{\textbf{DUD-E}} & \multicolumn{3}{c|}{\textbf{LIT-PCBA}} & \multicolumn{5}{c}{\textbf{Cross-modal mask statistics}} \\
 & AUC & BEDROC$_{80.5}$ & EF@1\% & AUC & BEDROC$_{80.5}$ & EF@1\% & $L_1$ avg & \%zeros & $r@95\%$ & $r@99\%$ & Anti.\ viol. \\
\midrule
$0$            & 0.938 & 0.746 & 48.14 & 0.617 & 0.080 & 6.83 & 0.93 & 0.0  & 1   & 25.0 & 32{,}512 \\
$10^{-6}$      & 0.940 & 0.749 & 47.94 & 0.635 & 0.088 & 7.26 & 0.92 & 0.0  & 1   & 25.0 & 32{,}512 \\
$10^{-5}$      & 0.938 & 0.735 & 47.09 & 0.607 & 0.078 & 6.29 & 0.88 & 0.0  & 1   & 26.0 & 32{,}512 \\
\rowcolor{green!7}$10^{-4}$ & 0.935 & 0.744 & 47.69 & \textbf{0.639} & \textbf{0.097} & \textbf{8.56} & 0.61 & 0.0  & \textbf{1}   & \textbf{34.0} & 32{,}512 \\
\rowcolor{green!7}$10^{-3}$ & \textbf{0.943} & \textbf{0.768} & \textbf{49.73} & 0.614 & 0.073 & 5.49 & \textbf{0.21} & \textbf{20.5} & 45.5 & 79.5 & \textbf{0} \\
$10^{-2}$      & 0.437 & 0.013 & 0.53  & 0.457 & 0.011 & 0.46 & 0.06 & 78.4 & 74.0 & 96.0 & 0 \\
$10^{-1}$      & 0.297 & 0.006 & 0.30  & 0.437 & 0.013 & 0.59 & 0.02 & 92.0 & 75.5 & 97.5 & 1 \\
$1$            & 0.725 & 0.099 & 5.30  & 0.595 & 0.028 & 1.44 & 0.01 & 96.0 & 75.5 & 96.5 & 324 \\
\bottomrule
\end{tabular}}
\vspace{-1em}
\end{table}

\rev{\textbf{Out-of-distribution generalization and interpretability.} \ We further evaluate zero-shot transfer on DEKOIS 2.0~\citep{bauer2013dekois} under a target-level split over nine targets whose UniProt IDs are absent from all training sources, and a joint target--scaffold split, where \ours{} improves over HypSeek and LigUnity on AUROC and BEDROC$_{80.5}$ in both settings (App.~\ref{app:dekois_ood}). An interpretability case study on the Thyroid Hormone Receptor $\beta$ (THRB) complex from DUD-E further shows that observed binding contacts concentrate within a learned concept pair, and that masking the associated atoms reduces the predicted binding score far more than size-matched random masks (App.~\ref{app:thrb_case}).}

\textbf{Appendix roadmap.} \ For readability, the main text reports only the core comparisons above. Implementation details and reproduction checks are in App.~\ref{app:implementation}, \rev{including the three variants \ours{}$_{\mathrm{SP}}$, \ours{}$_{\mathrm{LR}}$, and \ours{}$_{\mathrm{EMB}}$.} Additional protocol comparisons\rev{, including the \texttt{best\_auc} and \texttt{best\_bedroc} checkpoints of \ours{}$_{\mathrm{ATOM}}$,} are in App.~\ref{app:baseline_3prot}, and a longer discussion of the ablation study is in App.~\ref{app:ablation_analysis}. The full hyperparameter tables for $K/d_c$, $\lambda$, and $L$ are in App.~\ref{app:v2_kd_3prot}, App.~\ref{app:v2_lambda_3prot}, and App.~\ref{app:v2_layers_3prot}. The extended theory--practice analysis of the learned masks\rev{, including the antichain statistics of Table~\ref{tab:lambda_mask},} is in App.~\ref{app:v2_mask_stats}. Supplementary DUD-E ROC-enrichment and Merck-FEP affinity-ranking analyses are in App.~\ref{app:re_dude} and App.~\ref{app:sp_merck_affinity}, \rev{low-rank and embedding-level variant results are in App.~\ref{app:lr_results} and App.~\ref{app:emb_results}}, and the full-sweep synthesis is in App.~\ref{app:full_sweep}. \rev{The matched EF@0.5\% comparison with S$^2$Drug, the DEKOIS 2.0 out-of-distribution evaluation, and the THRB interpretability case study are in App.~\ref{app:s2drug}, App.~\ref{app:dekois_ood}, and App.~\ref{app:thrb_case}.} A comprehensive discussion of the remaining theory--practice gap is in App.~\ref{app:theory:gap}.

\section{Conclusion}
\label{sec:conclusion}

\rev{This paper revisits retrieval-based virtual screening from a structural perspective. Existing methods largely rely on dense shared-space alignment, although protein-molecule binding is typically driven by a small set of local compatibilities. We formalized this mismatch with a V-structure causal model under binder-only selection, showing that bilinear interaction is non-identifiable even when sparse, while sparse antichain interaction with a non-bilinear $\phi$ yields component-wise identifiability and low-rank interaction yields subspace identifiability. \ours{} turns this principle into a concept-based retrieval framework with three structure-constrained implementations, and consistently improves over strong baselines on DUD-E and LIT-PCBA, with especially clear gains on LIT-PCBA early enrichment. Beyond the two main benchmarks, the learned interaction structures transfer to target- and scaffold-level out-of-distribution splits of DEKOIS 2.0, and the learned masks are strongly low-rank at our operating point, the regime our subspace-identifiability theorem covers. Overall, scalable virtual screening should learn not only better shared embeddings but also more reusable cross-modal interaction structures, together with the conditions under which they remain identifiable.}

\section*{Acknowledgments}
We would like to acknowledge the support from NSF Award No.~2229881, AI Institute for Societal Decision Making (AI-SDM), the National Institutes of Health (NIH) under Contract R01HL159805, and grants from Quris AI, Florin Court Capital, MBZUAI-WIS Joint Program, and the Al Deira Causal Education project. Meanwhile, we also thank the reviewers for their constructive suggestions.

\bibliography{references_refined}
\bibliographystyle{plainnat}

\appendix

\clearpage

\textit{\large Appendix for}\\ \ \\
{\large \bf ``\ours{}: Causal Modeling and Learning for Protein-Molecule Virtual Screening''}\

\newcommand{\beginsupplement}{%
\setcounter{table}{0}
\renewcommand{\thetable}{A\arabic{table}}%
\renewcommand{\theHtable}{A\arabic{table}}%
\setcounter{figure}{0}
\renewcommand{\thefigure}{A\arabic{figure}}%
\renewcommand{\theHfigure}{A\arabic{figure}}%
\setcounter{section}{0}
\renewcommand{\thesection}{A\arabic{section}}%
\renewcommand{\theHsection}{A\arabic{section}}%
\setcounter{theorem}{0}
\renewcommand{\thetheorem}{\arabic{theorem}}
\renewcommand{\theHtheorem}{A\arabic{theorem}}
\setcounter{lemma}{0}
\renewcommand{\thelemma}{\arabic{lemma}}
\renewcommand{\theHlemma}{A\arabic{lemma}}
\setcounter{proposition}{0}
\renewcommand{\theproposition}{\arabic{proposition}}
\renewcommand{\theHproposition}{A\arabic{proposition}}
\setcounter{assumption}{0}
\renewcommand{\theassumption}{\arabic{assumption}}
\renewcommand{\theHassumption}{A\arabic{assumption}}
}

\beginsupplement

{\large Table of Contents:}

\DoToC 

\section{Drug Discovery Pipeline and Method Positioning}
\label{app:pipeline}

Figure~\ref{fig:pipeline} places \ours{} within the broader drug-discovery workflow. The key point, consistent with the main text, is that our method is designed for the \emph{high-throughput virtual-screening stage}: the part of the pipeline that must rapidly down-select a very large candidate library before more expensive docking and experimental follow-up. In other words, \ours{} is not a pose-prediction or lead-optimization method; it is a retrieval-style early filter whose job is to improve the quality of the top-ranked compounds passed to later stages.

\paragraph{Stage descriptions.} The pipeline in Figure~\ref{fig:pipeline} is best read as a funnel in which computational cost increases while the candidate set shrinks. It begins with a \emph{virtual compound library} of roughly $10^9$--$10^{10}$ commercially available or synthetically accessible molecules, such as those represented by ZINC or Enamine REAL. Basic drug-likeness and target-specific filters then reduce this space to a \emph{filtered library} of roughly $10^5$--$10^6$ candidates. The next stage is \emph{virtual screening}, where methods rank compounds by predicted binding relevance and reduce the candidate set to something tractable for follow-up; this is the stage where \ours{} operates. More expensive \emph{docking or rescoring} is then applied to a much smaller set, typically $10^2$--$10^3$ compounds, using physics-based or structure-aware scoring functions to reason in greater detail about poses and local interactions. Only after this reduction do compounds proceed to \emph{experimental validation}, usually tens to hundreds of biochemical or cell-based assays, followed by \emph{hit-to-lead} optimization, and eventually the even narrower stages of \emph{clinical testing} and \emph{approval}. The practical role of \ours{} is therefore upstream: it improves which candidates are allowed to reach these far more expensive downstream stages.

\begin{figure}[t!]
\centering
\begin{tikzpicture}[
    stage/.style={trapezium, trapezium angle=75, draw=black!70, minimum height=0.7cm, text width=4.5cm, align=center, font=\small},
    arrow/.style={-{Stealth[length=2.5mm]}, thick, gray!60},
    label/.style={font=\scriptsize, text=gray!70},
    ourstage/.style={stage, fill=green!15, draw=green!60!black, thick},
    normalstage/.style={stage, fill=blue!8, draw=blue!40},
]
\node[normalstage, text width=6cm, minimum height=0.8cm] (s0) at (0, 0)
    {\textbf{Virtual Compound Library}\\(ZINC, Enamine REAL)};
\node[label, right=0.3cm of s0] {$10^9$--$10^{10}$};
\draw[arrow] (0, -0.6) -- (0, -1.0);
\node[normalstage, text width=5.2cm, minimum height=0.8cm] (s1) at (0, -1.5)
    {\textbf{Filtered Library}\\(Drug-likeness filters)};
\node[label, right=0.3cm of s1] {$10^5$--$10^6$};
\draw[arrow] (0, -2.1) -- (0, -2.5);
\node[ourstage, text width=4.5cm, minimum height=1.0cm] (s2) at (0, -3.2)
    {\textbf{Virtual Screening}\\(Contrastive Learning)\\{\scriptsize \textcolor{green!50!black}{\ours{} operates here}}};
\node[label, right=0.3cm of s2] {$10^2$--$10^4$};
\draw[green!60!black, thick, dashed] (-3.2, -2.5) rectangle (3.2, -4.0);
\node[font=\scriptsize, green!50!black, anchor=north west] at (-3.1, -2.55) {\textbf{This Work}};
\draw[arrow] (0, -3.9) -- (0, -4.3);
\node[normalstage, text width=3.8cm, minimum height=0.7cm] (s3) at (0, -4.8)
    {\textbf{Docking / Rescoring}\\(Glide, Vina, ML scoring)};
\node[label, right=0.3cm of s3] {$10^2$--$10^3$};
\draw[arrow] (0, -5.4) -- (0, -5.8);
\node[normalstage, text width=3.2cm, minimum height=0.7cm] (s4) at (0, -6.3)
    {\textbf{Experimental Validation}\\(Biochemical assays)};
\node[label, right=0.3cm of s4] {$10^1$--$10^2$};
\draw[arrow] (0, -6.9) -- (0, -7.3);
\node[normalstage, text width=2.6cm, minimum height=0.7cm] (s5) at (0, -7.8)
    {\textbf{Hit-to-Lead}\\(Lead optimization)};
\node[label, right=0.3cm of s5] {$10^0$--$10^1$};
\draw[arrow] (0, -8.4) -- (0, -8.8);
\node[normalstage, text width=2.0cm, minimum height=0.6cm] (s6) at (0, -9.2)
    {\textbf{Clinical Trials}};
\node[label, right=0.3cm of s6] {$<10$};
\draw[arrow] (0, -9.7) -- (0, -10.1);
\node[normalstage, text width=1.5cm, minimum height=0.5cm, fill=orange!15, draw=orange!50] (s7) at (0, -10.5)
    {\textbf{Approval}};
\node[label, right=0.3cm of s7] {$\sim 1$};
\end{tikzpicture}
\caption{Drug discovery pipeline funnel. Each stage progressively filters compounds, with numbers indicating typical compound counts. Our \ours{} method (green box) operates at the virtual screening stage, improving the quality of top-ranked candidates before expensive downstream validation. By identifying invariant binding determinants and suppressing spurious correlations, our method aims to increase the hit rate in subsequent experimental testing.}
\label{fig:pipeline}
\end{figure}
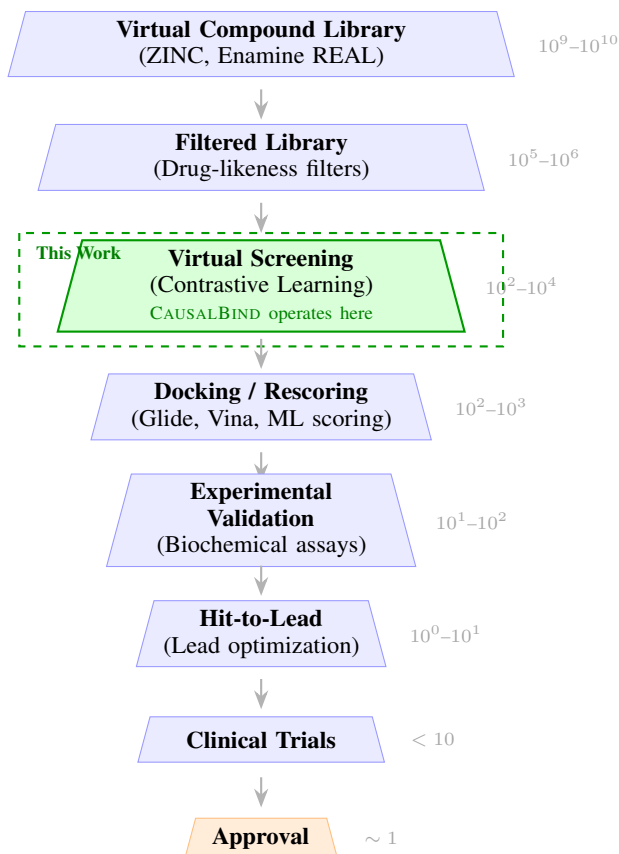

\paragraph{Why virtual screening matters.} The virtual-screening stage is where the paper's evaluation metrics become practically meaningful. This stage is responsible for the largest candidate reduction before wet-lab testing, so the quality of the \emph{very top} of the ranking matters far more than average prediction quality over a fixed benchmark. This is exactly why the main text emphasizes early-enrichment metrics such as BEDROC$_{80.5}$ and EF@1\%: even a modest improvement there can substantially reduce the number of false positives advanced to docking or assay validation. Our focus on sparse binding determinants is intended to improve this high-value early-filtering regime.

\paragraph{Complementary to docking.} This positioning also clarifies why the paper compares docking, structure-aware scoring, and retrieval as distinct regimes rather than as direct substitutes. Docking remains valuable because it reasons explicitly about poses and local geometry, but it is too expensive to serve as the first-pass search over very large libraries. Retrieval-style methods occupy a different operating point: they trade mechanistic detail for scale and are therefore best used upstream of docking. \ours{} fits this role. Its contribution is to improve the representation quality of that pre-filtering stage so that the compounds passed downstream are more likely to contain true binders.

\section{Extended Related Work}
\label{app:related}

\paragraph{From docking to scalable screening.} Classical virtual screening is dominated by docking pipelines such as Glide, AutoDock Vina, and Surflex, which explicitly search binding poses and rank compounds using physics-inspired scoring functions~\citep{glide2004,vina2010,surflex2012}. Their strengths are interpretability and direct geometric reasoning about binding modes, but their cost grows quickly with library size and target flexibility. In practical workflows this often makes docking a \emph{downstream} stage, applied after an initial coarse filter rather than as the first-pass search over millions of compounds. This is exactly the regime targeted in the main text: large-scale pre-filtering, where retrieval speed and early enrichment matter more than full pose resolution.

\paragraph{Learned scoring functions and pose-aware models.} A large literature replaces hand-engineered docking scores with learned interaction models over protein-molecule complexes. Gnina introduced deep learning into docking-time scoring with 3D CNNs~\citep{gnina2021}; RTMScore, GenScore, EquiScore, PIGNet, and PLANET further strengthened this line with graph, transformer, equivariant, and physics-informed architectures~\citep{rtmscore2023,genscore2023,equiscore2023,pignet2022,planet2024}. Related structure-based models such as TankBind and BigBind learn from geometric complexes or large-scale structural supervision to improve affinity prediction and screening robustness~\citep{tankbind2022,bigbind2022}. These methods are highly relevant baselines because they capture detailed intermolecular geometry, but they usually assume a candidate complex, pose, or structure-conditioned interaction graph at inference time. By contrast, our focus is on \emph{retrieval-style} virtual screening where compound scoring must remain fast enough for large libraries.

\paragraph{Affinity prediction beyond screening.} Drug--target affinity (DTA) models such as DeepDTA and GraphDTA predict binding strength from sequence or graph representations without necessarily producing an efficient retrieval system~\citep{deepdta2018,graphdta2021}. This line clarified how learned molecular and protein embeddings can support affinity estimation, but it is typically evaluated on regression-style benchmarks rather than target-wise early enrichment. The distinction matters for our paper because the objective is not simply to estimate affinity well on average; it is to identify which compounds should be placed at the very top of the screening list.

\paragraph{Retrieval-based contrastive learning for virtual screening.} The recent shift most relevant to our work is the move from explicit complex scoring to \emph{shared embedding spaces}. DrugCLIP~\citep{drugclip2023} framed protein-molecule virtual screening as dense retrieval, using Uni-Mol-based encoders and CLIP-style contrastive learning~\citep{clip2021,zhou2023uni}. This design substantially improved scalability and helped establish retrieval-based virtual screening as a viable alternative to docking, a direction later amplified by a genome-scale \emph{Science} study built on DrugCLIP-style contrastive screening~\citep{jia2026deep}. Subsequent methods enriched this paradigm in different ways: DrugHash compresses embeddings into binary codes for memory-efficient screening~\citep{drughash2025}; LigUnity unifies virtual screening with hit-to-lead affinity ranking through scaffold discrimination and pharmacophore ranking in a shared pocket--ligand space~\citep{ligunity2024}; HypSeek argues that hyperbolic geometry better captures hierarchical and fine-grained affinity variation than Euclidean retrieval alone~\citep{hypseek2025}; \rev{and the concurrent S$^2$Drug bridges protein sequence and 3D structure through ChEMBL sequence pretraining and PDBBind fine-tuning with residue-level sequence--structure fusion and an auxiliary binding-site task~\citep{s2drug2026}}. These papers collectively show that the field has moved from asking whether retrieval can work at all to asking what kind of structural bias should govern the shared space. The present paper answers that question by shifting the focus from dense holistic alignment to sparse or low-rank cross-modal interaction.

\paragraph{Benchmark bias and generalization.} A recurring concern in virtual screening is that benchmark artifacts can inflate apparent performance. Hidden biases in DUD-E and related datasets have long been known to create shortcuts based on simple physicochemical cues rather than genuine binding recognition~\citep{chen2019hidden}. More recent analyses argue that debiasing is essential for robust binding-affinity generalization across targets and chemical series~\citep{graber2025resolving}. This motivates the paper's emphasis on LIT-PCBA and on early-enrichment gains under more realistic decoys: the goal is not merely to improve in-distribution ranking, but to learn representations that remain useful when easy shortcuts disappear.

\paragraph{Why three regimes? Core computational pattern.} Although all three regimes ultimately rank protein-molecule pairs, their \emph{inference-time computational pattern} differs sharply, and this is what determines their practical scaling behaviour and where their inductive bias has to live. The defining feature of retrieval-based methods is that protein and molecule embeddings are computed \emph{independently} (a dual-tower architecture); the binding score then reduces to a similarity, typically a dot product, between two pre-computable vectors, exactly as in dense passage retrieval in NLP~\citep{infonce2018,clip2021} or image--text retrieval. This pattern is what unlocks very large-scale screening, but it also constrains where cross-modal interaction can be expressed: the interaction has to be encoded into the representation or the similarity function, rather than being computed by a joint pair-input network at inference time. This is precisely where \ours{} intervenes. Instead of leaving the retrieval score as a dense holistic similarity, we impose sparse concept-level interaction inside the retrieval regime itself. Table~\ref{tab:regime_compare} makes the contrast explicit; Table~\ref{tab:related_compare} then drills down to representative methods.

\begin{table}[h]
\caption{Core differences between the three virtual-screening regimes. The retrieval regime decouples encoders into two towers, enabling pre-computed library embeddings and millisecond-scale similarity scoring; the other two regimes condition on the (protein, ligand) pair at inference. \ours{}$_{\mathrm{SP}}$ stays in the retrieval regime but replaces the dense holistic similarity with a sparse bipartite concept-pair interaction.}
\label{tab:regime_compare}
\centering\small\setlength{\tabcolsep}{4pt}
\renewcommand{\arraystretch}{1.20}
\begin{tabular}{@{}>{\raggedright\arraybackslash}p{0.15\textwidth}|>{\raggedright\arraybackslash}p{0.25\textwidth}|>{\raggedright\arraybackslash}p{0.25\textwidth}|>{\raggedright\arraybackslash}p{0.25\textwidth}@{}}
\toprule
 & \textbf{Physics-based docking} & \textbf{Structure-aware ML scoring} & \textbf{Retrieval-based learning} \\
\midrule
Inference input & (protein, ligand) pair + pose search & (protein, ligand) pair, often pose-conditioned & protein and ligand encoded \emph{independently} \\
Binding score & physics + empirical scoring on the pair & one forward pass through a joint network on the pair & dot product (or cosine) of two precomputed vectors \\
Pre-computable embeddings? & No (pose-coupled) & No (pair-coupled) & \textbf{Yes} (dual-tower) \\
Cost per (P, L) at inference & seconds--minutes (pose search dominates) & $\sim$10--100\,ms (one full forward) & sub-millisecond (one dot product) \\
Practical library scale & $10^4$--$10^5$ & $10^5$--$10^6$ & $\mathbf{10^8\text{--}10^9}$ \\
Where the inductive bias lives & force field + sampling protocol & complex graph / 3D-voxel architecture & shape of the shared embedding space and the similarity used on it \\
Typical output & docked pose + affinity score & affinity score (sometimes a pose) & ranked retrieval list \\
\midrule
Where \ours{}$_{\mathrm{SP}}$ sits & N/A & N/A & retrieval regime, but with a \emph{sparse bipartite} similarity over concept pairs instead of a dense one \\
\bottomrule
\end{tabular}
\end{table}

\paragraph{Cross-regime comparison.} Building on the regime-level distinctions of Table~\ref{tab:regime_compare}, Table~\ref{tab:related_compare} drills down to representative methods within each regime and highlights the design choice that distinguishes \ours{} from prior retrieval baselines: not a new backbone family, but a different structural prior on cross-modal interaction.

\begin{table}[h]
\caption{Representative methods in the three virtual-screening regimes discussed in the main text. The table highlights the main design choice of each method and the role it plays in the comparison, with \ours{}$_{\mathrm{SP}}$ differing from prior retrieval baselines by imposing a sparse bipartite interaction prior rather than a dense holistic similarity.}
\label{tab:related_compare}
\centering\small\setlength{\tabcolsep}{4pt}
\renewcommand{\arraystretch}{1.15}
\begin{tabular}{@{}>{\raggedright\arraybackslash}p{0.20\textwidth}|>{\raggedright\arraybackslash}p{0.36\textwidth}|>{\raggedright\arraybackslash}p{0.36\textwidth}@{}}
\toprule
\textbf{Method} & \textbf{Core Design} & \textbf{Relevance Here} \\
\midrule
\multicolumn{3}{@{}l}{\emph{Physics-based docking systems}} \\
\midrule
Glide-SP~\citep{glide2004} & Hierarchical sampling with empirical scoring function & Industry standard; accurate on small candidate sets but expensive at library scale \\
AutoDock Vina~\citep{vina2010} & Empirical scoring + iterated local search & Open-source, widely used; not designed for ultra-large libraries \\
\midrule
\multicolumn{3}{@{}l}{\emph{Structure-aware machine-learning scoring models}} \\
\midrule
DeepDTA~\citep{deepdta2018} & Sequence/graph-based affinity prediction without pairwise retrieval design & Useful DTA baseline; not optimized for target-wise early enrichment \\
Gnina~\citep{gnina2021} & 3D CNN scoring on docking-grid voxels & Strong pose-rescoring; depends on docked poses \\
RTMScore~\citep{rtmscore2023} & Residue-level GNN re-scoring of poses & Good DUD-E AUROC; needs upstream pose enumeration \\
GenScore~\citep{genscore2023} & Generalist pose generation + rescoring net & Unifies pose and score with one network \\
EquiScore~\citep{equiscore2023} & SE(3)-equivariant GNN on the binding complex & Geometric inductive bias on the complex graph \\
PIGNet~\citep{pignet2022} & Physics-informed GNN combining empirical + learned terms & Hybrid physics/ML score; strong DUD-E early enrichment \\
TankBind~\citep{tankbind2022} & Trigonometry-aware GNN, end-to-end pose+affinity & Joint binding-site prediction and scoring \\
BigBind~\citep{bigbind2022} & GNN trained on large-scale curated binding corpus & Stresses the role of training-data scale and curation \\
\midrule
\multicolumn{3}{@{}l}{\emph{Retrieval-based representation learning}} \\
\midrule
DrugCLIP~\citep{drugclip2023} & CLIP-style contrastive on Uni-Mol encoders & Scalable; adopted at genome scale~\citep{jia2026deep} \\
DrugHash~\citep{drughash2025} & Binary-hash retrieval extension of DrugCLIP & Memory-efficient at billion-compound scale; weaker PCBA AUROC \\
LigUnity~\citep{ligunity2024} & Unified screening + hit-to-lead ranking loss & Strong BEDROC/EF on both DUD-E and LIT-PCBA \\
HypSeek~\citep{hypseek2025} & Hyperbolic geometry + cone hierarchy + ESM-2 sequence branch & Prior SOTA on DUD-E and LIT-PCBA; geometry-driven inductive bias \\
\rowcolor{green!7}\textbf{\ours{}$_{\mathrm{SP}}$ (ours)} & \textbf{Sparse bipartite concept-pair interaction over Perceiver concepts} & \textbf{Theoretically grounded under Heckman selection; best LIT-PCBA early enrichment} \\
\bottomrule
\end{tabular}
\end{table}

\paragraph{Detailed positioning.} The methods in Table~\ref{tab:main_results} can be read as a progression in where the binding inductive bias is placed. Glide-SP is the classic physics-based reference point: it reasons explicitly about poses and local contacts, but at a computational cost that makes it unsuitable for first-pass screening over very large libraries. DeepDTA occupies a different corner of the design space: it is important historically because it showed that learned molecular and protein representations can support affinity prediction, but it is not a retrieval-oriented screening model and therefore is not optimized for the target-wise early-enrichment setting emphasized in this paper.

Among structure-aware learning-based methods, Gnina, RTMScore, GenScore, EquiScore, TankBind, and BigBind all retain a joint pair-input view of the protein-molecule problem. Gnina and RTMScore are closely tied to docked-pose scoring, inheriting much of the pose dependence of classical docking. GenScore and EquiScore strengthen this line with more expressive learned complex representations, while TankBind pushes further toward end-to-end geometric reasoning over binding sites and complexes. BigBind is especially notable for highlighting the importance of large-scale curated structural supervision. These methods are relevant because they capture explicit intermolecular geometry, but their inference-time cost and pair-conditioned architecture make them fundamentally different from the large-scale retrieval regime studied in our main experiments.

The retrieval baselines are the closest comparators. DrugCLIP established the modern dense-retrieval formulation for virtual screening and demonstrated that CLIP-style contrastive learning with Uni-Mol encoders can scale to very large libraries. DrugHash preserves that same dense-alignment philosophy while prioritizing storage and retrieval efficiency through binary hashing, at the cost of some predictive fidelity. LigUnity extends dense shared-space retrieval by coupling virtual screening with hit-to-lead ranking signals, which helps its early-enrichment behavior but does not change the fact that the interaction remains holistic in the shared space. HypSeek is the strongest baseline in this family: it enriches retrieval with hyperbolic geometry, cone-hierarchy regularization, and a protein-sequence branch, yielding a powerful geometric inductive bias but still no explicit structural prior on which cross-modal interactions should matter.

This is the point at which \ours{} differs. We stay in the retrieval regime and keep the scalability advantages of independently encoded molecule and protein representations, but we change the structural bottleneck through which binding information must pass. Compared with docking or pose-aware scoring, \ours{} does not attempt to replace explicit structural modeling; instead it improves the \emph{front-end screening} stage that decides which compounds deserve expensive downstream analysis. Compared with DrugCLIP, DrugHash, LigUnity, and HypSeek, we leave the high-level retrieval pipeline intact but modify the representation bottleneck: a shared \textit{ConceptExtractor} decomposes each modality into atomic concepts, and a sparse cross-modal mask highlights which concept pairs are causally relevant for binding. Semantically, prior retrieval methods ask what inductive bias should structure the embedding \emph{space} and answer with metric geometry, compression, or training-time objectives. \ours{} instead asks what inductive bias should structure the \emph{interaction between} molecule and protein representations, and answers with a sparse or low-rank bipartite kernel grounded in the V-structure causal model of Section~\ref{sec:theory}. In that sense, \ours{} is best understood not as a new backbone family, but as an orthogonal structural prior inside retrieval-based virtual screening itself.

\paragraph{Causal representation learning.} The conceptual basis for our method comes from causal representation learning~\citep{scholkopf2021causal,hyvarinen2019nonlinear}, especially work showing that multimodal observations and sparse graphical structure can improve identifiability of latent variables~\citep{daunhawer2023identifiability,kivva2021learning}. ConceptAligner~\citep{conceptaligner2025} operationalized this intuition in a multimodal generative setting, while \citet{suncausal} extended related ideas to biomedical multimodal data with identifiability guarantees for the latent causal variables. We adapt this literature to protein-molecule modeling under a different causal graph: unlike text-to-image generation, neither modality generates the other; instead molecule and protein concepts jointly determine a binding event through a selected collider. This shift from a directional chain to a V-structure is what motivates the sparse bipartite mask at the center of \ours{}.

\paragraph{\rev{Causal discovery and selection bias.}} \rev{Complementary to causal representation learning, causal discovery recovers causal graphs among observed variables, with recent advances that handle deterministic relations~\citep{li2024causal}, decentralized heterogeneous data~\citep{li2024federated}, and scalable DAG structure learning~\citep{ziu2024psi}. A related line of work studies how selection distorts the observed distribution: causal graphs with latent confounders and post-treatment selection can be characterized and learned from interventional data~\citep{luo2026characterization}, gene regulatory networks can be inferred in the presence of selection bias~\citep{luo2026gene}, and selective mechanisms have also been exploited in text generation~\citep{tang2025reflection}. Our setting shares this concern with selection: training data contain only observed binding pairs, and \ours{} explicitly models this binder-only selection as conditioning on a collider when deriving which interaction structures remain identifiable.}

\section{Biological Observations of Sparse Binding}
\label{app:biological}

\paragraph{\rev{Hot spots and sparse binding energy.}} \rev{A direct biological motivation for our sparse-interaction view comes from the classical hot-spot literature on molecular recognition. In the human growth hormone--receptor system, \citet{clackson1995hot} combined the crystal structure of the complex with alanine-scanning mutagenesis and found that many side chains participate in the physical interface, but the energetic contribution is highly concentrated: a central hydrophobic region on the receptor side, dominated by two tryptophan residues, accounts for more than three quarters of the binding free energy. This result separates the \emph{structural epitope}, namely the broad set of contacts visible in the complex, from the \emph{functional epitope}, namely the much smaller subset of contacts that actually maintains affinity. \citet{bogan1998anatomy} then showed that this concentration of binding energy is a general property of protein interfaces rather than an isolated example: hot spots are formed by a small subset of residues, are enriched in tryptophan, tyrosine, and arginine, and are often surrounded by less energetic residues that shield the hot spot from bulk solvent. These findings support the central modeling assumption of \ours{}: the binding signal should not be represented as a uniformly dense agreement between whole protein and molecule embeddings. Instead, it is more natural to model binding through a limited number of cross-modal compatibilities, while treating much of the remaining surface or molecular context as supportive, redundant, or potentially nuisance-correlated.}

\paragraph{Pharmacophores as invariant binding determinants.} Even though induced fit allows conformational adjustment, medicinal chemistry has long emphasized that certain feature types remain consistently binding-relevant across conformational states. These are the \emph{pharmacophoric} features~\citep{wermuth2006pharmacophore}: hydrogen-bond donors and acceptors, hydrophobic centers, aromatic groups, charged moieties, and related steric/electronic motifs that repeatedly determine whether a ligand can engage a target pocket productively. The exact geometry of a bound complex may shift, but these feature types remain the stable determinants that make those local interactions possible. In the language of this paper, they are the biologically plausible candidates for the ``atomic concepts'' that should survive across proteins, ligands, and conformational variants.

\paragraph{Why static structures are still informative.} A natural concern is that if induced fit is dynamic, then learning from static structures might miss the relevant mechanism. The key point is that our goal is not to recover every conformational degree of freedom; it is to recover the \emph{interaction-relevant determinants} that persist across such variability. A hydroxyl group remains a hydrogen-bond donor; an aromatic ring remains capable of $\pi$-stacking; a hydrophobic patch remains nonpolar even if side chains rotate slightly around it. Benchmarks such as DUD-E and LIT-PCBA are built from static structures, yet successful virtual-screening methods already rely on the fact that these chemically meaningful determinants are more stable than the incidental conformational details. Our method does not remove this assumption; it makes it more explicit by organizing the representation around sparse cross-modal interactions rather than around dense global similarity.

\paragraph{From biology to the modeling choice in \ours{}.} This biological picture motivates the two central methodological choices in the paper. First, if binding depends on a small set of recurring local determinants, then each modality should be decomposed into smaller concept-like units rather than treated as a single holistic vector. Second, if only a subset of those units actually participates in binding for a given target family, then the cross-modal interaction should itself be sparse or low-rank rather than dense. This is exactly what the \textit{ConceptExtractor} and \textit{SparseMask} are designed to implement. The model is therefore not claiming that every learned concept is a textbook pharmacophore; rather, it is imposing an inductive bias that is much closer to the biological mechanism than dense shared-space alignment.

\paragraph{Relevance to generalization and drug discovery.} This interpretation also clarifies why the largest gains tend to appear on the more bias-controlled LIT-PCBA benchmark. If a method relies on shortcut correlations, it may rank compounds well on easier datasets without truly isolating binding determinants. By contrast, a representation organized around invariant local compatibilities should transfer better to novel targets and novel scaffolds, because those compatibilities are closer to the actual mechanism medicinal chemists reason about during lead optimization. In this sense, the biological motivation of \ours{} is not only explanatory but also practical: it supports the paper's broader claim that sparse concept-level interaction is a more plausible route to robust virtual screening than dense holistic alignment.

\section{Theoretical Framework: Selection and Identifiability}
\label{app:theory}

\rev{This appendix provides the complete theoretical framework for \ours{}. We cast protein-molecule binding prediction as a \emph{Heckman-style selection problem} on a V-structure causal graph: the training set consists only of binding pairs ($B=1$), and our goal is to identify the latent concepts $\mathbf{c}^M, \mathbf{c}^P$ and the bipartite interaction mask $\mathcal{M}$ from this selection-biased observation. Following the didactic organization of SEDGE~\citep{sedge2026} and ConceptAligner~\citep{conceptaligner2025}, we proceed from a careful SCM formulation (\ref{app:theory:scm}) through a worked $K_M=K_P=2$ toy case (\ref{app:theory:toy}) to the general non-identifiability and identifiability results (\ref{app:theory:non-id}--\ref{app:theory:main}), a low-rank subspace-identifiability result (\ref{app:theory:lowrank}), and an explicit comparison to SEDGE (\ref{app:theory:sedge}). Each theorem subsection is organized into a proof sketch, an intuitive example, and a full proof, and a later discussion in the experimental appendix (\ref{app:theory:gap}) returns to how these identifiability results should be interpreted in light of the learned masks.}

Our framework borrows the Heckman-selection machinery recently formalized for vision--language concept generation by \citet{conceptaligner2025} and for compositional extrapolation by SEDGE, but differs in four crucial ways: (i) we have a \emph{collider} causal structure (V-structure) rather than a chain; (ii) we have two \emph{symmetric} observation sources (molecule and protein) rather than one feature vector; (iii) our binding kernel is bilinear-sparse rather than a parametric MLP head; and (iv) we replace the SEDGE-style ``conservative'' assumption ($P(S{=}0) \!\approx\! 0$) with a \emph{low-rank binding kernel} assumption that is better suited to drug discovery, where most protein-molecule pairs do not bind.

\subsection{Structural Causal Model}
\label{app:theory:scm}

\paragraph{Variables and generative process.} Let $m \in \mathcal{M}$ and $p \in \mathcal{P}$ denote an observed molecule (3D structure) and protein (3D structure). We posit that each modality is generated from latent \emph{concepts}:
\begin{equation}
{\color{black}\mathbf{c}^M = [c^M_1, \ldots, c^M_{K_M}] \in \mathbb{R}^{K_M}, \qquad \mathbf{c}^P = [c^P_1, \ldots, c^P_{K_P}] \in \mathbb{R}^{K_P},}
\label{eq:concepts}
\end{equation}
\rev{through observation maps with modality-specific nuisance noise}
\begin{equation}
{\color{black}m = g^M(\mathbf{c}^M,\varepsilon^M), \qquad p = g^P(\mathbf{c}^P,\varepsilon^P),}
\label{eq:generators}
\end{equation}
\rev{where $\varepsilon^M$ and $\varepsilon^P$ capture observation noise, conformational variability, or other modality-specific nuisance factors that are independent of the cross-modal binding event given the concepts. We do not require $g^M$ and $g^P$ to be globally invertible on the full observation spaces. The recoverability assumption below only requires that the interaction-relevant concepts can be expressed as functions of their corresponding observations on the data support.}
Biologically, each $c^M_i$ (resp.\ $c^P_j$) captures an \emph{atomic} chemical/structural feature of the molecule (protein), analogous to a pharmacophoric element (hydrogen-bond donor, hydrophobic contact, aromatic ring, etc.).

\paragraph{Binding as a collider.} The binding indicator $B \in \{0, 1\}$ is a \emph{collider}: it is determined jointly by the two concept vectors through a sparse bipartite interaction,
\begin{equation}
\label{eq:binding}
{\color{black}B \sim \mathrm{Bernoulli}\!\left( \sigma\!\left( \sum_{i=1}^{K_M} \sum_{j=1}^{K_P} \mathcal{M}_{ij}\,\phi(c^M_i, c^P_j) \right) \right),}
\end{equation}
\rev{where $\mathcal{M} \in \mathbb R_{\ge0}^{K_M \times K_P}$ is the sparse \emph{bipartite binding mask}; binary masks and $[0,1]$ relaxations are special normalizations, while scale is shared with $\phi$. Here $K_M$ and $K_P$ count concepts, and each atomic concept is scalar in this theoretical model, i.e., $c_i^M,c_j^P\in\mathbb{R}$. Thus $\phi: \mathbb{R} \times \mathbb{R} \to \mathbb{R}$ maps one molecule concept and one protein concept to a scalar logit-level compatibility score; $\sigma$ is the logistic.} The graph is
\begin{equation}
\label{eq:v-structure}
{\color{black}\varepsilon^M \rightarrow m \leftarrow \mathbf{c}^M \rightarrow B \leftarrow \mathbf{c}^P \rightarrow p \leftarrow \varepsilon^P,}
\end{equation}
exhibiting a V-structure with $B$ as the collider (Figure~\ref{fig:causal_graph} in the main text).

\paragraph{Bayesian network and conditional-independence structure.} We treat $(\mathbf{c}^M, \mathbf{c}^P, \varepsilon^M,\varepsilon^P, B, m, p)$ as a Bayesian network~\citep{pearl1985bayesian,koller09probabilistic} with respect to the DAG in Eq.~\eqref{eq:v-structure}. The local Markov property yields the following conditional independence (CI) relations, which we use repeatedly in proofs:
\begin{align}
\mathbf{c}^M \perp\!\!\!\perp \mathbf{c}^P \quad &\text{(marginal independence from V-structure)}, \label{eq:ci-marginal}\\
{\color{black}m \perp\!\!\!\perp \{\mathbf{c}^P, \varepsilon^P, B, p\} \mid (\mathbf{c}^M,\varepsilon^M)} \quad &\text{\rev{(molecule view depends only on molecule concepts and noise)}}, \label{eq:ci-m}\\
{\color{black}p \perp\!\!\!\perp \{\mathbf{c}^M, \varepsilon^M, B, m\} \mid (\mathbf{c}^P,\varepsilon^P)} \quad &\text{\rev{(protein view depends only on protein concepts and noise)}}, \label{eq:ci-p}\\
{\color{black}B \perp\!\!\!\perp \{m, p, \varepsilon^M,\varepsilon^P\} \mid (\mathbf{c}^M, \mathbf{c}^P)} \quad &\text{\rev{(binding mediated entirely by concepts)}}. \label{eq:ci-b}
\end{align}

\paragraph{Heckman selection.} The training corpus $\mathcal{D}$ contains only positive binding pairs, i.e., samples drawn from
\begin{equation}
\label{eq:selection}
p^{\mathcal{D}}(m, p, \mathbf{c}^M, \mathbf{c}^P) \;=\; p\!\left(m, p, \mathbf{c}^M, \mathbf{c}^P \mid B=1\right).
\end{equation}
\rev{Equivalently, the full selected joint before marginalizing nuisance variables is}
\begin{equation}
\label{eq:selection-with-noise}
{\color{black}p^{\mathcal{D}}(m, p, \mathbf{c}^M, \mathbf{c}^P,\varepsilon^M,\varepsilon^P)
=p\!\left(m, p, \mathbf{c}^M, \mathbf{c}^P,\varepsilon^M,\varepsilon^P \mid B=1\right).}
\end{equation}
In Heckman-selection terminology~\citep{heckman1979}, $B$ plays the role of the selection variable $S$: we observe $(x^M, x^P)$ if and only if $B=1$. Because $B$ is a collider in Eq.~\eqref{eq:v-structure}, conditioning on $B=1$ induces the classical \emph{explaining-away} dependence between $\mathbf{c}^M$ and $\mathbf{c}^P$ even though they are marginally independent. Explicitly, using Eq.~\eqref{eq:ci-marginal} and Bayes' rule:
\begin{equation}
p^\mathcal{D}(\mathbf{c}^M, \mathbf{c}^P) \;:=\; p(\mathbf{c}^M, \mathbf{c}^P \mid B=1) \;=\; \frac{p(\mathbf{c}^M)\,p(\mathbf{c}^P)\,P(B=1\mid \mathbf{c}^M, \mathbf{c}^P)}{P(B=1)}.
\label{eq:selection-factorization}
\end{equation}
This is the key handle: the observed joint distribution carries information about $(\mathcal{M}, \phi)$ through the explaining-away factor $P(B=1\mid \mathbf{c}^M, \mathbf{c}^P)$.

\paragraph{Notation.} Following SEDGE's conventions~\citep{sedge2026}, uppercase letters denote random variables and lowercase their realizations. We write $p^\mathcal{D}(\cdot) := p(\cdot \mid B=1)$ for distributions implied by the (selection-biased) training data, and reserve $p(\cdot)$ for population distributions (which may involve the unobserved $B=0$ regime). We use $\propto$ for equality up to a positive multiplicative constant and $\mathbf{1}$ for an all-ones vector. \rev{Supports and densities are assumed to exist with respect to common dominating measures (typically Lebesgue in $\mathbb{R}^{K_M}$ and $\mathbb{R}^{K_P}$, or counting in the corresponding binary spaces).}

\subsection{Discussion of Assumptions}
\label{app:theory:assumptions}

\rev{The formal assumptions and theorem-specific conditions are stated in the main text, so we do not repeat them here. Instead, this subsection explains what each assumption or condition contributes to the proofs and how strong it should be read. We keep the main-text notation throughout: Assumptions~\ref{ass:concept-recoverability}--\ref{ass:positivity-coverage} are shared by all theorems, Assumption~\ref{ass:loglinear} is used by the two positive results, Conditions~2.i--2.iii belong to Theorem~\ref{thm:main:ident}, and Conditions~3.i--3.ii belong to Theorem~\ref{thm:main:lowrank}.}

\begin{itemize}[leftmargin=1.6em, itemsep=4pt]
    \item \rev{\textbf{Assumption~\ref{ass:concept-recoverability} (concept recoverability).} This assumption says that the concepts entering the binding mechanism are statistically accessible from the observed molecule or protein. It should not be read as requiring a global inverse for the full observation maps $g^M,g^P$ or for the neural encoders used in implementation. The proofs use it only to pull the binder-selected observation distribution $p^\mathcal D(x^M,x^P)$ back to concept coordinates on the data support.}

    \item \rev{\textbf{Assumption~\ref{ass:marginal-independence} (prior independence).} This is the V-structure assumption: before conditioning on binding, molecule-side and protein-side concepts are independent objects. The observed dependence in $p^\mathcal D$ then comes from conditioning on the collider $B=1$. Without this assumption, cross-modal dependence could already be present before selection, and the proof could not attribute it to the binding mechanism.}

    \item \rev{\textbf{Assumption~\ref{ass:positivity-coverage} (positivity / coverage).} This is the minimal support requirement. If a concept never appears in binder-selected data, it leaves no statistical trace and cannot be identified by any method. The assumption does not require uniform coverage of all concept combinations; it only rules out concepts that are completely absent from the selected support.}

    \item \rev{\textbf{Assumption~\ref{ass:loglinear} (log-linear selection).} This assumption makes the binding score itself, rather than $\log\sigma$ of it, appear in the selected log-density, so that $\log p^\mathcal D$ equals $S$ up to additive functions of a single modality. It matters for the proofs: under the logistic link, the cross-derivative of $\log\sigma(S)$ contains the term $-\sigma(S)(1-\sigma(S))\,\partial_{c_i^M}S\,\partial_{c_j^P}S$, which is dense even when $\mathcal M$ is sparse, so the support of $\mathcal M$ would not be visible in the object the proof differentiates. The assumption is the low-hit-rate form of the logistic model, since $\sigma(s)=e^{s}(1+O(e^{s}))$ as $s\to-\infty$ and most protein-molecule pairs do not bind, and it is also the model under which the contrastive objective of Section~\ref{sec:method} is well specified: the optimal contrastive critic is the log density ratio up to terms depending on one modality~\citep{infonce2018}, which under this assumption is exactly $S$ up to such terms.}

    \item \rev{\textbf{Condition~2.i (sparse antichain interaction).} Sparsity alone is not enough for component-wise identifiability: two concepts can still be confounded if one concept's interaction support is absorbed by another's. The antichain requirement gives each molecule-side and protein-side concept a distinct combinatorial footprint in the bipartite interaction graph. It is the combinatorial input of Lemma~\ref{lem:antichain}, and Example~\ref{ex:lemma}(b) shows that it cannot be dropped.}

    \item \rev{\textbf{Condition~2.ii (non-bilinear compatibility).} This condition makes the cross-derivative of each interaction term vary over the support, so that matching the two selected distributions yields a family of matrix equations spanning all matrices supported on $\mathrm{supp}(\mathcal M)$, rather than the single equation available in the bilinear case. It is a generic property of $\phi$ alone and does not depend on the mask: $\phi(x,y)=\tanh(xy)$ satisfies it, whereas $\phi(x,y)=xy$ and $\phi(x,y)=-(x-y)^2$ do not. Without it, Theorem~\ref{thm:main:nonid} provides counterexamples even for sparse antichain masks, which is why the inner-product compatibility used in our implementation falls under Theorem~\ref{thm:main:lowrank} rather than Theorem~\ref{thm:main:ident}.}

    \item \rev{\textbf{Condition~2.iii (linear ambiguity and sparsity minimality).} The first clause restricts the competing parameterizations to the linear reparameterizations of Theorem~\ref{thm:main:nonid}. Linear identifiability is the usual first stage of identifiability analyses and holds for representations learned with contrastive objectives~\citep{roeder2021linear}; Proposition~\ref{prop:nonlinear} removes it at the price of a stronger, pointwise nonlinearity condition. The second clause requires that the data not be explained equally well by a denser mask. It is the population counterpart of the sparsity penalty on the mask in the training objective, and Example~\ref{ex:lemma}(c) shows that it cannot be dropped.}

    \item \rev{\textbf{Condition~3.i (low-rank interaction kernel).} This is the structural condition for the subspace theorem. It replaces exact concept-pair support recovery with the assumption that the binding score has a minimal rank-$r$ separable factorization. Thus the target of identification changes: we no longer seek a unique entry-wise mask, but the molecule-side and protein-side subspaces spanned by the low-rank interaction functions.}

    \item \rev{\textbf{Condition~3.ii (low-rank sufficient variability).} This is the statistical richness condition for the subspace theorem, not another low-rank constraint. It requires the mode-evaluation matrices to have full rank at finitely many molecule-side and protein-side probes. This guarantees that the low-rank modes specified by Condition~3.i are sufficiently excited on the selected support, so their spans can be recovered from the observed distribution.}
\end{itemize}

\subsection{Toy Case: Two Concepts}
\label{app:theory:toy}

Before the general theorems, we work out the simplest non-trivial case explicitly, following SEDGE's didactic strategy. \rev{Assume $K_M=K_P=2$ with continuous scalar concepts, binding score $S(\mathbf c^M,\mathbf c^P)=\sum_{i,j=1}^{2}\mathcal M_{ij}\phi(c^M_i,c^P_j)$, and log-linear selection (Assumption~\ref{ass:loglinear}), so that two parameterizations induce the same $p^\mathcal D$ exactly when their binding scores agree up to additive functions of a single modality. Write $\mathbf c^M=(c_1,c_2)$ and $\mathbf c^P=(c_1',c_2')$.}

\paragraph{\rev{Case (a): Diagonal mask with bilinear compatibility (not identifiable).}} \rev{Let $\mathcal M=I_2$ and $\phi(x,y)=xy$, so $S=c_1c_1'+c_2c_2'$. The row supports are $\{1\}$ and $\{2\}$, an antichain, and the mask is as sparse as it can be while leaving every concept active. Nevertheless, with}
\[
{\color{black}
T_M=\begin{bmatrix}1&1\\0&1\end{bmatrix},\qquad T_P=T_M^{-\top}=\begin{bmatrix}1&0\\-1&1\end{bmatrix},
}
\]
\rev{the concepts $\tilde{\mathbf c}^M=T_M\mathbf c^M=(c_1+c_2,\,c_2)$ and $\tilde{\mathbf c}^P=T_P\mathbf c^P=(c_1',\,c_2'-c_1')$ satisfy $\tilde{\mathbf c}^{M\top}I_2\,\tilde{\mathbf c}^P=\mathbf c^{M\top}T_M^\top T_M^{-\top}\mathbf c^P=S$. The same mask $I_2$ therefore explains the data with mixed concepts: sparsity and the antichain property alone do not identify the concepts. This is the smallest instance of the second part of Theorem~\ref{lem:non-id}.}

\paragraph{\rev{Case (b): Diagonal mask with non-bilinear compatibility (identifiable).}} \rev{Keep $\mathcal M=I_2$ but let $\phi(x,y)=\tanh(xy)$, which satisfies Condition~2.ii. The cross-derivative of channel $i$, namely $\partial_1\partial_2\phi(c_i,c_i')$, now varies with $(c_i,c_i')$, so the two channels generate a family of diagonal cross-derivative matrices spanning all diagonal matrices. The mixing of Case~(a) then fails: the transformed score $\tanh((c_1+c_2)c_1')+\tanh(c_2(c_2'-c_1'))$ has a nonzero cross-derivative in $(c_2,c_1')$, which no score with mask $I_2$ has in the original coordinates. Theorem~\ref{thm:identifiability} shows that every affine mixing is excluded in this way, leaving only permutations and rescalings of the two channels.}

\paragraph{\rev{Case (c): Nested supports.}} \rev{Let $\mathcal M=\big[\begin{smallmatrix}1&0\\1&1\end{smallmatrix}\big]$, so $R_1=\{1\}\subseteq R_2=\{1,2\}$ and Condition~2.i fails. Then $P=\big[\begin{smallmatrix}1&0\\a&1\end{smallmatrix}\big]$ with $a\neq0$ maps every matrix supported on $\mathrm{supp}(\mathcal M)$ to a matrix with the same support, so the combinatorial step of the proof (Lemma~\ref{lem:antichain}) cannot exclude adding a multiple of the first concept's cross-derivatives to the second's.}

\paragraph{Case (d): Dense mask with bilinear compatibility.} Suppose $\mathcal{M}$ is the $2\times 2$ all-ones matrix \rev{and $\phi(x,y)=xy$}. Then $S = (c_1 + c_2)(c_1' + c_2')$, which depends \emph{only} on the sums $s^M := c_1 + c_2$ and $s^P := c_1' + c_2'$. Consequently, any reparameterization that preserves these sums yields the same observed $p^\mathcal{D}$\rev{; for instance $(\tilde c_1,\tilde c_2)=(c_1+ac_2,\,(1-a)c_2)$ with $a\neq1$, keeping the all-ones mask}.

\paragraph{Takeaway.} \rev{The toy case isolates the two ingredients of component-wise identifiability: a compatibility function whose cross-derivative varies, which turns a single matrix equation into a spanning family, and a support pattern in which no concept's footprint is contained in another's, fitted with no more interactions than necessary. Case~(a) shows that the second ingredient without the first is not enough, which is why our bilinear implementation is covered by the subspace result rather than the component-wise one.}

\subsection{\rev{Theorem~\ref{lem:non-id}: Non-Identifiability Without Structural Constraints}}
\label{app:theory:non-id}

\rev{Before establishing positive identifiability, we show that structural constraints on the interaction mask are indispensable: without them, the latent concepts and mask are not identifiable from the training distribution $p^\mathcal{D}(x^M, x^P)$, even under Assumptions~\ref{ass:concept-recoverability}--\ref{ass:positivity-coverage}.}

\begin{tcolorbox}[colback=gray!10, colframe=gray!50, boxrule=0.5pt, left=5pt, right=5pt, top=5pt, bottom=5pt]
\begin{theorem}[\rev{Non-identifiability of bilinear interactions}]
\label{thm:appendix:nonid}
\label{lem:non-id}
\rev{Suppose Assumptions~\ref{ass:concept-recoverability}--\ref{ass:positivity-coverage} hold and $\phi(x,y)=xy$. Then for any training distribution $p^\mathcal{D}(x^M, x^P)$ there exist infinitely many distinct tuples $(g^M, g^P, \mathcal{M}, \phi)$ all inducing the same $p^\mathcal{D}(x^M, x^P)$. Moreover, if $K_M\ge2$ and $\mathcal M\neq0$, some of these tuples satisfy $\widetilde{\mathcal M}=\mathcal M$ while their concepts mix at least two of the original concepts. Hence, in the bilinear subclass the concepts are non-identifiable, and constraining the mask, for example to be sparse with antichain supports, does not remove the ambiguity.}
\end{theorem}
\end{tcolorbox}

\paragraph{Discussion.} This theorem is the appendix counterpart of Theorem~\ref{thm:main:nonid} in the main text. Without structural restrictions on $\mathcal{M}$, the latent coordinate systems remain unconstrained. \rev{The resulting ambiguity is exactly the large $\mathrm{GL}(K_M)\times \mathrm{GL}(K_P)$ reparameterization orbit discussed in Section~\ref{subsec:nonid}: the same training distribution can be explained by infinitely many equivalent latent coordinate systems.}

\subsubsection{Proof Sketch}
\rev{Fix any valid tuple $(g^M, g^P, \mathcal{M}, \phi)$ consistent with $p^\mathcal{D}$. In the proof only, let $e^M$ and $e^P$ denote auxiliary concept-recovery functions witnessing Assumption~\ref{ass:concept-recoverability}. For any pair of invertible linear maps $T_M \in \mathrm{GL}(K_M)$ and $T_P \in \mathrm{GL}(K_P)$, set $\tilde{\mathbf{c}}^M=T_M\mathbf{c}^M$ and $\tilde{\mathbf{c}}^P=T_P\mathbf{c}^P$, and define}
\[
{\color{black}
\begin{aligned}
\tilde{g}^M(\mathbf{c},\varepsilon^M) &:= g^M(T_M^{-1}\mathbf{c},\varepsilon^M), &
\tilde e^M(x^M)&:=T_M e^M(x^M),\\
\tilde{g}^P(\mathbf{c},\varepsilon^P) &:= g^P(T_P^{-1}\mathbf{c},\varepsilon^P), &
\tilde e^P(x^P)&:=T_P e^P(x^P),\\
\tilde{\mathcal{M}} &:= T_M^{-T} \mathcal{M}\, T_P^{-1}, &
\tilde\phi &:= \phi .
\end{aligned}}
\]
Then the binding score in \eqref{eq:binding} is preserved:
\[
\sum_{i,j} \tilde{\mathcal{M}}_{ij} \phi(\tilde c^M_i, \tilde c^P_j) = \sum_{i,j} \mathcal{M}_{ij} \phi(c^M_i, c^P_j),
\]
\rev{so the selection factor is unchanged. The transformed recovery functions show that Assumption~\ref{ass:concept-recoverability} remains satisfied, and the nuisance variables enter only through the reparameterized observation maps. After integrating those variables out, the reparameterized tuple induces the same training distribution $p^\mathcal{D}(x^M, x^P)$. Since $\mathrm{GL}(K_M) \times \mathrm{GL}(K_P)$ is an infinite group, infinitely many distinct transformations yield the same $p^\mathcal{D}$. Critically, these transformed masks generally violate the structural restrictions required for positive recovery, so the ambiguity is ruled out only after such restrictions are imposed.}

\subsubsection{Intuitive Example}
\rev{The dense case in App.~\ref{app:theory:toy} is the smallest example. With a $2\times2$ all-ones mask and $\phi(x,y)=xy$, the binding score becomes}
\[
{\color{black}
S(\mathbf c^M,\mathbf c^P)
=(c^M_1+c^M_2)(c^P_1+c^P_2).
}
\]
\rev{The selected distribution therefore sees only the molecule-side sum and the protein-side sum, not the two coordinates separately. Different latent coordinates that preserve these sums, or more generally invertible reparameterizations absorbed by a transformed mask, produce the same binder-selected observations. This is why dense holistic interaction leaves the concepts and mask underdetermined.}

\subsubsection{Full Proof of Theorem~\ref{lem:non-id}}
\label{app:theory:proof-nonid}

\begin{proof}
We construct an explicit infinite family of tuples all producing the same $p^\mathcal{D}$.

\rev{Fix any reference tuple $(g^M, g^P, \mathcal{M}, \phi)$ compatible with $p^\mathcal{D}$, and let $e^M,e^P$ be auxiliary functions witnessing Assumption~\ref{ass:concept-recoverability} on the data support. For clarity, consider the bilinear case $\phi(x,y)=xy$; the argument extends to compatibility functions that are locally linear in each argument on the relevant support.}

\rev{Pick any $(T_M, T_P) \in \mathrm{GL}(K_M)\times\mathrm{GL}(K_P)$. Set $\tilde{\mathbf{c}}^M=T_M\mathbf{c}^M$ and $\tilde{\mathbf{c}}^P=T_P\mathbf{c}^P$, and define the reparameterized tuple:}
\begin{align}
\tilde g^M(\tilde{\mathbf{c}}^M,\varepsilon^M) &:= g^M(T_M^{-1} \tilde{\mathbf{c}}^M, \varepsilon^M), &
\tilde g^P(\tilde{\mathbf{c}}^P,\varepsilon^P) &:= g^P(T_P^{-1} \tilde{\mathbf{c}}^P, \varepsilon^P), &
\widetilde{\mathcal{M}} &:= T_M^{-\top}\, \mathcal{M}\, T_P^{-1}.
\end{align}
\rev{The corresponding auxiliary recovery functions are $\tilde e^M(x^M):=T_M e^M(x^M)$ and $\tilde e^P(x^P):=T_P e^P(x^P)$, so the reparameterized tuple also satisfies Assumption~\ref{ass:concept-recoverability}.}
Setting $\tilde{\mathbf{c}}^M = T_M \mathbf{c}^M$ and $\tilde{\mathbf{c}}^P = T_P \mathbf{c}^P$:
\begin{equation*}
\sum_{i,j} \widetilde{\mathcal{M}}_{ij} \tilde c^M_i \tilde c^P_j
= \tilde{\mathbf{c}}^{M\top} \widetilde{\mathcal{M}} \tilde{\mathbf{c}}^P
= (T_M \mathbf{c}^M)^\top T_M^{-\top} \mathcal{M} T_P^{-1} (T_P \mathbf{c}^P)
= \mathbf{c}^{M\top} \mathcal{M} \mathbf{c}^P.
\end{equation*}
\rev{Hence the binding probability $P(B=1\mid \mathbf{c}^M, \mathbf{c}^P)$ is identical for the two tuples. Therefore the selected joint over observations, with nuisance variables integrated out, is identical:}
\[
\rev{p^\mathcal{D}(x^M,x^P) = \iint p(x^M\mid\mathbf{c}^M) p(x^P\mid\mathbf{c}^P) p^\mathcal{D}(\mathbf{c}^M,\mathbf{c}^P)\,d\mathbf{c}^M\,d\mathbf{c}^P,}
\]
\rev{where $p(x^M\mid\mathbf{c}^M)=\int p(x^M\mid\mathbf{c}^M,\varepsilon^M)p(\varepsilon^M)\,d\varepsilon^M$ and analogously for $x^P$. The same expression is obtained under the transformed tuple, so $\tilde p^\mathcal{D}(x^M,x^P)=p^\mathcal{D}(x^M,x^P)$.}

\rev{Since $\mathrm{GL}(K_M)\times\mathrm{GL}(K_P)$ is a connected Lie group of dimension $K_M^2+K_P^2$, this parameterizes an uncountable family of equivalent tuples, which proves the first part.}

\rev{For the second part, let $K_M\ge2$ and $\mathcal M\neq0$. It suffices to find $A\in\mathrm{GL}(K_M)$ that is \emph{not} a permutation matrix times a diagonal matrix, together with $B\in\mathrm{GL}(K_P)$, such that $A^\top\mathcal M B=\mathcal M$: setting $T_M:=A^{-1}$ and $T_P:=B^{-1}$ then gives $\widetilde{\mathcal M}=T_M^{-\top}\mathcal M T_P^{-1}=A^\top\mathcal M B=\mathcal M$, while $\tilde{\mathbf c}^M=A^{-1}\mathbf c^M$ mixes at least two concepts, because the inverse of a permutation times a diagonal matrix is again of that form.}
\begin{itemize}[leftmargin=1.6em, itemsep=2pt]
    \item \rev{If $\operatorname{rank}(\mathcal M)<K_M$, pick $\mathbf w\neq0$ with $\mathbf w^\top\mathcal M=0$, an index $k$ with $w_k\neq0$, and any $m\neq k$, and set $\mathbf u:=w_m\mathbf e_k-w_k\mathbf e_m\neq0$, so that $\mathbf w^\top\mathbf u=0$. Let $A^\top:=I+\varepsilon\,\mathbf u\mathbf w^\top$ and $B:=I$. Then $A^\top\mathcal M=\mathcal M$; moreover $(\mathbf u\mathbf w^\top)^2=\mathbf u(\mathbf w^\top\mathbf u)\mathbf w^\top=0$, so $A^\top$ is invertible; and for small $\varepsilon>0$ its column $k$ has both the diagonal entry $1+\varepsilon u_kw_k$ and the off-diagonal entry $-\varepsilon w_k^2$ in row $m$ nonzero.}
    \item \rev{If $\operatorname{rank}(\mathcal M)=K_M$, take any $A\in\mathrm{GL}(K_M)$ that is not a permutation times a diagonal matrix, let $\mathcal M^{+}$ be the Moore--Penrose pseudoinverse, and set $B:=\mathcal M^{+}A^{-\top}\mathcal M+(I-\mathcal M^{+}\mathcal M)$. Since $\mathcal M\mathcal M^{+}=I$, we get $\mathcal M B=A^{-\top}\mathcal M$, i.e.\ $A^\top\mathcal M B=\mathcal M$. Moreover $B$ is the identity on $\ker\mathcal M$ and, on the row space $\mathcal R=\mathrm{range}(\mathcal M^\top)$, equals the composition of the bijections $\mathcal M|_{\mathcal R}:\mathcal R\to\mathbb R^{K_M}$, $A^{-\top}$, and $\mathcal M^{+}:\mathbb R^{K_M}\to\mathcal R$; as $\mathbb R^{K_P}=\mathcal R\oplus\ker\mathcal M$, $B$ is invertible.}
\end{itemize}
\rev{In both cases the mask, hence its support, sparsity level and antichain property, is unchanged. The example $\mathcal M=I_2$, $A^{-1}=\big[\begin{smallmatrix}1&1\\0&1\end{smallmatrix}\big]$, $B=A^{-\top}$ is Case~(a) of App.~\ref{app:theory:toy}.}

\rev{Thus, with a bilinear compatibility function the reparameterizations cannot be distinguished by $p^\mathcal{D}$, and no constraint on $\mathcal M$ alone removes the ambiguity; a non-bilinear $\phi$ (Condition~2.ii) is what Theorem~\ref{thm:identifiability} adds.}
\end{proof}

\subsection{\rev{Theorem~\ref{thm:identifiability}: Component-Wise Identifiability Under Sparse Interaction}}
\label{app:theory:main}

\rev{We now prove the central identifiability result. The argument has two parts. A cross-derivative step turns the equality of the two selected distributions into a family of matrix equations, one for each point of the concept support. A combinatorial step (Lemma~\ref{lem:antichain}) then shows that, for a sparse antichain support, this family admits only permutations and rescalings as solutions. The combinatorial step follows the sparsity-based identifiability arguments of~\citet{lachapelle2022disentanglement} and~\citet{zheng2022identifiability}.}

\begin{tcolorbox}[colback=gray!10, colframe=gray!50, boxrule=0.5pt, left=5pt, right=5pt, top=5pt, bottom=5pt]
\begin{theorem}[\rev{Component-wise identifiability under sparse interaction}]
\label{thm:identifiability}
\rev{Suppose Assumptions~\ref{ass:concept-recoverability}--\ref{ass:loglinear} hold for the true tuple $(g^M,g^P,\mathcal M,\phi)$ and for an alternative tuple $(\tilde g^M,\tilde g^P,\widetilde{\mathcal M},\tilde\phi)$ with $\tilde\phi\in C^2$ inducing the same $p^\mathcal D(x^M,x^P)$, and that Conditions~2.i--2.iii of Theorem~\ref{thm:main:ident} hold. Then there exist permutations $\pi$ of $[K_M]$ and $\varsigma$ of $[K_P]$ such that each $\tilde c^M_{\pi(i)}$ is an affine function of $c^M_i$ alone, each $\tilde c^P_{\varsigma(j)}$ is an affine function of $c^P_j$ alone, and $\mathrm{supp}(\widetilde{\mathcal M})=\{(\pi(i),\varsigma(j)):(i,j)\in\mathrm{supp}(\mathcal M)\}$.}
\end{theorem}
\end{tcolorbox}

\paragraph{Discussion.} \rev{This is the appendix counterpart of Theorem~\ref{thm:main:ident} and the main positive result of the paper. Condition~2.iii reduces the problem to the linear ambiguity of Theorem~\ref{lem:non-id}. Condition~2.ii supplies variation: it turns the single matrix equation available for a bilinear compatibility function into a family of equations spanning all matrices supported on $\mathrm{supp}(\mathcal M)$. Condition~2.i and the minimality clause supply the combinatorial constraint that this family can satisfy only through permutations and rescalings. A single matrix equation is never enough, whatever the support of the mask (Example~\ref{ex:lemma}(a)).}

\subsubsection{\rev{A Sparsity Lemma}}

\rev{For a set $S\subseteq[K_M]\times[K_P]$, let $\mathbb R^S$ denote the space of $K_M\times K_P$ matrices supported in $S$, and write $R_k(S)=\{l:(k,l)\in S\}$ and $C_l(S)=\{k:(k,l)\in S\}$ for its row and column supports.}

\begin{lemma}[Antichain recovery]
\label{lem:antichain}
\rev{Let $S\subseteq[K_M]\times[K_P]$ have nonempty row supports forming an antichain and nonempty column supports forming an antichain. Let $P\in\mathrm{GL}(K_M)$, $Q\in\mathrm{GL}(K_P)$, and $\widetilde S\subseteq[K_M]\times[K_P]$ with $|\widetilde S|\le|S|$. Suppose there is a family $\{T_t\}_{t\in\mathcal T}\subseteq\mathbb R^S$ that spans $\mathbb R^S$ and satisfies $\mathrm{supp}(PT_tQ)\subseteq\widetilde S$ for every $t\in\mathcal T$. Then there exist permutations $\pi$ of $[K_M]$ and $\varsigma$ of $[K_P]$ such that $P_{ik}\neq0$ if and only if $i=\pi(k)$, and $Q_{lj}\neq0$ if and only if $j=\varsigma(l)$; that is, $P$ and $Q$ are each a permutation matrix times an invertible diagonal matrix. Moreover $\widetilde S=\{(\pi(k),\varsigma(l)):(k,l)\in S\}$.}
\end{lemma}

\begin{proof}
\rev{The linear map $L(X)=PXQ$ is injective. Since the $T_t$ span $\mathbb R^S$ and $L(T_t)\in\mathbb R^{\widetilde S}$ for every $t$, linearity gives $L(\mathbb R^S)\subseteq\mathbb R^{\widetilde S}$. Comparing dimensions, $|S|\le|\widetilde S|\le|S|$, so $|\widetilde S|=|S|$ and $L(\mathbb R^S)=\mathbb R^{\widetilde S}$.}

\rev{For $(k,l)\in S$ the matrix unit $E_{kl}$ lies in $\mathbb R^S$, and $L(E_{kl})=\mathbf p_k\mathbf q_l^\top$, where $\mathbf p_k$ is the $k$-th column of $P$ and $\mathbf q_l^\top$ the $l$-th row of $Q$. Hence}
\begin{equation}
{\color{black}
\mathrm{supp}(\mathbf p_k)\times\mathrm{supp}(\mathbf q_l)\subseteq\widetilde S\qquad\text{for all }(k,l)\in S.
}
\label{eq:lemma-outer}
\end{equation}
\rev{Since $\det P=\sum_{\pi}\mathrm{sgn}(\pi)\prod_k P_{\pi(k)k}\neq0$, some permutation $\pi$ has $P_{\pi(k)k}\neq0$ for all $k$; likewise some $\varsigma$ has $Q_{l\varsigma(l)}\neq0$ for all $l$. By~\eqref{eq:lemma-outer}, $\Gamma:=\{(\pi(k),\varsigma(l)):(k,l)\in S\}\subseteq\widetilde S$. As $(k,l)\mapsto(\pi(k),\varsigma(l))$ is injective, $|\Gamma|=|S|=|\widetilde S|$, so $\widetilde S=\Gamma$.}

\rev{Suppose $P_{ik}\neq0$ for some $i\neq\pi(k)$, and put $k'=\pi^{-1}(i)\neq k$. For every $l\in R_k(S)$, \eqref{eq:lemma-outer} gives $(i,\varsigma(l))\in\widetilde S=\Gamma$, hence $(k',l)\in S$, i.e.\ $l\in R_{k'}(S)$. Thus $R_k(S)\subseteq R_{k'}(S)$, contradicting the row antichain; so column $k$ of $P$ has the single nonzero entry $P_{\pi(k)k}$. Symmetrically, if $Q_{lj}\neq0$ for some $j\neq\varsigma(l)$, put $l'=\varsigma^{-1}(j)\neq l$; for every $k\in C_l(S)$, \eqref{eq:lemma-outer} gives $(\pi(k),j)\in\Gamma$, hence $(k,l')\in S$, so $C_l(S)\subseteq C_{l'}(S)$, contradicting the column antichain.}
\end{proof}

\begin{example}[\rev{Each hypothesis of Lemma~\ref{lem:antichain} is needed}]
\label{ex:lemma}
\rev{Let $K_M=K_P=2$.}
\begin{enumerate}[label=(\alph*), leftmargin=1.8em, itemsep=2pt]
    \item \rev{\emph{Spanning family.} Let $S=\{(1,1),(2,2)\}$, $P=\big[\begin{smallmatrix}1&1\\0&1\end{smallmatrix}\big]$ and $Q=P^{-1}$. The single matrix $T=I_2$ gives $PTQ=I_2$, supported in $\widetilde S=S$, although $P$ is not a permutation times a diagonal matrix. The family $\{I_2\}$ does not span $\mathbb R^S$: for $T=\mathrm{diag}(2,5)$ one gets $PTQ=\big[\begin{smallmatrix}2&3\\0&5\end{smallmatrix}\big]$, which leaves $\widetilde S$. This is exactly the bilinear situation, where the cross-derivative matrix is the constant $\mathcal M$.}
    \item \rev{\emph{Antichain.} Let $S=\{(1,1),(2,1),(2,2)\}$, so $R_1=\{1\}\subseteq R_2=\{1,2\}$, and let $P=\big[\begin{smallmatrix}1&0\\a&1\end{smallmatrix}\big]$ with $a\neq0$, $Q=I_2$. Then $PTQ\in\mathbb R^S$ for every $T\in\mathbb R^S$.}
    \item \rev{\emph{Sparsity minimality.} If $\widetilde S=[K_M]\times[K_P]$, every $P$ and $Q$ satisfy the support condition.}
\end{enumerate}
\end{example}

\subsubsection{Full Proof of Theorem~\ref{thm:identifiability}}
\label{app:theory:proof-ident}

\begin{proof}
\rev{Write $S(\mathbf c^M,\mathbf c^P)=\sum_{k,l}\mathcal M_{kl}\phi(c_k^M,c_l^P)$ and $\widetilde S(\tilde{\mathbf c}^M,\tilde{\mathbf c}^P)=\sum_{k,l}\widetilde{\mathcal M}_{kl}\tilde\phi(\tilde c_k^M,\tilde c_l^P)$, and let $e^M,e^P$ and $\tilde e^M,\tilde e^P$ be the recovery maps of Assumption~\ref{ass:concept-recoverability} for the two tuples.}

\paragraph{Step 1: Pull-back to concept coordinates.} \rev{By Eq.~\eqref{eq:ci-b} and Assumption~\ref{ass:concept-recoverability}, $B$ depends on the observations only through $\mathbf c^M=e^M(x^M)$ and $\mathbf c^P=e^P(x^P)$. By Assumption~\ref{ass:marginal-independence} and the DAG in Eq.~\eqref{eq:v-structure}, $(\mathbf c^M,\varepsilon^M)$ is independent of $(\mathbf c^P,\varepsilon^P)$, so $p(x^M,x^P)=p(x^M)p(x^P)$. Bayes' rule and Assumption~\ref{ass:loglinear} then give}
\begin{equation}
{\color{black}
\log p^\mathcal D(x^M,x^P)=\log p(x^M)+\log p(x^P)+S\bigl(e^M(x^M),e^P(x^P)\bigr)+\log\kappa-\log P(B{=}1).
}
\label{eq:loglinear-pullback}
\end{equation}
\rev{The same identity holds for the alternative tuple with its own marginals, constant, and recovery maps. Subtracting,}
\[
{\color{black}
S\bigl(e^M(x^M),e^P(x^P)\bigr)-\widetilde S\bigl(\tilde e^M(x^M),\tilde e^P(x^P)\bigr)=\alpha(x^M)+\beta(x^P)
}
\]
\rev{for some $\alpha,\beta$. By Condition~2.iii, $\tilde e^M=T_Me^M+\mathbf t_M$ and $\tilde e^P=T_Pe^P+\mathbf t_P$, so the left-hand side depends on $x^M$ only through $e^M(x^M)$; fixing $x^P$ shows the same for $\alpha$, and symmetrically for $\beta$. Hence, on the concept support,}
\begin{equation}
{\color{black}
S(\mathbf c^M,\mathbf c^P)-\widetilde S\bigl(T_M\mathbf c^M+\mathbf t_M,\,T_P\mathbf c^P+\mathbf t_P\bigr)=a(\mathbf c^M)+b(\mathbf c^P).
}
\label{eq:match}
\end{equation}
\rev{Under the logistic link, $S$ in~\eqref{eq:loglinear-pullback} would be replaced by $\log\sigma(S)$, whose cross-derivatives are dense; Assumption~\ref{ass:loglinear} is what lets the support of $\mathcal M$ survive the next step.}

\paragraph{Step 2: Cross-derivatives.} \rev{By Assumptions~\ref{ass:marginal-independence} and~\ref{ass:positivity-coverage} the concept support is a product of open sets, so we may differentiate~\eqref{eq:match}. Apply $\partial^2/\partial c_i^M\partial c_j^P$. The right-hand side vanishes. In $S$ only the term $(i,j)$ depends on both $c_i^M$ and $c_j^P$, so $\partial^2S/\partial c_i^M\partial c_j^P=\mathcal M_{ij}\,\partial_1\partial_2\phi(c_i^M,c_j^P)$; by the chain rule the second term contributes $\sum_{k,l}(T_M)_{ki}(T_P)_{lj}\widetilde{\mathcal M}_{kl}\,\partial_1\partial_2\tilde\phi(\tilde c_k^M,\tilde c_l^P)$. Define}
\[
{\color{black}
T(\mathbf c)_{ij}:=\mathcal M_{ij}\,\partial_1\partial_2\phi(c_i^M,c_j^P),\qquad
\widetilde T(\mathbf c)_{kl}:=\widetilde{\mathcal M}_{kl}\,\partial_1\partial_2\tilde\phi(\tilde c_k^M,\tilde c_l^P),
}
\]
\rev{with $\mathbf c=(\mathbf c^M,\mathbf c^P)$. Then $T(\mathbf c)\in\mathbb R^{S}$ for $S:=\mathrm{supp}(\mathcal M)$, $\mathrm{supp}(\widetilde T(\mathbf c))\subseteq\widetilde S:=\mathrm{supp}(\widetilde{\mathcal M})$, and $T(\mathbf c)=T_M^\top\widetilde T(\mathbf c)T_P$, that is,}
\begin{equation}
{\color{black}
\widetilde T(\mathbf c)=P\,T(\mathbf c)\,Q,\qquad P:=T_M^{-\top},\quad Q:=T_P^{-1},
}
\label{eq:double-deriv}
\end{equation}
\rev{for every $\mathbf c$ in the support, with $P$ and $Q$ independent of $\mathbf c$.}

\paragraph{\rev{Step 3: The cross-derivative matrices span $\mathbb R^{S}$.}} \rev{Suppose not. Then some nonzero $W\in\mathbb R^{S}$ satisfies $\sum_{(i,j)\in S}W_{ij}\mathcal M_{ij}\,\partial_1\partial_2\phi(c_i^M,c_j^P)=0$ for all $\mathbf c$ in the open support. Pick $(i,j)\in S$ with $W_{ij}\neq0$ and apply $\partial^2/\partial c_i^M\partial c_j^P$; again only the term $(i,j)$ survives, so $W_{ij}\mathcal M_{ij}\,\partial_1^2\partial_2^2\phi(c_i^M,c_j^P)=0$ on a nonempty open set. By Condition~2.ii, $\partial_1^2\partial_2^2\phi$ is real-analytic and not identically zero, so it cannot vanish on a nonempty open set; since $W_{ij}\mathcal M_{ij}\neq0$, this is a contradiction.}

\paragraph{\rev{Step 4: Combinatorial step.}} \rev{By Condition~2.i, $S$ has nonempty row and column supports forming antichains, and by the minimality clause of Condition~2.iii, $|\widetilde S|\le|S|$. By Steps~2 and~3 the family $\{T(\mathbf c)\}$ spans $\mathbb R^{S}$ and satisfies $\mathrm{supp}(PT(\mathbf c)Q)\subseteq\widetilde S$. Lemma~\ref{lem:antichain} therefore yields permutations $\pi,\varsigma$ with $P_{ik}\neq0\iff i=\pi(k)$, $Q_{lj}\neq0\iff j=\varsigma(l)$, and $\widetilde S=\{(\pi(k),\varsigma(l)):(k,l)\in S\}$.}

\paragraph{\rev{Step 5: Conclusion.}} \rev{Since $P=T_M^{-\top}$ has the single nonzero entry $P_{\pi(k)k}$ in each column $k$, the matrix $T_M=P^{-\top}$ has the single nonzero entry $(T_M)_{\pi(k)k}=1/P_{\pi(k)k}$ in each column $k$; hence $\tilde c^M_{\pi(k)}=(T_M)_{\pi(k)k}c^M_k+(\mathbf t_M)_{\pi(k)}$ depends on $c^M_k$ alone. Likewise $T_P=Q^{-1}$ gives $\tilde c^P_{\varsigma(l)}$ as an affine function of $c^P_l$ alone. Together with the support identity of Step~4, the two tuples differ only by concept permutations, component-wise invertible maps, and the scale shared between mask and compatibility function, which are exactly the equivalences allowed in Definition~\ref{def:identifiability}.}
\end{proof}

\begin{remark}[\rev{Which conditions are needed}]
\rev{Condition~2.ii cannot be dropped: for $\phi(x,y)=xy$ the second part of Theorem~\ref{lem:non-id} gives counterexamples satisfying Conditions~2.i and~2.iii, the smallest being Case~(a) of App.~\ref{app:theory:toy}. Condition~2.i and the minimality clause cannot be dropped from Lemma~\ref{lem:antichain} (Example~\ref{ex:lemma}). They are, however, not necessary for identifiability as such: Proposition~\ref{prop:nonlinear} obtains it without any sparsity, under a stronger nonlinearity condition that depends on the mask and must hold at every point of the support. What the sparse antichain structure buys is that the analytic requirement, Condition~2.ii, is a generic property of $\phi$ alone, while the combinatorial requirement can be checked directly on a learned mask.}
\end{remark}

\subsubsection{\rev{Beyond Linear Concept Ambiguity}}

\rev{Condition~2.iii restricts the competing parameterizations to affine reparameterizations of the concepts. The pairwise form of the binding score also constrains general nonlinear reparameterizations, by an argument similar to the one for additive decoders~\citep{lachapelle2023additive}.}

\begin{proposition}[\rev{Nonlinear reparameterizations}]
\label{prop:nonlinear}
\rev{Suppose Assumptions~\ref{ass:concept-recoverability}--\ref{ass:loglinear} hold for both tuples, the support of $p(\mathbf c^M)$ is open and convex, $\phi\in C^3$, $\tilde\phi\in C^1$, and the concepts of the two tuples are related by $C^2$ diffeomorphisms $\tilde{\mathbf c}^M=F(\mathbf c^M)$ and $\tilde{\mathbf c}^P=G(\mathbf c^P)$. Fix protein-side probes $\mathbf b_1,\ldots,\mathbf b_R$ in the support of $\mathbf c^P$, write $b_{r,j}$ for the $j$-th coordinate of $\mathbf b_r$, and define for each $i\in[K_M]$}
\[
{\color{black}
\mathbf h_i(x):=\bigl(\mathcal M_{ij}\,\partial_2\phi(x,b_{r,j})\bigr)_{r\in[R],\,j\in[K_P]}\in\mathbb R^{RK_P}.
}
\]
\rev{If, for every $\mathbf c^M$ in the support, the $2K_M$ vectors $\{\mathbf h_i'(c_i^M),\mathbf h_i''(c_i^M)\}_{i=1}^{K_M}$ are linearly independent, then there is a permutation $\pi$ of $[K_M]$ such that each $F_{\pi(i)}$ depends on $c_i^M$ alone. The symmetric condition with molecule-side probes gives the same conclusion for $G$.}
\end{proposition}

\begin{proof}
\rev{As in Step~1 above, with $F,G$ in place of the affine maps, $S(\mathbf c^M,\mathbf c^P)-\widetilde S(F(\mathbf c^M),G(\mathbf c^P))=a(\mathbf c^M)+b(\mathbf c^P)$. Differentiate with respect to $c_j^P$ and evaluate at $\mathbf c^P=\mathbf b_r$. On the left, $\partial S/\partial c_j^P(\mathbf c^M,\mathbf b_r)=\sum_i\mathcal M_{ij}\partial_2\phi(c_i^M,b_{r,j})$, so stacking over $(r,j)$ gives $\mathbf h(\mathbf c^M):=\sum_i\mathbf h_i(c_i^M)$. On the right, the corresponding derivative equals $\sum_l\partial_jG_l(\mathbf b_r)\sum_k\widetilde{\mathcal M}_{kl}\partial_2\tilde\phi(F_k(\mathbf c^M),G_l(\mathbf b_r))$, a sum of functions of the single coordinates $F_k(\mathbf c^M)$, while $\partial b/\partial c_j^P(\mathbf b_r)$ is constant. Hence $\mathbf h=\widetilde{\mathbf H}\circ F$ for some $\widetilde{\mathbf H}(\mathbf u)=\sum_k\widetilde{\mathbf H}_k(u_k)$ additive in $\mathbf u$.}

\rev{Let $V:=F^{-1}$. Then $\widetilde{\mathbf H}=\mathbf h\circ V$ is $C^2$ and additive, so for $k\neq k'$,}
\[
{\color{black}
0=\frac{\partial^2}{\partial u_k\partial u_{k'}}\sum_i\mathbf h_i(V_i(\mathbf u))
=\sum_i\Bigl[\mathbf h_i''(V_i)\,\partial_kV_i\,\partial_{k'}V_i+\mathbf h_i'(V_i)\,\partial^2_{kk'}V_i\Bigr].
}
\]
\rev{By the linear independence assumption, $\partial_kV_i\,\partial_{k'}V_i=0$ for all $i$ and $k\neq k'$, so each row of $J_V(\mathbf u)$ has at most one nonzero entry; as $J_V$ is invertible, it is a permutation times a diagonal matrix, and so is $J_F(\mathbf c^M)=J_V(F(\mathbf c^M))^{-1}$. The pattern is locally constant by continuity, hence constant on the connected support, so there is $\pi$ with $\partial F_{\pi(i)}/\partial c_k^M\equiv0$ for $k\neq i$; convexity of the support then makes $F_{\pi(i)}$ a function of $c_i^M$ alone. The argument for $G$ is symmetric.}
\end{proof}

\rev{The linear independence condition requires $RK_P\ge2K_M$. It fails for a bilinear $\phi$, for which $\mathbf h_i''\equiv0$, and it fails wherever two concepts whose rows of $\mathcal M$ are proportional take the same value, since their vectors then coincide up to scale; distinct row supports, as in Condition~2.i, remove this degeneracy. If in addition $\|\widetilde{\mathcal M}\|_0\le\|\mathcal M\|_0$ and $\partial_1\partial_2\phi\neq0$ on the support, the cross-derivative identity of Step~2 shows that $\mathrm{supp}(\widetilde{\mathcal M})$ is the permuted support of $\mathcal M$, as in Theorem~\ref{thm:identifiability}.}

\subsection{\rev{Theorem~\ref{thm:lowrank-recovery}: Subspace Identifiability Under Low-Rank Interaction}}
\label{app:theory:lowrank}

\rev{Condition~2.i requires an exactly sparse antichain support and Condition~2.ii a non-bilinear compatibility function, both of which may be too strong in practice.} \rev{For protein-molecule binding, we argue that the subspace-identifiability alternative in Condition~3.i is also natural empirically: the binding kernel has low rank.}

\paragraph{Motivation.} Protein--molecule binding in drug discovery is dominated by a small number of pharmacophoric motifs: hydrogen bonds, $\pi$-stacking, hydrophobic packing, and electrostatic complementarity. It is therefore physically reasonable that the complete binding response can be decomposed into $r$ independent \emph{interaction modes}, each characterized by a pair $(u_k, v_k)$ describing how each modality contributes to mode $k$. We expect $r$ to be small; typical estimates from medicinal chemistry place $r$ at $5$--$20$ for a given target family.

\paragraph{Why this replaces SEDGE's conservative assumption.} SEDGE's identifiability argument relies on $P(S{=}0) \approx 0$, i.e., ``unseen configurations are rare.'' This is false for binding: most protein-molecule pairs do \emph{not} bind, so $P(B{=}0) \gg P(B{=}1)$. Condition~3.i replaces this with a physics-motivated structural constraint on the kernel itself, preserving identifiability without requiring the pathological SEDGE-style coverage condition.

\paragraph{Architectural implication.} \rev{Under Condition~3.i, the bilinear binding score can be implemented with a rank-$r$ factorized mask $\mathcal{M} = U V^\top$ where $U \in \mathbb{R}^{K_M \times r}$ and $V \in \mathbb{R}^{K_P \times r}$, reducing parameters from $K_MK_P$ to $r(K_M+K_P)$ and naturally enforcing the low-rank prior. This is the structural basis for our low-rank concept-basis implementation in App.~\ref{app:impl_lr}; the main concept-pair implementation in App.~\ref{app:impl_sp} can also reach this regime empirically through regularization-induced low rank.}

\paragraph{Subspace identifiability under low-rank interaction.} When Conditions~2.i--2.iii are replaced by Conditions~3.i--3.ii, the full concept-pair mask need not be identifiable. However, an analog of Theorem~\ref{thm:identifiability} still holds at the level of the rank-$r$ \emph{interaction subspaces}:

\begin{tcolorbox}[colback=gray!10, colframe=gray!50, boxrule=0.5pt, left=5pt, right=5pt, top=5pt, bottom=5pt]
\begin{theorem}[\rev{Subspace identifiability under low-rank interaction}]
\label{thm:lowrank-recovery}
\label{prop:lowrank-recovery}
\rev{Suppose Assumptions~\ref{ass:concept-recoverability}--\ref{ass:loglinear} and Conditions~3.i--3.ii of Theorem~\ref{thm:main:lowrank} hold, so the centred binding score $S^\circ$ of Condition~3.i admits a minimal rank-$r$ separable factorization}
\[
{\color{black}
S^\circ(\mathbf c^M,\mathbf c^P)
=\sum_{\ell=1}^{r}u_\ell(\mathbf c^M)v_\ell(\mathbf c^P)
=\mathbf u(\mathbf c^M)^\top\mathbf v(\mathbf c^P),
\qquad r<\min\{K_M,K_P\},
}
\]
\rev{with linearly independent mode families $\{u_\ell\}_{\ell=1}^r\subset L^2(p^\mathcal D(\mathbf c^M))$ and $\{v_\ell\}_{\ell=1}^r\subset L^2(p^\mathcal D(\mathbf c^P))$. Assume the low-rank sufficient-variability condition: there exist probe values $\mathbf a_1,\ldots,\mathbf a_{R_M}\in\mathrm{int}(\mathrm{supp}(\mathbf c^M))$ and $\mathbf b_1,\ldots,\mathbf b_{R_P}\in\mathrm{int}(\mathrm{supp}(\mathbf c^P))$, with $R_M,R_P\ge r$, such that}
\[
{\color{black}
\mathbf U_A
:=
\begin{bmatrix}
u_1(\mathbf a_1)&\cdots&u_r(\mathbf a_1)\\[-1mm]
\vdots&\ddots&\vdots\\[-1mm]
u_1(\mathbf a_{R_M})&\cdots&u_r(\mathbf a_{R_M})
\end{bmatrix}
\qquad\text{and}\qquad
\mathbf V_B
:=
\begin{bmatrix}
v_1(\mathbf b_1)&\cdots&v_r(\mathbf b_1)\\[-1mm]
\vdots&\ddots&\vdots\\[-1mm]
v_1(\mathbf b_{R_P})&\cdots&v_r(\mathbf b_{R_P})
\end{bmatrix}
}
\]
\[
{\color{black}
\operatorname{rank}(\mathbf U_A)=\operatorname{rank}(\mathbf V_B)=r.
}
\]
\rev{Then the molecule-side and protein-side interaction subspaces}
\[
{\color{black}
\mathcal U_M:=\mathrm{span}\{u_1,\ldots,u_r\}\subset L^2(p^\mathcal D(\mathbf c^M)),
\qquad
\mathcal V_P:=\mathrm{span}\{v_1,\ldots,v_r\}\subset L^2(p^\mathcal D(\mathbf c^P))
}
\]
\rev{are subspace-identifiable from $p^\mathcal D(x^M,x^P)$ in the sense of Definition~\ref{def:identifiability}.}
\end{theorem}
\end{tcolorbox}

\paragraph{Discussion.} \rev{This is the appendix counterpart of Theorem~\ref{thm:main:lowrank}. The two theorem-specific conditions play different roles:}
\begin{itemize}[leftmargin=1.6em, itemsep=2pt]
    \item \rev{\textbf{Condition~3.i} is the structural low-rank condition. It weakens exact entry-wise sparsity into a mode-level prior, so the identifiable object is the pair of low-dimensional interaction subspaces rather than each individual concept pair.}
    \item \rev{\textbf{Condition~3.ii} is the sufficient-variability condition for those low-rank modes. It ensures that the selected data excite the mode-evaluation matrices with full rank, which makes the subspaces statistically recoverable.}
\end{itemize}
\rev{Together, these conditions define the subspace-identifiability regime relevant to our experiments, where the learned masks are empirically low-rank rather than exactly sparse.}

\subsubsection{Proof Sketch}
\rev{As in the preceding proofs, nuisance variables $\varepsilon^M,\varepsilon^P$ are marginalized into the observation kernels and Assumption~\ref{ass:concept-recoverability} pulls the selected observation distribution back to recovered concept coordinates. Since the selected distribution determines the binding score only up to additive functions of a single modality, the argument operates on the rank-$r$ \emph{centred} kernel $S^\circ(\mathbf c^M,\mathbf c^P)$.} If two minimal factorizations induce the same centred kernel, then evaluating it at the full-rank protein-side probes from Condition~3.ii expresses each molecule-side mode basis as a linear combination of the same finite set of kernel sections $S^\circ(\cdot,\mathbf b_q)$. Hence the two molecule-side mode families span the same subspace. The full-rank molecule-side probes give the symmetric statement for the protein-side modes. Thus the identifiable object is the pair of interaction subspaces, not an ordered basis of individual modes.

Intuitively, the rank-$r$ subspaces are invariants of the low-rank interaction kernel, even when concept-wise support is not identifiable. This is exactly the subspace-identifiability analog we seek: we identify interaction \emph{subspaces} rather than individual concept-pair entries. The full derivation is given below.

\subsubsection{Intuitive Example}
\rev{A low-rank interaction can be dense. For example, with}
\[
{\color{black}
\mathbf u=
\begin{bmatrix}
0.2\\
0.8\\
1.5
\end{bmatrix},
\qquad
\mathbf v=
\begin{bmatrix}
1.0\\
0.5\\
2.0
\end{bmatrix},
\qquad
\mathcal M=\mathbf u\mathbf v^\top
=
\begin{bmatrix}
0.2&0.1&0.4\\
0.8&0.4&1.6\\
1.5&0.75&3.0
\end{bmatrix},
}
\]
\rev{all nine concept-pair entries are nonzero, so sparse edge-level recovery is not the right target. Nevertheless, every entry is generated by one shared interaction mode: $M_{ij}=u_i v_j$. A larger block mask with three dense $3\times3$ all-ones blocks is similarly rank $3$: it is dense within blocks, but generated by three interaction modes. Theorem~\ref{thm:lowrank-recovery} captures exactly this situation: the identifiable object is the span of the dominant modes, not the individual visual block labels or every pairwise edge.}

\subsubsection{Full Proof of Theorem~\ref{thm:lowrank-recovery}}
\label{app:theory:proof-lowrank}

\begin{proof}
\rev{As in Step~1 of the proof of Theorem~\ref{thm:identifiability} (Eq.~\eqref{eq:loglinear-pullback}), two tuples inducing the same $p^\mathcal D$ have binding scores that agree, as functions of the observations, up to $\alpha(x^M)+\beta(x^P)$; centring at the reference observations corresponding to $\mathbf a_0,\mathbf b_0$ cancels these terms, so the two centred kernels agree as functions of the observations. We therefore compare mode functions of the two tuples as functions of the observations, $u_\ell\circ e^M$ and $\tilde u_\ell\circ\tilde e^M$, suppressing the composition below. The remaining question is whether the rank-$r$ interaction subspaces are uniquely determined by the kernel.}

\rev{Suppose two minimal rank-$r$ factorizations induce the same centred kernel on the data support:}
\[
{\color{black}
S^\circ(\mathbf c^M,\mathbf c^P)
=\mathbf u(\mathbf c^M)^\top\mathbf v(\mathbf c^P)
=\tilde{\mathbf u}(\mathbf c^M)^\top\tilde{\mathbf v}(\mathbf c^P).
}
\]
\rev{Let $\mathcal U=\mathrm{span}\{u_1,\ldots,u_r\}$ and $\tilde{\mathcal U}=\mathrm{span}\{\tilde u_1,\ldots,\tilde u_r\}$, and define $\mathcal V,\tilde{\mathcal V}$ analogously. For each protein-side probe $\mathbf b_q$, the kernel section}
\[
{\color{black}
S^\circ(\cdot,\mathbf b_q)
=\sum_{\ell=1}^{r}v_\ell(\mathbf b_q)u_\ell(\cdot)
=\sum_{\ell=1}^{r}\tilde v_\ell(\mathbf b_q)\tilde u_\ell(\cdot)
}
\]
\rev{belongs to both $\mathcal U$ and $\tilde{\mathcal U}$. Because $\operatorname{rank}(\mathbf V_B)=r$, these sections span $\mathcal U$. Since the same $r$-dimensional span is contained in $\tilde{\mathcal U}$ and the alternative factorization is also minimal rank-$r$, $\dim(\tilde{\mathcal U})=r$; hence $\mathcal U=\tilde{\mathcal U}$.}

\rev{The same argument with molecule-side probes $\mathbf a_m$ and $\operatorname{rank}(\mathbf U_A)=r$ shows $\mathcal V=\tilde{\mathcal V}$. Since each pair of mode functions is a basis for the same $r$-dimensional subspace, there exists an invertible matrix $A\in\mathrm{GL}(r)$ such that, on the data support, $\tilde{\mathbf u}=A^\top\mathbf u$ and $\tilde{\mathbf v}=A^{-1}\mathbf v$. This change of basis leaves $\mathbf u^\top\mathbf v$ unchanged but does not change the identified subspaces. Therefore the pair $(\mathcal U_M,\mathcal V_P)$ is identifiable, while the ordered mode basis is not.}
\end{proof}

\subsection{Summary and Discussion of the Theory}
\label{app:theory:sedge}

The results above establish a coherent theoretical picture for \ours{}. First, the structural causal model clarifies why binder-only training should be treated as a Heckman-style selection problem on a V-structure rather than as ordinary supervised learning. \rev{Second, Theorem~\ref{lem:non-id} shows that a bilinear interaction is non-identifiable under this selection mechanism even when the mask is sparse. Third, Theorem~\ref{thm:identifiability} shows that a sparse antichain interaction with a non-bilinear compatibility function restores identifiability up to the standard ambiguities of permutation, component-wise reparameterization, and joint rescaling, and Proposition~\ref{prop:nonlinear} extends this to nonlinear reparameterizations under a stronger nonlinearity condition.} \rev{Finally, Theorem~\ref{thm:lowrank-recovery} gives a subspace-identifiability result when exact sparse support is too strong: even if entry-wise mask identification fails, the dominant interaction subspace can still remain identifiable.}

\paragraph{What the three results establish together.} \rev{The theoretical development above should be read as a progression from impossibility, to component-wise identifiability under strong sparse structure, to subspace identifiability under weaker but more realistic low-rank structure.} Theorem~\ref{lem:non-id} shows that without additional constraints, binder-only observations do not determine a unique latent interaction mechanism: many dense factorizations can explain the same selected distribution. \rev{Theorem~\ref{thm:identifiability} then shows that component-wise identifiability becomes possible when the interaction mask is sparse in a combinatorially distinctive way and the selected conditional family varies sufficiently across molecule-side concepts. Finally, Theorem~\ref{thm:lowrank-recovery} shows that even when exact element-wise sparsity is relaxed, the dominant interaction modes remain identifiable at the subspace level under a low-rank kernel assumption.}

\paragraph{Why the causal structure matters.} The key source of information is the V-structure
\[
\mathbf c^M \to B \leftarrow \mathbf c^P,
\]
not an externally supplied domain index or an auxiliary supervision signal. Before conditioning on binding, the molecule-side and protein-side concepts are independent under Assumption~\ref{ass:marginal-independence}; after restricting to $B{=}1$, they become statistically coupled through the collider. The identifiability problem is therefore not solved by adding a second dataset or a separate view, but by exploiting the dependence structure already induced by binder-only selection. \rev{The sparse-support theorem uses that dependence to establish component-wise identifiability up to standard equivariances, while the low-rank theorem uses it to establish subspace identifiability of the leading interaction modes when exact support identification is too strong a requirement.}

\section{Implementation Details}
\label{app:implementation}

\paragraph{Architecture overview.} For completeness, Figure~\ref{fig:overview} summarizes the implementation-level view of \ours{}. We place it in the appendix to keep the main text focused on the theory-to-method narrative, while still providing a compact diagram of the encoder, concept-factorization, sparse-interaction, and retrieval-loss pipeline.

\begin{figure}[t]
\centering
\begin{tikzpicture}[scale=0.92, transform shape,
    concept/.style={circle, draw=blue!60, fill=blue!8, minimum size=7mm, font=\small},
    conceptmol/.style={circle, draw=blue!65, fill=blue!10, minimum size=7mm, font=\small},
    conceptpoc/.style={circle, draw=orange!70!black, fill=orange!12, minimum size=7mm, font=\small},
    conceptseq/.style={circle, draw=green!60!black, fill=green!10, minimum size=7mm, font=\small},
    indone/.style={rectangle, rounded corners=2pt, draw=red!65!black, fill=red!12, minimum width=6mm, minimum height=5mm, font=\small},
    indzero/.style={rectangle, rounded corners=2pt, draw=gray!70, fill=gray!15, minimum width=6mm, minimum height=5mm, font=\small},
    title/.style={font=\small\bfseries},
    subt/.style={font=\scriptsize, align=center},
    active/.style={green!60!black, thick, -{Stealth[length=1.8mm]}},
    blocked/.style={red!70, thick, dashed, -{Stealth[length=1.8mm]}},
    agg/.style={gray!75, thick, -{Stealth[length=1.8mm]}},
]

% Column titles
\node[title] at (-4.0, 2.6) {Pocket};
\node[title] at (0.0, 2.6) {Molecule};
\node[title] at (4.0, 2.6) {Sequence};

\node[subt] at (-4.0, 2.1) {3D Structure};
\node[subt] at (0.0, 2.1) {3D Structure};
\node[subt] at (4.0, 2.1) {Protein Sequence};

% Concepts
\node[conceptpoc] (p1) at (-5.0, 0.9) {$c_{\mathrm{poc}}^1$};
\node[conceptpoc] (p2) at (-4.0, 0.9) {$c_{\mathrm{poc}}^2$};
\node[conceptpoc] (p3) at (-3.0, 0.9) {$c_{\mathrm{poc}}^3$};

\node[conceptmol] (m1) at (-1.0, 0.9) {$c_{\mathrm{mol}}^1$};
\node[conceptmol] (m2) at (0.0, 0.9) {$c_{\mathrm{mol}}^2$};
\node[conceptmol] (m3) at (1.0, 0.9) {$c_{\mathrm{mol}}^3$};

\node[conceptseq] (s1) at (3.0, 0.9) {$c_{\mathrm{seq}}^1$};
\node[conceptseq] (s2) at (4.0, 0.9) {$c_{\mathrm{seq}}^2$};
\node[conceptseq] (s3) at (5.0, 0.9) {$c_{\mathrm{seq}}^3$};

% Pair indicators: pocket-molecule
\node[indone]  (pm1) at (-3.0, -0.6) {$1$};
\node[indone]  (pm2) at (-1.8, -0.6) {$1$};
\node[indzero] (pm3) at (-0.6, -0.6) {$0$};

% Pair indicators: molecule-sequence
\node[indone]  (ms1) at (0.6, -0.6) {$1$};
\node[indone]  (ms2) at (1.8, -0.6) {$1$};
\node[indzero] (ms3) at (3.0, -0.6) {$0$};

% Final binding variable
\node[indone, minimum width=9mm] (bind) at (0.0, -2.0) {$B{=}1$};

% Active connections: pocket-molecule
\draw[active] (p1) -- (pm1);
\draw[active] (m1) -- (pm1);
\draw[active] (p2) -- (pm2);
\draw[active] (m2) -- (pm2);
\draw[blocked] (p3) -- (pm3);
\draw[blocked] (m3) -- (pm3);

% Active connections: molecule-sequence
\draw[active] (m1) -- (ms1);
\draw[active] (s1) -- (ms1);
\draw[active] (m2) -- (ms2);
\draw[active] (s2) -- (ms2);
\draw[blocked] (m3) -- (ms3);
\draw[blocked] (s3) -- (ms3);

% Aggregation to final binding
\draw[agg] (pm1) -- (bind);
\draw[agg] (pm2) -- (bind);
\draw[agg] (pm3) -- (bind);
\draw[agg] (ms1) -- (bind);
\draw[agg] (ms2) -- (bind);
\draw[agg] (ms3) -- (bind);
\end{tikzpicture}
\caption{Three-view concept-level schematic of \ours{}. The inputs are pocket structure, molecule structure, and protein sequence. Each view is decomposed into modality-specific concepts $c_{\mathrm{poc}}$, $c_{\mathrm{mol}}$, and $c_{\mathrm{seq}}$. Local binary indicators summarize whether a concept pair is active ($1$) or blocked ($0$). We show two interaction branches, pocket-molecule and molecule-sequence, which are then aggregated into the final pair-level binding event $B$. Green solid arrows denote active local interactions, red dashed arrows denote blocked interactions, and gray arrows denote aggregation to the overall binding variable.}
\label{fig:overview}
\end{figure}
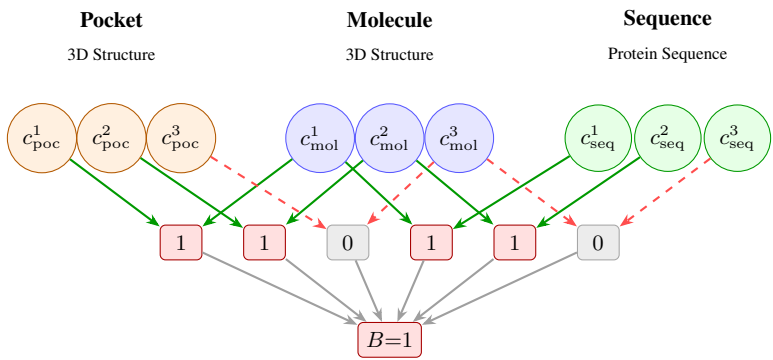

\subsection{Relation Between the Two-Modal Theory and the Three-Modal Implementation}
\label{app:two_three_impl}

\rev{The theory in Sections~\ref{sec:theory} and~\ref{app:theory} is deliberately written for the simplest protein-molecule setting with two modalities: one molecule view and one protein view. The implementation, however, uses three modalities: molecule structure, pocket structure, and protein sequence. This does not contradict the theory; rather, it instantiates the same two-modal idea twice. Concretely, the model contains two retrieval branches, molecule-pocket and molecule-sequence, with score matrices $\mathbf U^{\mathrm{poc,mol}}$ and $\mathbf U^{\mathrm{seq,mol}}$. Each branch is a protein-molecule pairing of exactly the kind analyzed in the theory, and each branch is equipped with its own constrained interaction bottleneck. The final objective simply sums the two pairwise retrieval losses and their associated regularization terms.}

\rev{For this reason, the main theoretical claims still apply at the level at which they were intended: they describe why dense holistic cross-modal interaction is non-identifiable, why sparse interaction can be component-wise identifiable, and why low-rank interaction can be subspace-identifiable for each protein-molecule branch individually.} \rev{The three-stream architecture should therefore be read as a multi-view extension of the same pairwise principle rather than as a different causal problem. The pocket branch and the sequence branch provide two complementary protein views, but neither changes the core structural role of the protein-molecule interaction analyzed in the theory. Put differently, the implementation is richer than the theory in terms of modalities, but not different in its fundamental interaction pattern: it is still a sum of two protein-molecule pairing problems, each governed by the same sparse/low-rank inductive bias.}

\subsection{\texorpdfstring{\ours{}$_{\mathrm{SP}}$}{CausalBind-SP}: Sparse Concept-Pair Implementation}
\label{app:impl_sp}

\rev{\ours{}$_{\mathrm{SP}}$ is the primary implementation used in the main text. It keeps an explicit $K\times K$ concept-pair mask for each protein-molecule branch and therefore most directly matches the sparse interaction form in Theorem~\ref{thm:main:ident}. The learned dense mask can also enter the low-rank regime analyzed in Theorem~\ref{thm:main:lowrank}, as diagnosed by the singular-value analyses in App.~\ref{app:v2_mask_stats}. Unless stated otherwise, \ours{} in the main paper refers to \ours{}$_{\mathrm{SP}}$.}

\subsubsection{Technical Supplement for the \texorpdfstring{\ours{}$_{\mathrm{SP}}$}{CausalBind-SP} Implementation}

\rev{This subsection complements Section~\ref{sec:method} at the design level. The emphasis here is on how the abstract ingredients in the main text are instantiated in practice: how concept extraction is implemented, how the sparse mask is parameterized and applied, and how the inherited geometric losses fit into the final objective.}

\paragraph{\textit{ConceptExtractor} and \textit{SparseMask}.} \rev{For the \emph{main sparse concept-pair operating point} used in the paper, we set $K{=}128$, $d_c{=}1024$, $L{=}4$ Perceiver layers, and $8$ attention heads per layer; the strongest alternative concept-level settings in the sweeps vary one of these quantities at a time. The default mask type is \texttt{tanh\_plus\_1}, the hard-gating threshold is $0.5$, and the effective mask is}
\[
\mathbf Z=\tanh(\widehat{\mathcal{M}})+1,
\]
\rev{followed by hard thresholding so entries at or below the threshold are set to zero exactly in the forward pass. Gradients through the hard gate use a straight-through convention, and the sparsity penalty is applied to the continuous activated masks $\mathbf Z$ rather than to the raw mask logits. The main-text operating point uses $\lambda_{\mathrm{spa}}=10^{-4}$, while the sweep additionally studies $\lambda\in\{0,10^{-6},10^{-5},10^{-4},10^{-3},10^{-2},10^{-1},1\}$. In all cases, the branchwise penalties for the molecule-pocket and molecule-sequence masks are summed and then scaled by the chosen $\lambda_{\mathrm{spa}}$.}

\paragraph{Key implementation equations.} \rev{For each branch $(a,b)\in\{(\mathrm{mol},\mathrm{poc}),(\mathrm{mol},\mathrm{seq})\}$, the sparse concept-pair implementation can be summarized as}
\begin{equation}
{\color{black}
\begin{aligned}
\mathbf C_v &= F_v(x_v)\in\mathbb R^{K\times d_c}, \\
\mathbf Z_{a,b}^{\mathrm{SP}}
&=\tanh(\widehat{\mathcal M}_{a,b})+1,\qquad
\mathcal M_{a,b}^{\mathrm{SP}}
= H_{\tau}^{\mathrm{ST}}\!\left(\mathbf Z_{a,b}^{\mathrm{SP}}\right),
\qquad
\widehat{\mathcal M}_{a,b}\in\mathbb R^{K\times K}, \\
\mathbf A_{a,b}[i,j]
&= \phi(\mathbf C_a[i,:],\mathbf C_b[j,:])
=\langle \mathbf C_a[i,:],\mathbf C_b[j,:]\rangle, \\
u_{a,b}^{\mathrm{SP}}
&= \mathrm{Agg}\!\left(\mathcal M_{a,b}^{\mathrm{SP}}\odot \mathbf A_{a,b}\right).
\end{aligned}
}
\label{eq:app-sp-forward}
\end{equation}
\rev{Here $F_v$ denotes the backbone-plus-concept-extractor map for modality $v$, $\tau=0.5$ is the forward hard-gating threshold, $\mathrm{Agg}(\mathbf B)=K^{-1}\sum_{i,j}B_{ij}$ in the implementation, and $u_{a,b}^{\mathrm{SP}}$ is the branch-level score entering the contrastive objective.}

\paragraph{Detailed form of $\mathcal{L}_{\mathrm{cone}}$ and $\mathcal{L}_{\mathrm{reg}}$.} In the main text we keep the HypSeek-inherited geometric terms abstract to avoid overloading Section~\ref{sec:method}. Here we make them explicit. Consider assay $i$ with pocket embedding $\mathbf h_{\mathrm{poc},i}$ and candidate molecule embeddings $\{\mathbf h_{\mathrm{mol},ij}\}_{j\in\mathcal L_i}$, where $\mathcal L_i$ is the set of molecules paired with assay $i$ in the training batch. Let $N_\ell=\sum_i |\mathcal L_i|$ be the total number of pocket-molecule pairs. Following HypSeek~\citep{hypseek2025}, the Lorentz distance between pocket $i$ and molecule $j$ is denoted $d_{ij}$, the exterior angle at the pocket is $\phi_{ij}$, and the pocket half-aperture is $\omega_i$. Affinity values are discretized into buckets $b_{ij}$, from which the bucket-specific radial and angular limits are derived as
\[
r_{ij}=r_0+b_{ij}\Delta r,\qquad
\eta_{ij}=\eta_0-b_{ij}\Delta\eta.
\]
The cone-hierarchy loss then penalizes violations in both radial distance and angular admissibility:
\begin{equation}
\begin{aligned}
L_{\mathrm{rad}}
&=
\frac{1}{\sqrt{N_\ell}}
\sum_{i,j}
\max\!\bigl(d_{ij}-r_{ij},\,0\bigr),
\\
L_{\mathrm{ang}}
&=
\frac{1}{\sqrt{N_\ell}}
\sum_{i,j}
\max\!\bigl(\phi_{ij}-\eta_{ij}\omega_i,\,0\bigr),
\\
\mathcal L_{\mathrm{cone}}
&=
\lambda_{\mathrm{rad}}L_{\mathrm{rad}}
+
\lambda_{\mathrm{ang}}L_{\mathrm{ang}}.
\end{aligned}
\end{equation}
\rev{The $\sqrt{N_\ell}$ normalization is inherited unchanged from HypSeek; it is a code-level loss scaling convention rather than a new statistical average introduced by \ours{}.}

The additional regularization term in our implementation inherits two further components from HypSeek. Concretely, the $\mathcal L_{\mathrm{reg}}$ appearing in the main text is exactly the weighted sum of the two terms below:
\[
\mathcal L_{\mathrm{reg}}
=
\lambda_{\mathrm{angu}}R_{\mathrm{ang}}
+
\lambda_{\mathrm{het}}R_{\mathrm{het}}.
\]
The first term is an \emph{angular margin regularizer}. It prevents trivial angular collapse by requiring each molecule to remain at least a fixed margin $m_{\mathrm{margin}}>0$ beyond its cone boundary:
\begin{equation}
R_{\mathrm{ang}}
=
\frac{1}{\sqrt{N_\ell}}
\sum_{i,j}
\max\!\bigl(\phi_{ij}-\eta_{ij}\omega_i+m_{\mathrm{margin}},\,0\bigr).
\end{equation}
Here the same $\phi_{ij}$, $\eta_{ij}$, and $\omega_i$ as above are reused; the only new notation is the constant angular margin $m_{\mathrm{margin}}$, which is set to $0.15$ in the current code.

The second term is an \emph{intra-assay heterogeneity regularizer}. Its role is to encourage the model to assign larger probability mass to the stronger active molecules within the same assay, rather than treating all actives as equally informative. For assay $i$, let $L_i = |\mathcal L_i|$ be the number of paired molecules, and let $d_i=(d_{i1},\ldots,d_{iL_i})$ be their pocket-molecule distances. We define the within-assay rank of molecule $j$ by
\[
\operatorname{rank}_{ij}
=
1+\sum_{j'=1}^{L_i}\mathbf{1}[d_{ij'}<d_{ij}],
\]
so smaller distance implies better rank. HypSeek then assigns a rank-based weight
\[
w_{ij}
=
\exp\!\Bigl(-\beta \,\frac{\operatorname{rank}_{ij}-1}{L_i}\Bigr),
\]
where $\beta$ controls how strongly higher-ranked molecules are emphasized (the implementation uses $\beta=80.5$). Let $\ell_{ij}$ denote the molecule-retrieval logit for assay $i$ and molecule $j$ within that assay-specific subset, and define the within-assay softmax probability
\[
p_{ij}
=
\frac{\exp(\ell_{ij})}{\sum_{j'=1}^{L_i}\exp(\ell_{ij'})}.
\]
Let $v_{ij}$ be the activity or affinity annotation of molecule $j$ in assay $i$, and let $v_{\mathrm{th}}$ be the HypSeek preprocessing threshold used to mark active compounds. The implementation keeps only molecules with $v_{ij}<v_{\mathrm{th}}$ (equivalently, the positive subset under the assay's activity convention). Writing $C$ for the number of assays that contain at least one such active molecule, the heterogeneity regularizer is
\begin{equation}
R_{\mathrm{het}}
=
\frac{1}{\max(C,1)}
\sum_i
\sum_{\{j:\,v_{ij}<v_{\mathrm{th}}\}}
-\,w_{ij}\log p_{ij}.
\end{equation}
Intuitively, $R_{\mathrm{het}}$ is a rank-weighted within-assay negative log-likelihood: stronger actives receive larger weights, so the model is encouraged to separate fine-grained affinity variation inside each pocket-specific candidate set.

\paragraph{Relation to our Section~\ref{sec:method} notation.} In the current paper, the retrieval loss in Section~\ref{subsec:training} is written in terms of batch-level score matrices $\mathbf U^{\mathrm{poc,mol}}$ and $\mathbf U^{\mathrm{seq,mol}}$. The HypSeek terms above operate on the same training pairs after they have been lifted into the hyperbolic branch-specific geometry. In other words, $\mathcal L_{\mathrm{cone}}$ and $\mathcal L_{\mathrm{reg}}$ do not replace the retrieval objective; they act as geometric regularizers on top of it. This is why the main text groups them separately from $\mathcal L_{\mathrm{cont}}$, while the appendix spells out the exact radial, angular, and within-assay ranking quantities used in the inherited implementation.

\subsubsection{Concrete Parameter Assignments and Experimental Setting}

\rev{We now turn to the code-level side of reproducibility for \ours{}$_{\mathrm{SP}}$: what values are assigned to the key hyperparameters, how the inherited loss weights are instantiated in code, and what training/evaluation settings are shared across the reported runs.}

\paragraph{Training setup and optimization.} We build on the HypSeek code base~\citep{hypseek2025} with pretrained Uni-Mol encoders for molecules and pockets and an ESM-2 encoder for protein sequence. Training follows the same three-branch retrieval pipeline on the \texttt{blended ChEMBL+BindingDB} corpus \rev{($18{,}316$ assays after removing assays that overlap the FEP benchmarks and every assay whose target UniProt ID appears in DUD-E, LIT-PCBA, or DEKOIS 2.0; as in the LigUnity and HypSeek protocol, the PDBBind structural pairs are not filtered by UniProt ID, so some PDBBind complexes share targets with these benchmarks)}, with CASF used for validation when validation-selected protocols are needed. Unless noted otherwise, all reported runs train for $50$ epochs on $4\times$A100-80GB GPUs with per-GPU batch size $24$ (distributed world size $4$, effective global batch size $96$), Adam optimizer, learning rate $10^{-4}$, polynomial decay, warmup ratio $0.06$, weight decay $0$, gradient clipping at $1.0$, and fp16 mixed precision. In our cluster logs and planning estimates, a typical 50-epoch run takes several hours on $4\times$A100, so full sweeps scale roughly linearly with the number of variants. The default seed is $1$; additional seeds are only introduced in the dedicated robustness checks. All compared retrieval runs (HypSeek, concept-level ablations, concept-level sweeps, and embedding-level variants) share this optimization setup unless a subsection explicitly states otherwise.

\paragraph{Inference and storage.} \rev{\ours{} remains a retrieval-style dual-tower model: molecule-side and protein-side concept features are computed independently, and the fixed branch masks are applied before producing score matrices. Thus library molecules can still be precomputed rather than passed through a joint protein-molecule network for every target. The explicit $K\times K$ \ours{}$_{\mathrm{SP}}$ masks add a small fixed branch computation, while \ours{}$_{\mathrm{LR}}$ reduces mask parameters from $K^2$ to $2Kr$ per branch. For billion-scale deployment one would typically combine these representations with approximate-nearest-neighbor indexing or distillation; the present experiments evaluate ranking quality rather than an optimized production index.}

\paragraph{Loss weights and the relation between paper notation and code notation.} In Section~\ref{sec:method}, we write the training objective using a compact notation with a sparsity weight $\lambda_{\mathrm{spa}}$ and a grouped geometric term $\lambda_{\mathrm{reg}}\mathcal{L}_{\mathrm{reg}}$. In the code, however, the inherited HypSeek regularization is implemented through several explicit coefficients rather than through a single scalar $\lambda_{\mathrm{reg}}$. Concretely, the cone-hierarchy term uses $\gamma_{\mathrm{chl}}=0.1$ together with $\lambda_{\mathrm{rad}}=0.5$ and $\lambda_{\mathrm{ang}}=0.5$, while the additional regularization bundle uses $\lambda_{\mathrm{angu}}=0.1$ and $\lambda_{\mathrm{het}}=0.1$. Thus, the paper's $\mathcal{L}_{\mathrm{reg}}$ should be read as a compact umbrella notation for the inherited angular and heterogeneity regularizers, while $\lambda_{\mathrm{reg}}$ is not a single independently tuned scalar in the current implementation. The only new coefficient introduced by \ours{} itself is $\lambda_{\mathrm{spa}}$, whose values are reported explicitly in the corresponding experiment tables.

\paragraph{Evaluation protocol.} We evaluate on DUD-E (102 targets) and LIT-PCBA (15 targets). For each target we rank all molecules in the test set and compute AUROC, BEDROC ($\alpha=80.5$), and EF@1\%. Per-target metrics are macro-averaged across targets.

\paragraph{Evaluation metrics.} \textbf{AUROC} measures overall ranking quality across the full ranked list: it is the probability that a randomly chosen active is ranked above a randomly chosen inactive, so it reflects global discrimination rather than early retrieval. \textbf{BEDROC$_{80.5}$} is an early-enrichment metric that exponentially upweights the top of the ranked list; the parameter $\alpha=80.5$ is the standard virtual-screening setting that places roughly $80\%$ of the weight within the top $\sim 2\%$ of ranks, making the metric highly sensitive to whether true actives appear near the front. \textbf{EF@1\%} (enrichment factor at the top $1\%$) compares the number of actives retrieved in the top $1\%$ of the ranked list against the number expected from random ranking. In practical screening, BEDROC$_{80.5}$ and EF@1\% are often more informative than AUROC because only a small top-ranked subset is typically advanced to experimental follow-up.

\subsection{\texorpdfstring{\ours{}$_{\mathrm{LR}}$}{CausalBind-LR}: Low-Rank Concept-Basis Implementation}
\label{app:impl_lr}

\rev{\ours{}$_{\mathrm{LR}}$ keeps the same concept extractor, retrieval branches, and losses as \ours{}$_{\mathrm{SP}}$, but replaces the unconstrained $K\times K$ mask logits with an explicit rank-$r$ factorization in the concept basis. It is therefore the direct implementation of the low-rank interaction structure in Theorem~\ref{thm:main:lowrank}.}

\subsubsection{Technical Supplement for the \texorpdfstring{\ours{}$_{\mathrm{LR}}$}{CausalBind-LR} Implementation}

\rev{For each branch $(a,b)\in\{(\mathrm{mol},\mathrm{poc}),(\mathrm{mol},\mathrm{seq})\}$, \ours{}$_{\mathrm{LR}}$ parameterizes the raw mask as}
\[
\widehat{\mathcal M}_{a,b} = \mathbf U_{a,b}\mathbf V_{a,b}^{\top},
\qquad
\mathbf U_{a,b},\mathbf V_{a,b}\in\mathbb R^{K\times r},
\]
\rev{and then applies the same activation and hard gate as \ours{}$_{\mathrm{SP}}$:}
\[
\mathbf Z_{a,b}^{\mathrm{LR}}=\tanh(\mathbf U_{a,b}\mathbf V_{a,b}^{\top})+1,\qquad
\mathcal M_{a,b}^{\mathrm{LR}}=H_{\tau}^{\mathrm{ST}}\!\left(\mathbf Z_{a,b}^{\mathrm{LR}}\right).
\]
\rev{The raw mask logit has rank at most $r$ and remains a concept-axis operator: it modulates the $K$ concept tokens while preserving each token's $d_c$-dimensional feature vector. The monotone activation and hard gate can change exact algebraic rank, but the trainable interaction directions are still controlled by the $2Kr$ low-rank factors. This is the key distinction from an embedding-space rank constraint, which would project pooled features into a low-dimensional feature subspace and can create a much stronger information bottleneck.}

\paragraph{Key implementation equations.} \rev{Equivalently, the full \ours{}$_{\mathrm{LR}}$ branch computation is}
\begin{equation}
{\color{black}
\begin{aligned}
\widehat{\mathcal M}_{a,b}^{\mathrm{LR}}
&=\mathbf U_{a,b}\mathbf V_{a,b}^{\top}
=\sum_{q=1}^{r}\mathbf u_{a,b}^{(q)}(\mathbf v_{a,b}^{(q)})^{\top}, \\
\mathbf Z_{a,b}^{\mathrm{LR}}
&=\tanh(\widehat{\mathcal M}_{a,b}^{\mathrm{LR}})+1,\\
\mathcal M_{a,b}^{\mathrm{LR}}
&=H_{\tau}^{\mathrm{ST}}\!\left(\mathbf Z_{a,b}^{\mathrm{LR}}\right), \\
\mathbf A_{a,b}[i,j]
&=\phi(\mathbf C_a[i,:],\mathbf C_b[j,:]), \\
u_{a,b}^{\mathrm{LR}}
&=\mathrm{Agg}\!\left(\mathcal M_{a,b}^{\mathrm{LR}}\odot \mathbf A_{a,b}\right).
\end{aligned}
}
\label{eq:app-lr-forward}
\end{equation}
\rev{Thus \ours{}$_{\mathrm{LR}}$ differs from \ours{}$_{\mathrm{SP}}$ only in the parameterization of the same concept-axis mask: the dense mask logits $\widehat{\mathcal M}_{a,b}$ are replaced by a rank-$r$ factorization. The defining constraint in \ours{}$_{\mathrm{LR}}$ is the rank constraint itself. The default \ours{}$_{\mathrm{LR}}$ setting uses $\lambda_{\mathrm{spa}}^{\mathrm{LR}}=0$, while the optional $+L_{\mathrm{spa}}$ diagnostic adds an $L_1$ penalty on the continuous activated low-rank masks without changing the rank-$r$ logit factorization.}

\paragraph{\rev{Executed loss.}} \rev{The training loss used for the default \ours{}$_{\mathrm{LR}}$ is the same code-level objective as the main implementation, with the low-rank masks substituted into the forward pass and $\lambda_{\mathrm{spa}}^{\mathrm{LR}}=0$. Concretely, the code applies the HypSeek branch retrieval loss to the molecule-pocket and molecule-sequence score matrices and then adds the inherited geometric terms; the optional $+L_{\mathrm{spa}}$ diagnostic additionally includes the mask penalty:}
\begin{equation}
{\color{black}
\begin{aligned}
\mathcal L_{\mathrm{exec}}^{\mathrm{LR}}
&=
\alpha_{\mathrm{poc}}\,
\mathcal L_{\mathrm{hcc}}\!\left(\mathbf U^{\mathrm{poc,mol},\mathrm{LR}}\right)
\;+\;
\alpha_{\mathrm{prot}}\,
\mathcal L_{\mathrm{hcc}}\!\left(\mathbf U^{\mathrm{seq,mol},\mathrm{LR}}\right) \\
&\quad
+\gamma_{\mathrm{chl}}\mathcal L_{\mathrm{cone}}
+\lambda_{\mathrm{angu}}R_{\mathrm{ang}}
+\lambda_{\mathrm{het}}R_{\mathrm{het}}
+\lambda_{\mathrm{spa}}^{\mathrm{LR}}\mathcal L_{\mathrm{spa}}^{\mathrm{LR}},\qquad
\lambda_{\mathrm{spa}}^{\mathrm{LR}}=0\ \text{by default},\\
\mathcal L_{\mathrm{hcc}}(\mathbf U)
&=
\mathcal L_{\mathrm{pocket}}(\mathbf U)
+\mathcal L_{\mathrm{mol}}(\mathbf U)
+\mathcal L_{\mathrm{rank}}(\mathbf U),\\
\mathcal L_{\mathrm{spa}}^{\mathrm{LR}}
&=
\|\mathbf Z_{\mathrm{mol}}^{\mathrm{LR}}\|_1
+\|\mathbf Z_{\mathrm{poc}}^{\mathrm{LR}}\|_1
+\|\mathbf Z_{\mathrm{seq}}^{\mathrm{LR}}\|_1 .
\end{aligned}
}
\label{eq:app-lr-loss}
\end{equation}
\rev{Here $\mathcal L_{\mathrm{hcc}}$ is the actual branch loss returned by \texttt{compute\_hcc\_pair\_official\_style}: a pocket-to-molecule term, a molecule-to-pocket term, and an intra-assay ranking term. The terms $\mathcal L_{\mathrm{cone}}$, $R_{\mathrm{ang}}$, and $R_{\mathrm{het}}$ are exactly the geometric regularizers defined above for \ours{}$_{\mathrm{SP}}$. At the code level, \texttt{SparseMask} modules are attached to the molecule, pocket, and sequence streams; the branch-level notation in Eq.~\eqref{eq:app-lr-forward} is the paper-level interaction view induced by these stream masks.}

\subsubsection{Concrete Parameter Assignments and Experimental Setting}

\rev{Unless otherwise stated, \ours{}$_{\mathrm{LR}}$ uses the same backbone, optimizer, training set, evaluation protocol, and concept-extractor configuration as the main \ours{}$_{\mathrm{SP}}$ run: $K{=}128$, $d_c{=}1024$, $L{=}4$, $8$ attention heads, and threshold $\tau=0.5$. The representative direct low-rank setting uses $r=1$ and $\lambda_{\mathrm{spa}}^{\mathrm{LR}}=0$; rank scans additionally use $r\in\{1,4,8,16,32,64\}$. The optional $+L_{\mathrm{spa}}$ diagnostic uses $\lambda_{\mathrm{spa}}^{\mathrm{LR}}=10^{-4}$. For $K{=}128$, the rank-$1$ mask uses $2Kr=256$ mask parameters per branch, compared with $K^2=16{,}384$ for the dense concept-pair mask.}

\subsection{\texorpdfstring{\ours{}$_{\mathrm{EMB}}$}{CausalBind-EMB}: Embedding-Level Masking}
\label{app:impl_emb}
\label{app:v1_appendix}

\rev{\ours{}$_{\mathrm{EMB}}$ is an embedding-level alternative that applies the sparse/low-rank interaction idea after pooling. It is useful as an empirical robustness check, but it is not the primary theory-aligned method because its mask entries do not correspond to explicit concept pairs.}

\subsubsection{Technical Supplement for the \texorpdfstring{\ours{}$_{\mathrm{EMB}}$}{CausalBind-EMB} Implementation}

\rev{The embedding-level variant places the learnable interaction mask after mean-pooling, rather than directly on the concept axis. Concretely, \ours{}$_{\mathrm{SP}}$ first extracts $K$ causal concepts per modality, applies the mask on the concept axis, and then pools; by contrast, \ours{}$_{\mathrm{EMB}}$ first obtains pooled modality embeddings and then applies a mask $\mathcal{M}\in\mathbb{R}^{d_c\times d_c}$ in the shared embedding space. The two variants therefore share the same high-level retrieval pipeline and the same sparsity-controlled interaction idea, but differ in where the bottleneck is imposed: before pooling for \ours{}$_{\mathrm{SP}}$, versus after pooling for \ours{}$_{\mathrm{EMB}}$.}

\rev{This difference matters conceptually. \ours{}$_{\mathrm{SP}}$ is more directly aligned with the theory in Section~\ref{sec:theory}, because each entry of the mask can be interpreted as a relation between two concepts. \ours{}$_{\mathrm{EMB}}$ is less interpretable in that sense, but it is architecturally simpler and still tests a related hypothesis: dense holistic interaction is unnecessary, while a constrained pooled-space interaction can already improve retrieval.}

\paragraph{Key implementation equations.} \rev{The embedding-level variant uses the same concept extractor, but moves the mask to the pooled feature axis:}
\begin{equation}
{\color{black}
\begin{aligned}
\mathbf C_v &= F_v(x_v)\in\mathbb R^{K\times d_c},
\qquad
\bar{\mathbf h}_v
=\mathrm{Pool}(\mathbf C_v)
=\frac{1}{K}\sum_{i=1}^{K}\mathbf C_v[i,:], \\
\mathbf Z_v^{\mathrm{EMB}}
&=\tanh(\widehat{\mathcal M}_{v})+1,\qquad
\mathcal M_v^{\mathrm{EMB}}
=H_{\tau}^{\mathrm{ST}}\!\left(\mathbf Z_v^{\mathrm{EMB}}\right),
\qquad
\widehat{\mathcal M}_{v}\in\mathbb R^{d_c\times d_c}, \\
\widetilde{\mathbf h}_v
&=\bar{\mathbf h}_v\mathcal M_v^{\mathrm{EMB}}, \qquad
\mathbf z_v=P_v(\widetilde{\mathbf h}_v), \\
u_{a,b}^{\mathrm{EMB}}
&=\operatorname{sim}(\mathbf z_a,\mathbf z_b).
\end{aligned}
}
\label{eq:app-emb-forward}
\end{equation}
\rev{Here $P_v$ is the modality-specific projection head and $\operatorname{sim}(\cdot,\cdot)$ denotes the branch retrieval similarity. Unlike \ours{}$_{\mathrm{SP}}$ and \ours{}$_{\mathrm{LR}}$, the mask in \ours{}$_{\mathrm{EMB}}$ acts on pooled $d_c$-dimensional features, so its entries are not concept-pair interactions.}

\paragraph{\rev{Executed loss.}} \rev{The embedding-level implementation uses the same executed objective, but its sparsity term is applied to the pooled embedding-space masks rather than to concept-axis masks. The score matrices are formed from the embedding-level branch scores $u_{a,b}^{\mathrm{EMB}}$ in Eq.~\eqref{eq:app-emb-forward}:}
\begin{equation}
{\color{black}
\begin{aligned}
\mathcal L_{\mathrm{exec}}^{\mathrm{EMB}}
&=
\alpha_{\mathrm{poc}}\,
\mathcal L_{\mathrm{hcc}}\!\left(\mathbf U^{\mathrm{poc,mol},\mathrm{EMB}}\right)
\;+\;
\alpha_{\mathrm{prot}}\,
\mathcal L_{\mathrm{hcc}}\!\left(\mathbf U^{\mathrm{seq,mol},\mathrm{EMB}}\right) \\
&\quad
+\gamma_{\mathrm{chl}}\mathcal L_{\mathrm{cone}}
+\lambda_{\mathrm{angu}}R_{\mathrm{ang}}
+\lambda_{\mathrm{het}}R_{\mathrm{het}}
+\lambda_{\mathrm{spa}}\mathcal L_{\mathrm{spa}}^{\mathrm{EMB}},\\
\mathcal L_{\mathrm{hcc}}(\mathbf U)
&=
\mathcal L_{\mathrm{pocket}}(\mathbf U)
+\mathcal L_{\mathrm{mol}}(\mathbf U)
+\mathcal L_{\mathrm{rank}}(\mathbf U),\\
\mathcal L_{\mathrm{spa}}^{\mathrm{EMB}}
&=
\|\mathbf Z_{\mathrm{mol}}^{\mathrm{EMB}}\|_1
+\|\mathbf Z_{\mathrm{poc}}^{\mathrm{EMB}}\|_1
+\|\mathbf Z_{\mathrm{seq}}^{\mathrm{EMB}}\|_1 .
\end{aligned}
}
\label{eq:app-emb-loss}
\end{equation}
\rev{Thus the difference between \ours{}$_{\mathrm{EMB}}$ and the concept-basis variants is not the supervision signal: all three use the same retrieval, cone, angular, and heterogeneity terms. The difference is where the sparsity-controlled mask enters the forward computation and therefore what object $\mathcal L_{\mathrm{spa}}$ regularizes.}

\subsubsection{Concrete Parameter Assignments and Experimental Setting}

\rev{The representative \ours{}$_{\mathrm{EMB}}$ setting uses $K{=}64$, $d_c{=}256$, $L{=}4$, $8$ attention heads, mask type \texttt{tanh\_plus\_1}, threshold $\tau=0.5$, and $\lambda_{\mathrm{spa}}=10^{-2}$. It shares the same training corpus, optimizer, GPU setup, checkpoint protocols, and evaluation metrics as \ours{}$_{\mathrm{SP}}$. In this variant the mask has size $d_c\times d_c$ rather than $K\times K$; for the representative setting, this gives $256^2=65{,}536$ mask parameters per stream.}

\subsection{Comparison Across Three Implementations}
\label{app:impl_compare}

\rev{Table~\ref{tab:impl_compare_three} summarizes the relationship among the three implementations. The main text prioritizes \ours{}$_{\mathrm{SP}}$ because it gives the cleanest sparse concept-pair story and keeps the method section compact. \ours{}$_{\mathrm{LR}}$ is the most direct architectural instantiation of the low-rank theorem, while \ours{}$_{\mathrm{EMB}}$ serves as a pooled-basis alternative that is empirically useful but theoretically less direct.}

\begin{table}[h]
\caption{Comparison of the three \ours{} implementations. The main paper presents \ours{}$_{\mathrm{SP}}$; \ours{}$_{\mathrm{LR}}$ and \ours{}$_{\mathrm{EMB}}$ are appendix variants.}
\label{tab:impl_compare_three}
\centering\small
\resizebox{\textwidth}{!}{%
\begin{tabular}{p{2.5cm}|p{3.0cm}|p{3.1cm}|p{3.2cm}|p{3.4cm}}
\toprule
\textbf{Variant} & \textbf{Mask object} & \textbf{Forward bottleneck} & \textbf{Closest theory link} & \textbf{Role in the paper} \\
\midrule
\ours{}$_{\mathrm{SP}}$ & Dense $K\times K$ concept-pair mask with $L_1$ sparsity and hard thresholding. & Concept axis before pooling; each active entry selects a protein-molecule concept pair. & Direct sparse-support form of Theorem~\ref{thm:main:ident}; learned masks can also enter the low-rank regime of Theorem~\ref{thm:main:lowrank}. & Primary method in the main text. \\
\midrule
\ours{}$_{\mathrm{LR}}$ & Rank-$r$ concept-basis mask $\mathbf U\mathbf V^\top$ with $\mathbf U,\mathbf V\in\mathbb R^{K\times r}$. & Concept axis before pooling; rank controls the number of interaction subspaces. & Direct implementation of the low-rank subspace-identifiability structure in Theorem~\ref{thm:main:lowrank}. & Appendix variant for explicit low-rank implementation. \\
\midrule
\ours{}$_{\mathrm{EMB}}$ & Dense $d_c\times d_c$ mask in pooled embedding space. & After mean-pooling; the mask acts on pooled features rather than concept pairs. & Compatible with a compressed interaction view, but not a direct instantiation of concept-pair sparse support. & Appendix alternative and robustness check. \\
\bottomrule
\end{tabular}}
\end{table}

\rev{The key implementation distinction is therefore the axis on which the constraint acts. \ours{}$_{\mathrm{SP}}$ and \ours{}$_{\mathrm{LR}}$ constrain the concept axis and preserve per-concept feature vectors, whereas \ours{}$_{\mathrm{EMB}}$ constrains the pooled feature axis after the $K$ concepts have already been averaged. This is why \ours{}$_{\mathrm{SP}}$ and \ours{}$_{\mathrm{LR}}$ have a cleaner theory-method correspondence, while \ours{}$_{\mathrm{EMB}}$ is best read as an auxiliary empirical variant.}

\section{Experimental Analysis}
\label{app:analysis}

\rev{This section revisits the main-text experiments with more explanation and details. It should be read as an extension of the main experimental section rather than as a separate experimental theme. The goal here is to unpack the evidence behind the main claims, make the evaluation protocol fully explicit, and separate several empirical questions that are necessarily compressed in the main paper.}

\rev{The subsection layout follows the three implementations introduced in App.~\ref{app:implementation}. App.~\ref{app:sp_results} collects the results for \ours{}$_{\mathrm{SP}}$, the sparse concept-pair implementation used in the main text. App.~\ref{app:lr_results} reports \ours{}$_{\mathrm{LR}}$, the explicit low-rank concept-basis implementation. App.~\ref{app:emb_results} reports \ours{}$_{\mathrm{EMB}}$, the embedding-level alternative. We then synthesize the three implementations and close with the theory--practice discussion.}

\subsection{\texorpdfstring{\ours{}$_{\mathrm{SP}}$}{CausalBind-SP}: Sparse Concept-Pair Results}
\label{app:sp_results}

\subsubsection{Extended Analyses for the Main Experiments}
\label{app:full_results}

\paragraph{Why does causal concept extraction help?} The main paper argues that dense holistic alignment is too coarse a structural prior for protein-molecule binding, because binding is typically driven by a small number of local compatibilities rather than by uniform all-to-all agreement. The role of causal concept extraction is to operationalize this inductive bias. The concept extractor encourages each modality to decompose into a smaller set of binding-relevant factors, and the sparse interaction mask then decides which cross-modal pairs matter for prediction. In the scratch-training setting, this provides a strong bias toward disentangled structure. In the pretrained setting considered in our main experiments, it plays a more selective role: rather than learning chemistry and protein semantics from nothing, the module filters a pretrained representation and suppresses features that are predictive but not genuinely binding-relevant.

\paragraph{Why do the gains concentrate on early enrichment?} Across both experimental settings, the improvements in BEDROC$_{80.5}$ and EF@1\% are often larger than the improvements in AUROC. This pattern is important rather than incidental. AUROC measures global ranking quality over the entire candidate list, whereas BEDROC$_{80.5}$ and EF@1\% emphasize whether true actives are pushed into the very top of the ranked list, which is the practically important regime in virtual screening. Our interpretation is that the sparse interaction prior is especially helpful when the model must decide which few candidates deserve the highest confidence, rather than when it is merely asked to produce a smooth global ordering.

\paragraph{Why is LIT-PCBA the stricter test?} The larger improvements on LIT-PCBA than on DUD-E suggest that the causal prior is most useful when shortcut features become unreliable. DUD-E remains a useful benchmark because of its scale and broad target coverage, but LIT-PCBA is deliberately constructed to reduce property-based biases and therefore makes it harder for a model to succeed by exploiting easy decoy artifacts. When improvements persist on LIT-PCBA, especially in early-enrichment metrics, they are more plausibly reflecting better binding discrimination rather than better exploitation of dataset quirks. This is why the appendix repeatedly treats LIT-PCBA as the main stress test for whether the sparse interaction prior is doing the intended job.

\subsubsection{Main Comparison: Three-Protocol Evaluation}
\label{app:baseline_3prot}

The main text reports the \texttt{last} checkpoint for all methods because, in this cross-dataset generalization setting, it is the cleanest fixed-endpoint protocol and avoids tuning checkpoint selection to the test benchmark. Here we additionally report two validation-selected alternatives, \texttt{best\_auc} and \texttt{best\_bedroc}, to show how much the main conclusions depend on checkpoint choice. To keep the comparison unambiguous, Table~\ref{tab:three_protocol_eval} uses a single seed throughout and does not mix in additional exploratory variants. The three methods shown are the two strongest retrieval baselines under our pipeline, HypSeek and LigUnity, together with our main concept-level configuration ($K{=}128$, $d_c{=}1024$, $\lambda{=}10^{-4}$, $L{=}4$)\rev{ and its atom-attentive version \ours{}$_{\mathrm{ATOM}}$ under the same configuration}.

\begin{table}[t]
\caption{Three-protocol evaluation for main comparisons. The main text reports the fixed-endpoint \texttt{last} protocol; \texttt{best\_auc} and \texttt{best\_bedroc} are complementary validation-selected references. All rows are seed-1 results.}
\label{tab:three_protocol_eval}
\centering\small
\resizebox{\textwidth}{!}{%
\begin{tabular}{l|l|ccc|ccc}
\toprule
Method & Protocol & \multicolumn{3}{c|}{DUD-E} & \multicolumn{3}{c}{LIT-PCBA} \\
 & & AUROC & BEDROC$_{80.5}$ & EF@1\% & AUROC & BEDROC$_{80.5}$ & EF@1\% \\
\midrule
\multirow{3}{*}{LigUnity} & last & 0.897 & 0.674 & 44.20 & 0.599 & 0.075 & 6.50 \\
 & best\_auc & 0.915 & 0.716 & 46.42 & 0.581 & 0.089 & 8.15 \\
 & best\_bedroc & 0.896 & 0.671 & 44.00 & 0.599 & 0.076 & 6.38 \\
\midrule
\multirow{3}{*}{HypSeek} & last & 0.909 & 0.605 & 37.83 & 0.603 & 0.061 & 4.62 \\
 & best\_auc & 0.913 & 0.629 & 39.63 & 0.608 & 0.063 & 5.13 \\
 & best\_bedroc & 0.923 & 0.676 & 42.71 & 0.607 & 0.063 & 4.56 \\
\midrule
\multirow{3}{*}{\ours{}$_{\mathrm{SP}}$} & last & 0.935 & 0.744 & 47.69 & 0.639 & \textbf{0.097} & \textbf{8.56} \\
 & best\_auc & 0.934 & 0.722 & 46.41 & 0.628 & 0.074 & 5.84 \\
 & best\_bedroc & 0.935 & 0.747 & 48.24 & \textbf{0.643} & 0.094 & 8.42 \\
\midrule
\multirow{3}{*}{\rev{\ours{}$_{\mathrm{ATOM}}$}} & \rev{last} & \rev{0.943} & \rev{\textbf{0.768}} & \rev{\textbf{49.66}} & \rev{0.621} & \rev{0.079} & \rev{6.25} \\
 & \rev{best\_auc} & \rev{\textbf{0.949}} & \rev{0.749} & \rev{47.51} & \rev{0.623} & \rev{0.070} & \rev{6.02} \\
 & \rev{best\_bedroc} & \rev{0.943} & \rev{0.768} & \rev{49.55} & \rev{0.621} & \rev{0.084} & \rev{6.75} \\
\bottomrule
\end{tabular}}
\end{table}

Two points matter for interpretation. First, the protocol choice does not reverse the main qualitative ranking between HypSeek and \ours{}$_{\mathrm{SP}}$: under all reported checkpoints, our concept-level model remains clearly stronger on LIT-PCBA, which is the more realistic benchmark in our evaluation. The largest gap appears under the fixed-endpoint \texttt{last} protocol used in the main text, where LIT-PCBA BEDROC$_{80.5}$ improves from $0.061$ to $0.097$ and EF@1\% improves from $4.62$ to $8.56$. Second, the three protocols emphasize different aspects of the same model. For \ours{}$_{\mathrm{SP}}$, \texttt{best\_bedroc} slightly improves DUD-E and raises LIT-PCBA AUROC to $0.643$, but \texttt{last} gives the strongest LIT-PCBA early-enrichment profile, which is precisely the regime we care about in practical screening. The \texttt{best\_auc} checkpoint is noticeably weaker on LIT-PCBA for our model, showing that checkpoint selection by a validation AUC criterion does not necessarily align with the practical objective of top-ranked hit discovery.

LigUnity shows a useful contrast. On DUD-E, validation-selected checkpoints improve over its \texttt{last} endpoint, especially under \texttt{best\_auc}. On LIT-PCBA, however, the picture is mixed: \texttt{best\_auc} improves early enrichment over its \texttt{last} result but lowers AUROC, while \texttt{best\_bedroc} stays close to the fixed-endpoint baseline. This reinforces the same lesson as for \ours{}$_{\mathrm{SP}}$: checkpoint selection changes the metric emphasis, but it does not overturn the main comparison. Across all three protocols, \ours{}$_{\mathrm{SP}}$ remains stronger than LigUnity on the LIT-PCBA metrics emphasized in the main text. \rev{\ours{}$_{\mathrm{ATOM}}$ gives the strongest DUD-E results under every protocol, whereas its LIT-PCBA early enrichment stays below that of \ours{}$_{\mathrm{SP}}$ and varies with checkpoint choice ($6.02$--$6.75$ EF@1\%), mirroring the pattern of the main comparison.}

\subsubsection{Ablation and Diagnostic Study: Detailed Analysis}
\label{app:ablation_analysis}

The ablation and diagnostic table in the main text is intended to answer a mechanistic question rather than to produce a second leaderboard: \emph{which structural biases are responsible for the improvement, and how should we interpret the fact that different constraints matter differently across benchmarks?} This distinction is important because a component ablation need not improve every metric monotonically in order to be scientifically informative. In particular, the LIT-PCBA rows should be read as evidence about \emph{which inductive bias is doing what}, not as evidence that each added module must dominate under every protocol and every metric.

\rev{Table~\ref{tab:ablation_results} should be read in two parts: the \ours{}$_{\mathrm{SP}}$ rows are direct component ablations, whereas the \ours{}$_{\mathrm{LR}}$ and \ours{}$_{\mathrm{EMB}}$ rows are diagnostics for where and how rank compression is imposed. Reading the rows from top to bottom reveals the following:}

\rev{\textbf{(1) HypSeek and LigUnity establish the dense-alignment baseline.} These two rows show what current retrieval systems can do without the explicit concept-level interaction structure introduced in \ours{}. They provide the reference point for asking whether constrained concept-level interaction improves over dense holistic retrieval.}

\rev{\textbf{(2) The full model (\ours{}$_{\mathrm{SP}}$) gives the strongest overall balance.} This row is the best summary of our main claim. On DUD-E it is the clear winner on all three metrics. On LIT-PCBA it gives the best AUROC and remains strong on BEDROC$_{80.5}$ and EF@1\%, so the full architecture gives the best cross-benchmark tradeoff.}

\rev{\textbf{(3) Removing only $L_{\mathrm{spa}}$ isolates an unguided learned mask.} The cross-modal mask is still present, but it is no longer coupled to an explicit sparsity objective. The resulting drop suggests that the mask is most useful when regularized toward a simple interaction structure.}

\rev{\textbf{(4) The LR and EMB diagnostics separate two kinds of compression.} In \ours{}$_{\mathrm{LR}}$, both rows keep the rank-one concept-axis factorization; adding the optional $L_{\mathrm{spa}}$ penalty changes little, indicating that the low-rank parameterization itself is the operative constraint. In \ours{}$_{\mathrm{EMB}}$, however, the $\mathbf M_{\mathrm{emb}}^{r=1}$ row imposes rank one after pooling and therefore becomes a much stronger information bottleneck.}

\paragraph{\rev{Main takeaway from the ablation.}} \rev{The cleanest reading is that the constrained concept-level contrastive route is the essential structural change. The concept extractor and mask help because they regularize how contrastive learning sees the two modalities, and the comparison with \ours{}$_{\mathrm{LR}}$ shows that explicit low-rank concept-axis compression can provide a similar structural bias.}

\paragraph{\rev{Connection to the later theory--practice analysis.}} \rev{This reading is consistent with the mask-structure results. The useful empirical regime is not aggressively hard-sparse, but moderately regularized and low-rank. The ablation therefore supports the same calibrated view as the theory--practice analysis: structural constraints matter, but their benefit comes from shaping the interaction into a simple reusable form rather than from adding parameters.}

\subsubsection{\texorpdfstring{Hyperparameter Study on $K$ and $d_c$}{Hyperparameter Study on K and dc}}
\label{app:v2_kd_3prot}

The next question is whether the selected concept count $K$ and concept dimension $d_c$ are robust design choices or just narrow local optima. We therefore report the full protocol-resolved scans for both quantities. These tables are more informative than the single operating-point curves in the main text because they make two things visible at once: where performance is consistently strong, and where apparent improvements depend strongly on the checkpoint-selection rule.

\begin{table}[t]
\caption{\ours{}$_{\mathrm{SP}}$ concept-level $K$-scan (fix $d_c{=}512$), seed=1, all three protocols. $K{=}256$ underperforms consistently.}
\label{tab:v2_Kscan_3prot}
\centering\small
\resizebox{0.92\textwidth}{!}{%
\begin{tabular}{c|l|ccc|ccc}
\toprule
$K$ & Protocol & \multicolumn{3}{c|}{DUD-E} & \multicolumn{3}{c}{LIT-PCBA} \\
 & & AUROC & BEDROC$_{80.5}$ & EF@1\% & AUROC & BEDROC$_{80.5}$ & EF@1\% \\
\midrule
\multirow{3}{*}{32}  & last & 0.938 & 0.749 & 48.41 & 0.620 & 0.081 & 6.90 \\
 & best\_auc & 0.939 & 0.750 & 48.38 & 0.625 & 0.083 & 6.85 \\
 & best\_bedroc & 0.938 & 0.750 & 48.57 & 0.621 & 0.081 & 6.90 \\
\midrule
\multirow{3}{*}{64}  & last & 0.930 & 0.724 & 46.55 & 0.637 & 0.085 & 6.41 \\
 & best\_auc & 0.933 & 0.715 & 45.90 & 0.638 & 0.088 & 7.88 \\
 & best\_bedroc & 0.935 & 0.732 & 47.08 & 0.635 & 0.079 & 5.96 \\
\midrule
\multirow{3}{*}{\textbf{128}}  & last & 0.943 & 0.752 & 47.93 & 0.633 & 0.081 & 7.06 \\
 & best\_auc & 0.943 & 0.755 & 48.33 & 0.634 & 0.082 & 7.04 \\
 & best\_bedroc & 0.943 & 0.752 & 47.93 & 0.632 & 0.082 & 7.06 \\
\midrule
\multirow{3}{*}{256} & last & 0.916 & 0.660 & 41.50 & 0.610 & 0.066 & 5.90 \\
 & best\_auc & 0.917 & 0.662 & 41.82 & 0.609 & 0.065 & 5.46 \\
 & best\_bedroc & 0.916 & 0.664 & 42.09 & 0.612 & 0.069 & 6.02 \\
\bottomrule
\end{tabular}}
\end{table}

\begin{table}[t]
\caption{\ours{}$_{\mathrm{SP}}$ concept-level $d_c$-scan (fix $K{=}128$), seed=1, all three protocols. $d_c{=}1024$ gives the strongest LIT-PCBA results.}
\label{tab:v2_Dscan_3prot}
\centering\small
\resizebox{0.92\textwidth}{!}{%
\begin{tabular}{c|l|ccc|ccc}
\toprule
$d_c$ & Protocol & \multicolumn{3}{c|}{DUD-E} & \multicolumn{3}{c}{LIT-PCBA} \\
 & & AUROC & BEDROC$_{80.5}$ & EF@1\% & AUROC & BEDROC$_{80.5}$ & EF@1\% \\
\midrule
\multirow{3}{*}{128}  & last & 0.919 & 0.589 & 35.75 & 0.615 & 0.062 & 5.25 \\
 & best\_auc & 0.922 & 0.629 & 39.43 & 0.630 & 0.060 & 4.75 \\
 & best\_bedroc & 0.920 & 0.592 & 36.03 & 0.620 & 0.062 & 5.35 \\
\midrule
\multirow{3}{*}{256}  & last & 0.912 & 0.600 & 36.93 & 0.614 & 0.062 & 5.68 \\
 & best\_auc & 0.922 & 0.683 & 43.07 & 0.620 & 0.067 & 4.73 \\
 & best\_bedroc & 0.918 & 0.642 & 40.05 & 0.621 & 0.070 & 6.09 \\
\midrule
\multirow{3}{*}{512}  & last & 0.943 & 0.752 & 47.93 & 0.633 & 0.081 & 7.06 \\
 & best\_auc & 0.943 & 0.755 & 48.33 & 0.634 & 0.082 & 7.04 \\
 & best\_bedroc & 0.943 & 0.752 & 47.93 & 0.632 & 0.082 & 7.06 \\
\midrule
\multirow{3}{*}{\textbf{1024}} & last & 0.935 & 0.744 & 47.69 & \textbf{0.639} & \textbf{0.097} & \textbf{8.56} \\
 & best\_auc & 0.934 & 0.722 & 46.41 & 0.628 & 0.074 & 5.84 \\
 & best\_bedroc & 0.935 & 0.747 & 48.24 & 0.643 & 0.094 & 8.42 \\
\bottomrule
\end{tabular}}
\end{table}

\paragraph{Interpretation.} \rev{The concept-level sweep shows a broad stable region rather than a single fragile optimum. In Table~\ref{tab:v2_Kscan_3prot}, $K{=}128$ is the most consistent choice across all three checkpoint protocols, while $K{=}256$ degrades on both DUD-E and LIT-PCBA. In Table~\ref{tab:v2_Dscan_3prot}, increasing $d_c$ from $512$ to $1024$ trades a small amount of DUD-E margin for a clear gain on LIT-PCBA early enrichment, which is why the main text uses the higher-capacity setting.}

\subsubsection{\texorpdfstring{Hyperparameter Study on $\lambda$}{Hyperparameter Study on lambda}}
\label{app:v2_lambda_3prot}

This sweep asks a more specific question than the $K/d_c$ tables above: once the concept budget is fixed at the main operating point ($K{=}128$, $d_c{=}1024$, $L{=}4$), what range of sparsity-loss weights produces useful behavior, and where does the model collapse? We therefore vary the $L_1$ sparsity weight $\lambda$ across eight orders of magnitude and report all three checkpoint protocols. The resulting pattern is useful for interpreting both the main-text hyperparameter plot and the theory--practice discussion that follows in Section~\ref{app:v2_mask_stats}.

\begin{table}[h]
\caption{\ours{}$_{\mathrm{SP}}$ concept-level $\lambda$-scan, seed=1, all three protocols. $K{=}128$, $d_c{=}1024$, $L{=}4$. Rows below the dotted rule show mask collapse under \texttt{last} and \texttt{best\_bedroc}.}
\label{tab:v2_lambda_3prot}
\centering\small
\resizebox{0.95\textwidth}{!}{%
\begin{tabular}{c|l|ccc|ccc}
\toprule
$\lambda$ & Protocol & \multicolumn{3}{c|}{DUD-E} & \multicolumn{3}{c}{LIT-PCBA} \\
 & & AUROC & BEDROC$_{80.5}$ & EF@1\% & AUROC & BEDROC$_{80.5}$ & EF@1\% \\
\midrule
\multirow{3}{*}{$0$}       & last & 0.938 & 0.746 & 48.14 & 0.617 & 0.080 & 6.83 \\
 & best\_auc    & 0.943 & 0.744 & 47.68 & 0.629 & 0.073 & 6.26 \\
 & best\_bedroc & 0.934 & 0.732 & 47.03 & 0.613 & 0.075 & 6.10 \\
\midrule
\multirow{3}{*}{$10^{-6}$} & last & 0.940 & 0.749 & 47.94 & 0.635 & 0.088 & 7.26 \\
 & best\_auc    & 0.941 & 0.750 & 47.75 & 0.628 & 0.086 & 7.41 \\
 & best\_bedroc & 0.941 & 0.755 & 48.38 & 0.631 & 0.086 & 7.67 \\
\midrule
\multirow{3}{*}{$10^{-5}$} & last & 0.938 & 0.735 & 47.09 & 0.607 & 0.078 & 6.29 \\
 & best\_auc    & 0.940 & 0.732 & 47.08 & 0.628 & 0.082 & 6.41 \\
 & best\_bedroc & 0.940 & 0.735 & 47.15 & 0.613 & 0.070 & 5.72 \\
\midrule
\multirow{3}{*}{$\mathbf{10^{-4}}$} & last & 0.935 & 0.744 & 47.69 & \textbf{0.639} & \textbf{0.097} & \textbf{8.56} \\
 & best\_auc    & 0.934 & 0.722 & 46.41 & 0.628 & 0.074 & 5.84 \\
 & best\_bedroc & 0.935 & 0.747 & 48.24 & 0.643 & 0.094 & 8.42 \\
\midrule
\multirow{3}{*}{$10^{-3}$} & last & \textbf{0.943} & \textbf{0.768} & \textbf{49.73} & 0.614 & 0.073 & 5.49 \\
 & best\_auc    & 0.936 & 0.712 & 45.34 & 0.620 & 0.072 & 5.89 \\
 & best\_bedroc & 0.936 & 0.712 & 45.34 & 0.620 & 0.072 & 5.89 \\
\midrule
\multicolumn{8}{c}{\textit{---- mask collapses below (\texttt{last} and \texttt{best\_bedroc} protocols) ----}} \\
\midrule
\multirow{3}{*}{$10^{-2}$} & last         & 0.437 & 0.013 & \phantom{0}0.53 & 0.457 & 0.011 & 0.46 \\
 & best\_auc    & 0.921 & 0.699 & 44.71 & 0.622 & 0.074 & 5.83 \\
 & best\_bedroc & 0.442 & 0.014 & \phantom{0}0.64 & 0.460 & 0.011 & 0.54 \\
\midrule
\multirow{3}{*}{$10^{-1}$} & last         & 0.297 & 0.006 & \phantom{0}0.30 & 0.437 & 0.013 & 0.59 \\
 & best\_auc    & 0.930 & 0.696 & 44.53 & 0.619 & 0.078 & 6.73 \\
 & best\_bedroc & 0.930 & 0.696 & 44.53 & 0.619 & 0.078 & 6.73 \\
\midrule
\multirow{3}{*}{$1$}       & last         & 0.725 & 0.099 & \phantom{0}5.30 & 0.595 & 0.028 & 1.44 \\
 & best\_auc    & 0.936 & 0.721 & 46.22 & 0.642 & 0.072 & 5.87 \\
 & best\_bedroc & 0.930 & 0.697 & 43.99 & 0.635 & 0.064 & 4.92 \\
\bottomrule
\end{tabular}}
\end{table}

\paragraph{Interpretation.} Two regimes are practically useful. $\lambda{=}10^{-4}$ and $\lambda{=}10^{-3}$ both give strong results, but they emphasize different benchmarks: the former is clearly better on LIT-PCBA, while the latter is strongest on DUD-E. By contrast, $\lambda{\leq}10^{-6}$ is almost indistinguishable from turning the sparsity loss off, and $\lambda{\geq}10^{-2}$ pushes the hard gate into collapse under the fixed-endpoint and BEDROC-selected protocols. The \texttt{best\_auc} rows remain more resilient in that high-$\lambda$ region because they often pick an earlier checkpoint, before the mask has fully saturated.

\subsubsection{\texorpdfstring{Hyperparameter Study on $L$}{Hyperparameter Study on L}}
\label{app:v2_layers_3prot}

\noindent This scan isolates a different modeling choice: once the concept count, channel width, and sparsity weight are fixed, how much depth is actually useful in the concept extractor itself? Since the main text argues that the key inductive bias comes from structured cross-modal interaction rather than depth for its own sake, this table is mainly a check that deeper extractors do not hide an unreported gain.

\begin{table}[h]
\caption{\ours{}$_{\mathrm{SP}}$ concept-level Perceiver-depth ($L$) scan, seed=1, all three protocols. $K{=}128$, $d_c{=}1024$, $\lambda{=}10^{-4}$.}
\label{tab:v2_layers_3prot}
\centering\small
\resizebox{\textwidth}{!}{%
\begin{tabular}{c|l|ccc|ccc}
\toprule
$L$ & Protocol & \multicolumn{3}{c|}{DUD-E} & \multicolumn{3}{c}{LIT-PCBA} \\
 & & AUROC & BEDROC$_{80.5}$ & EF@1\% & AUROC & BEDROC$_{80.5}$ & EF@1\% \\
\midrule
\multirow{3}{*}{$\mathbf{2}$} & last         & 0.940 & 0.763 & 49.20 & 0.625 & 0.087 & 6.98 \\
 & best\_auc    & \textbf{0.946} & \textbf{0.776} & \textbf{49.99} & 0.616 & 0.082 & 6.53 \\
 & best\_bedroc & 0.942 & 0.757 & 48.48 & 0.622 & 0.085 & 7.95 \\
\midrule
\multirow{3}{*}{$\mathbf{4}$ (default)} & last & 0.935 & 0.744 & 47.69 & \textbf{0.639} & \textbf{0.097} & \textbf{8.56} \\
 & best\_auc    & 0.934 & 0.722 & 46.41 & 0.628 & 0.074 & 5.84 \\
 & best\_bedroc & 0.935 & 0.747 & 48.24 & 0.643 & 0.094 & 8.42 \\
\midrule
\multirow{3}{*}{$6$} & last & 0.936 & 0.738 & 47.05 & 0.633 & 0.085 & 6.75 \\
 & best\_auc    & 0.937 & 0.729 & 46.70 & 0.634 & 0.086 & 6.89 \\
 & best\_bedroc & 0.937 & 0.733 & 46.79 & 0.629 & 0.086 & 6.93 \\
\midrule
\multirow{3}{*}{$8$} & last & 0.934 & 0.727 & 46.38 & 0.630 & 0.080 & 6.73 \\
 & best\_auc    & 0.938 & 0.724 & 46.32 & 0.622 & 0.072 & 6.05 \\
 & best\_bedroc & 0.931 & 0.702 & 45.03 & 0.638 & 0.079 & 6.38 \\
\bottomrule
\end{tabular}}
\end{table}

\paragraph{Interpretation.} The depth sweep is simple: modest depth is enough. $L{=}2$ is strongest on DUD-E under \texttt{best\_auc}, while $L{=}4$ gives the best LIT-PCBA early enrichment under the two protocols that matter most for our main narrative, \texttt{last} and \texttt{best\_bedroc}. Increasing depth further to $L{=}6$ or $8$ never helps. This is consistent with the paper's central claim that the structural bias comes mainly from the concept factorization and cross-modal mask, not from stacking a deeper extractor.

\subsubsection{Detailed Mask Structural Analysis}
\label{app:v2_mask_stats}

The compact $\lambda$ sweep in Table~\ref{tab:lambda_mask} motivates a more detailed question: does the learned interaction become element-wise sparse, spectrally low-rank, or both? We therefore measure both element-wise sparsity and spectral low-rank structure directly on the learned concept-level masks. Table~\ref{tab:v2_mask_stats} reports code-level stream diagnostics for the molecule, pocket, and sequence mask modules ($K{=}128$), while the main-text Table~\ref{tab:lambda_mask} reports compact paired-branch statistics for the two protein-molecule retrieval branches. In the notation of Section~\ref{sec:theory}, Condition~2.i corresponds to element-wise sparse support, whereas Condition~3.i corresponds to low-rank or mode-compressed interaction.

\paragraph{Measurement.} Let $\mathcal{M}_\text{eff} = \sigma(\mathcal{M}) = \tanh(\mathcal{M}_\text{raw}) + 1 \in (0, 2)$. The concept-level forward pass applies a hard gate: entries with $\mathcal{M}_\text{eff} \leq 0.5$ are zeroed exactly. The $L_1$ mass and spectral ranks below are measured on the continuous pre-gate $\mathcal{M}_\text{eff}$, while the zero counts are measured after the forward threshold. We report:

\begin{itemize}[leftmargin=1.5em, itemsep=2pt]
\item \textbf{Element-wise sparsity} (corresponds to Condition~2.i): fraction of entries with $\mathcal{M}_\text{eff} < 0.5$, i.e.\ the number of hard zeros after the forward gate.
\item \textbf{Spectral effective rank} (corresponds to Condition~3.i): the smallest $r$ such that the top-$r$ singular values explain at least $95\%$ or $99\%$ of $\|\mathcal{M}_\text{eff}\|_F^2$.
\item Summary scalars: the average mask value $\overline{\sigma(\mathcal{M})}$ and the total mask mass $L_1=\sum_{ij}\sigma(\mathcal{M})_{ij}$.
\item \rev{\textbf{Antichain violations} (corresponds to the antichain clause of Condition~2.i): on the hard support $\mathbf 1[\mathcal{M}_\text{eff}>0.5]$ of each cross-modal mask (molecule--pocket and molecule--sequence), we compare every pair of distinct row supports and every pair of distinct column supports, and count a violation whenever one support is contained in the other (equal supports also count). Each branch contributes $\binom{128}{2}=8{,}128$ row pairs and $8{,}128$ column pairs, giving $32{,}512$ tests in total. This statistic is reported in the last column of Table~\ref{tab:lambda_mask}.}
\end{itemize}

\rev{The antichain statistic sharpens the two-regime interpretation. With no hard zeros ($\lambda\le10^{-4}$), all row and column supports equal the full index set, so every pair violates the antichain condition; these operating points should be interpreted only through Condition~3.i and Theorem~\ref{thm:main:lowrank}. At $\lambda{=}10^{-3}$, every one of the $128$ row and $128$ column supports is unique in both branches, with zero containment relations, so the antichain clause is exactly satisfied by the learned hard support. At $\lambda{=}10^{-2}$ the masks remain exact antichains but retrieval collapses, so the antichain property alone does not explain performance. We note that the $\lambda{=}10^{-3}$ mask remains only moderately sparse ($79.5\%$ active entries) and that Condition~2.ii is not directly measured; the diagnostic validates the antichain subcondition rather than every premise of Theorem~\ref{thm:main:ident}. All statistics are computed exactly on the same single-seed \texttt{last} checkpoints as Table~\ref{tab:lambda_mask}.}

This subsection uses the per-stream quantities $\overline{\sigma(\mathcal{M})}$ and $L_1=\sum_{ij}\sigma(\mathcal{M})_{ij}$ because Table~\ref{tab:v2_mask_stats} reports stream-level mask structure. By contrast, the main-text Table~\ref{tab:lambda_mask} reports the aggregated statistic `$L_1$ avg', i.e., the average absolute mask entry after averaging over all entries and then over the two paired retrieval branches. The two views are consistent but not the same object: one is a stream-level total-mass diagnostic, while the other is a normalized paired-branch summary.

\paragraph{Operational criteria for ``sparse'' and ``low-rank''.} The theory itself does not specify a single numerical threshold such as ``10\% zeros'' or ``rank below 5''; the conditions are structural. Still, for the purpose of interpreting the experiments, it is useful to make the comparison quantitative. Here $K{=}128$, so Condition~3.i and the low-rank regime of Theorem~\ref{thm:main:lowrank} are most plausible when the effective rank needed to explain the mask energy is much smaller than $128$. By that standard, the evidence for low-rank structure is strong: for $\lambda{=}10^{-4}$ we have $r@95\%{=}1$ on all three stream masks and $r@99\% \in \{9,12,14\}$, i.e., only about $7\%$--$11\%$ of the full concept dimension is needed to explain $99\%$ of the squared singular-value energy. For $\lambda{\geq}10^{-3}$ the evidence is even stronger in rank terms, since both $r@95\%$ and $r@99\%$ collapse to $1$.

For Condition~2.i and the sparse-support regime of Theorem~\ref{thm:main:ident}, the practically relevant operating point is $\lambda{=}10^{-3}$. Condition~2.i requires a genuinely sparse support pattern together with the antichain-style separation discussed in the theory. At $\lambda{=}10^{-3}$, the three stream masks show about $9.7\%$--$13.4\%$ exact zeros while still maintaining strong retrieval performance, which is exactly the qualitative behavior expected from the sparse-support theorem. At $\lambda{=}10^{-4}$, by contrast, there are no hard zeros at all, so this setting should be interpreted through the low-rank theorem rather than the sparse one. \rev{In this sense, the experiments support two different theorem-aligned regimes: $\lambda{=}10^{-3}$ for sparse-support identifiability and $\lambda{=}10^{-4}$ for low-rank subspace identifiability.}

\begin{table}[t]
\caption{Mask structural analysis across $\lambda$ for the concept-level model. Per-stream numbers are reported at $K{=}128$ under the \texttt{last} protocol, seed=1. $L_1$ and rank statistics use the continuous pre-gate masks; zeros use the hard-gated forward masks. $\lambda{=}10^{-4}$ remains element-wise dense but has a richer spectral tail, while $\lambda{\geq}10^{-3}$ produces hard zeros and rank-1 masks.}
\label{tab:v2_mask_stats}
\centering\small
\resizebox{0.92\textwidth}{!}{%
\begin{tabular}{c|c|ccc|cc|cc}
\toprule
$\lambda$ & branch & $\overline{\sigma(\mathcal{M})}$ & max $\sigma(\mathcal{M})$ & $L_1$ (sum) & zeros & \%zeros & $r@95\%$ & $r@99\%$ \\
\midrule
\multirow{3}{*}{$0$}       & mol     & 0.920 & 1.311 & 15068 & \phantom{000}0 & \phantom{0}0.0\% & 1 & 5 \\
 & poc  & 0.993 & 1.398 & 16267 & \phantom{000}0 & \phantom{0}0.0\% & 1 & 1 \\
 & seq & 1.023 & 1.374 & 16758 & \phantom{000}0 & \phantom{0}0.0\% & 1 & 1 \\
\midrule
\multirow{3}{*}{$10^{-6}$} & mol     & 0.910 & 1.301 & 14914 & \phantom{000}0 & \phantom{0}0.0\% & 1 & 5 \\
 & poc  & 0.991 & 1.395 & 16228 & \phantom{000}0 & \phantom{0}0.0\% & 1 & 1 \\
 & seq & 1.023 & 1.374 & 16755 & \phantom{000}0 & \phantom{0}0.0\% & 1 & 1 \\
\midrule
\multirow{3}{*}{$10^{-5}$} & mol     & 0.894 & 1.286 & 14648 & \phantom{000}0 & \phantom{0}0.0\% & 1 & 6 \\
 & poc  & 0.968 & 1.377 & 15863 & \phantom{000}0 & \phantom{0}0.0\% & 1 & 2 \\
 & seq & 1.007 & 1.361 & 16505 & \phantom{000}0 & \phantom{0}0.0\% & 1 & 1 \\
\midrule
\multirow{3}{*}{$\mathbf{10^{-4}}$} & mol & 0.756 & 1.148 & 12383 & \phantom{000}0 & \phantom{0}0.0\% & 1 & \textbf{14} \\
 & poc  & 0.781 & 1.200 & 12793 & \phantom{000}0 & \phantom{0}0.0\% & 1 & \textbf{12} \\
 & seq & 0.841 & 1.206 & 13774 & \phantom{000}0 & \phantom{0}0.0\% & 1 & \textbf{9} \\
\midrule
\multirow{3}{*}{$\mathbf{10^{-3}}$} & mol     & 0.512 & 0.773 & \phantom{0}8383 & \phantom{0}1645 & 10.0\% & 1 & 1 \\
 & poc  & 0.512 & 0.797 & \phantom{0}8390 & \phantom{0}2198 & 13.4\% & 1 & 1 \\
 & seq & 0.513 & 0.751 & \phantom{0}8404 & \phantom{0}1593 & \phantom{0}9.7\% & 1 & 1 \\
\midrule
\multirow{3}{*}{$10^{-2}$} & mol     & 0.501 & 0.656 & \phantom{0}8206 & \phantom{0}8677 & 53.0\% & 1 & 1 \\
 & poc  & 0.502 & 0.708 & \phantom{0}8222 & \phantom{0}8448 & 51.6\% & 1 & 1 \\
 & seq & 0.502 & 0.659 & \phantom{0}8222 & \phantom{0}9199 & 56.1\% & 1 & 1 \\
\midrule
\multirow{3}{*}{$10^{-1}$} & mol     & 0.500 & 0.599 & \phantom{0}8195 & 11796 & 72.0\% & 1 & 1 \\
 & poc  & 0.500 & 0.653 & \phantom{0}8199 & 11484 & 70.1\% & 1 & 1 \\
 & seq & 0.500 & 0.610 & \phantom{0}8199 & 11790 & 72.0\% & 1 & 1 \\
\midrule
\multirow{3}{*}{$1$}       & mol     & 0.500 & 0.571 & \phantom{0}8193 & 13233 & 80.8\% & 1 & 1 \\
 & poc  & 0.500 & 0.613 & \phantom{0}8194 & 12994 & 79.3\% & 1 & 1 \\
 & seq & 0.500 & 0.561 & \phantom{0}8194 & 13075 & 79.8\% & 1 & 1 \\
\bottomrule
\end{tabular}}
\end{table}

\paragraph{Findings.} Three patterns matter most. First, the effective masks are low-rank throughout the sweep: $r@95\%$ is always $1$, so the dominant interaction mode is already highly compressed even without explicit sparsity. Second, $\lambda{=}10^{-4}$ does \emph{not} create hard zeros, but it does increase the spectral tail to $r@99\%=9$--$14$, which is exactly the regime where LIT-PCBA performance is strongest. Third, once $\lambda$ reaches $10^{-3}$ and above, hard zeros appear abruptly and the masks collapse back to rank-$1$ even at the 99\% level. In other words, stronger sparsity regularization does not gradually sharpen a useful structure; it eventually removes the richer low-rank structure that seems to matter most in practice.

\paragraph{How the conditions appear in practice.} The current measurements support both theorem regimes, but through different diagnostics. Condition~3.i is fundamentally about dimensionality reduction of the interaction kernel, and singular-value diagnostics measure exactly that, so $\lambda{=}10^{-4}$ gives direct quantitative support for the low-rank theorem. Condition~2.i is about sparse support structure, and the hard-zero statistics at $\lambda{=}10^{-3}$ show that the model has indeed entered that sparse-support regime while retaining strong retrieval quality. The table therefore supports a practical two-regime reading rather than a single-regime one.

\paragraph{Best-effort conclusion about the assumptions.} The most useful reading is therefore complementary rather than asymmetric. For $\lambda{=}10^{-3}$, the experiments support the sparse-support theorem in the regime where meaningful exact zeros appear and performance remains strong. For $\lambda{=}10^{-4}$, the experiments support the low-rank theorem in the regime where the mask remains dense entry-wise but is strongly spectrally compressed. The two theorems are therefore both realized in practice, but at different operating points of the same model.

\paragraph{Implication for theory.} The correspondence to the theory is therefore not ``Theorem~2 fails and only Theorem~3 matters.'' It is that the model exhibits two neighboring theorem-aligned regimes. Around $\lambda{=}10^{-3}$, the learned mask behaves in the sparse-support sense emphasized by Theorem~\ref{thm:main:ident}. Around $\lambda{=}10^{-4}$, it behaves in the low-rank sense emphasized by Theorem~\ref{thm:main:lowrank}. Since both settings remain empirically strong, the theory is useful in exactly the way we want: it explains two distinct but valid operating points of the learned interaction.

\subsubsection{ROC-Enrichment (RE) on DUD-E}
\label{app:re_dude}

ROC-enrichment (RE) is another early-recognition metric, complementary to BEDROC$_{80.5}$ and EF@1\%. At a fixed false-positive-rate threshold $\alpha$, it is defined as $\mathrm{RE}@\alpha=\mathrm{TPR}(\alpha)/\alpha$, so larger values mean that the model retrieves actives much more efficiently than random ranking in the very early part of the list. Following prior reporting practice~\citep{hypseek2025}, we report RE at $0.5\%$, $1\%$, $2\%$, and $5\%$ false-positive-rate thresholds on DUD-E.

We report RE only on DUD-E because it is the more standard setting in prior work and provides a more stable target-averaged estimate than LIT-PCBA. Concretely, DUD-E contains $102$ targets, whereas LIT-PCBA contains only $15$. Table~\ref{tab:re_dude} therefore augments the main-text DUD-E comparison with RE, while keeping AUROC / BEDROC$_{80.5}$ / EF@1\% for context. To stay aligned with the main text, all rows use the same fixed-endpoint \texttt{last} checkpoint protocol after the full training budget.

\begin{table}[t]
\caption{ROC-enrichment (RE) on DUD-E under the \texttt{last} checkpoint protocol. Values are $102$-target means. AUROC / BEDROC$_{80.5}$ / EF@1\% are reproduced from the main DUD-E comparison for context.}
\label{tab:re_dude}
\centering\small
\resizebox{\textwidth}{!}{%
\begin{tabular}{l|ccc|cccc}
\toprule
\multirow{2}{*}{Configuration} & \multicolumn{3}{c|}{Primary metrics} & \multicolumn{4}{c}{ROC-enrichment (RE)} \\
 & AUROC & BEDROC$_{80.5}$ & EF@1\% & @0.5\% & @1\% & @2\% & @5\% \\
\midrule
DrugCLIP~\citep{drugclip2023} & 0.809 & 0.505 & 31.89 & 73.97 & 41.79 & 23.68 & 11.16 \\
LigUnity~\citep{ligunity2024} & 0.897 & 0.674 & 44.20 & 104.69 & 57.47 & 33.76 & 13.88 \\
HypSeek~\citep{hypseek2025} & 0.909 & 0.605 & 37.83 & 95.20 & 54.49 & 30.74 & 14.01 \\
\midrule
\ours{}$_{\mathrm{SP}}$ ($K{=}128, d_c{=}512, L{=}4$) & 0.943 & 0.752 & 47.93 & 127.56 & 69.94 & 37.83 & \textbf{16.39} \\
\textbf{\ours{}$_{\mathrm{SP}}$ ($K{=}128, d_c{=}1024, L{=}4$)} & 0.935 & 0.744 & 47.69 & 128.06 & 69.26 & 37.07 & 15.99 \\
\rowcolor{green!7}\ours{}$_{\mathrm{SP}}$ ($K{=}128, d_c{=}1024, L{=}2$) & 0.940 & 0.763 & 49.20 & \textbf{131.55} & \textbf{71.17} & \textbf{38.04} & 16.30 \\
\bottomrule
\end{tabular}}
\end{table}

\paragraph{Detailed analysis.} Table~\ref{tab:re_dude} supports three conclusions. First, \ours{}$_{\mathrm{SP}}$ improves very-early retrieval not only relative to the HypSeek reproduction, but also relative to the broader retrieval baselines included in the main comparison. All three reported \ours{}$_{\mathrm{SP}}$ settings are clearly stronger than DrugCLIP and LigUnity on AUROC, BEDROC$_{80.5}$, EF@1\%, and all four RE thresholds. This is useful because RE asks a slightly different question from BEDROC$_{80.5}$ and EF@1\%: instead of weighting the top of the ranking or counting actives in the top percentile, it measures how efficiently true actives are recovered at extremely small false-positive rates. The consistent improvement across all of these views strengthens the claim that \ours{}$_{\mathrm{SP}}$ genuinely improves early retrieval on DUD-E.

\paragraph{How to read the \texorpdfstring{\ours{}$_{\mathrm{SP}}$}{CausalBind-SP} rows.} Second, the ordering among the three \ours{}$_{\mathrm{SP}}$ settings is informative. The main-text operating point ($K{=}128$, $d_c{=}1024$, $L{=}4$) is already much stronger than the prior baselines, but the shallower $L{=}2$ variant is slightly better on the RE thresholds and on DUD-E BEDROC$_{80.5}$ / EF@1\%. This matches the broader hyperparameter-study pattern that, on DUD-E alone, a slightly shallower interaction depth can be favorable for very-early retrieval. Third, this does not overturn the model choice in the main text. The paper uses $L{=}4$ because it gives the better cross-benchmark balance once LIT-PCBA is taken into account, whereas RE is reported only on DUD-E as a supplementary diagnostic. We therefore read Table~\ref{tab:re_dude} as supporting evidence for the DUD-E early-recognition story, not as a new primary criterion for model selection.

\subsubsection{\texorpdfstring{Affinity Ranking on Merck-FEP-8}{Affinity Ranking on Merck-FEP-8}}
\label{app:sp_merck_affinity}

\rev{We additionally evaluate whether the main \ours{}$_{\mathrm{SP}}$ checkpoint transfers from virtual screening to affinity ranking. We use HypSeek's \texttt{test\_fep()} pipeline on the 16 FEP benchmark targets, consisting of JACS-8 and Merck-FEP-8; because the current comparison focuses on structurally diverse affinity-ranking targets, Table~\ref{tab:sp_merck_affinity} reports only Merck-FEP-8. The evaluation uses the single-seed \texttt{last} checkpoint protocol for \ours{}$_{\mathrm{SP}}$ and the reproduced HypSeek baseline. Missing \texttt{bench/} CSV files only affect ACC/AUC outputs, which are skipped; Pearson $r$ and Spearman $\rho$ are unaffected.}

\begin{table}[t]
\caption{\rev{Merck-FEP-8 affinity-ranking correlations. Higher Pearson $r$ and Spearman $\rho$ are better. \ours{}$_{\mathrm{SP}}$ uses the main $K{=}128,d_c{=}1024$ configuration under seed 1 and the \texttt{last} checkpoint.}}
\label{tab:sp_merck_affinity}
\centering\small
\begin{tabular}{l|cc}
\toprule
Method & Pearson $r$ & Spearman $\rho$ \\
\midrule
LigUnity & 0.532 & 0.499 \\
HypSeek & 0.423 & 0.390 \\
\rowcolor{green!7}\ours{}$_{\mathrm{SP}}$ & \textbf{0.598} & \textbf{0.548} \\
\bottomrule
\end{tabular}
\end{table}

\rev{On Merck-FEP-8, \ours{}$_{\mathrm{SP}}$ improves over HypSeek by $+0.175$ Pearson $r$ and $+0.158$ Spearman $\rho$, and also exceeds LigUnity by $+0.066$ and $+0.049$, respectively. This suggests that the sparse concept-pair interaction learned for retrieval also transfers to affinity ranking on diverse ligand sets.}

\subsection{\rev{Comparison with S$^2$Drug}}
\label{app:s2drug}

\rev{S$^2$Drug~\citep{s2drug2026} is a concurrent retrieval-based virtual-screening method. Because its implementation was not publicly available, we could not reproduce it under our pipeline; Table~\ref{tab:main_results} therefore reports its paper-reported results, and this comparison should be read as paper-to-paper rather than controlled. S$^2$Drug reports EF@0.5\% as its primary LIT-PCBA enrichment metric, whereas our main table follows HypSeek and reports EF@1\%. Table~\ref{tab:s2drug_ef05} provides the matched EF@0.5\% comparison, computed by the same evaluator from the same checkpoints as Table~\ref{tab:main_results}. At the matched cutoff, \ours{}$_{\mathrm{LR}}$ and \ours{}$_{\mathrm{EMB}}$ exceed the paper-reported S$^2$Drug result, while \ours{}$_{\mathrm{SP}}$ is comparable; at EF@1\%, all three variants exceed it.}

\begin{table}[h]
\caption{\rev{LIT-PCBA EF@0.5\% comparison with the paper-reported S$^2$Drug result.}}
\label{tab:s2drug_ef05}
\centering\small
\begin{tabular}{l|cc}
\toprule
Method & EF@0.5\% & EF@1\% \\
\midrule
S$^2$Drug (paper-reported) & 11.44 & 7.38 \\
\rowcolor{green!7}\ours{}$_{\mathrm{SP}}$ & 10.88 & 8.56 \\
\rowcolor{green!7}\ours{}$_{\mathrm{LR}}$ & 11.59 & 8.84 \\
\rowcolor{green!7}\ours{}$_{\mathrm{EMB}}$ & \textbf{12.20} & \textbf{9.60} \\
\bottomrule
\end{tabular}
\end{table}

\subsection{\rev{Out-of-Distribution Evaluation on DEKOIS 2.0}}
\label{app:dekois_ood}

\rev{To test generalization beyond DUD-E and LIT-PCBA, we perform zero-shot evaluation on the established DEKOIS 2.0 benchmark~\citep{bauer2013dekois} under two complementary settings, both fixed before running any model. \textbf{Target-level OOD} retains the nine DEKOIS targets whose UniProt IDs are absent from all training sources, including PDBBind (11BETAHSD1, ACE2, ADRB2, ALR2, COX1, CYP2A6, INHA, MMP2, and TK). \textbf{Scaffold-level OOD} further retains only test pairs whose Bemis--Murcko scaffolds are absent from all $439{,}639$ unique training molecules; requiring at least five such actives per target leaves six targets, $5{,}385$ pairs, and $98$ actives. Metrics are macro-averaged over targets.}

\begin{table}[h]
\caption{\rev{Zero-shot out-of-distribution evaluation on DEKOIS 2.0. \ours{}$_{\mathrm{ATOM}}$ is the atom-attentive version of \ours{}$_{\mathrm{SP}}$ (Section~\ref{sec:experiments}).}}
\label{tab:dekois_ood}
\centering\small
\setlength{\tabcolsep}{4pt}
\begin{tabular}{l|ccc|ccc}
\toprule
\multirow{2}{*}{Method} & \multicolumn{3}{c|}{Target-level OOD (9 targets)} & \multicolumn{3}{c}{Scaffold-level OOD (6 targets)} \\
& AUROC & BEDROC$_{80.5}$ & EF@1\% & AUROC & BEDROC$_{80.5}$ & EF@1\% \\
\midrule
HypSeek~\citep{hypseek2025} & 0.819 & 0.593 & 20.06 & 0.880 & 0.630 & 44.47 \\
LigUnity~\citep{ligunity2024} & 0.859 & 0.570 & 19.82 & 0.915 & 0.612 & 43.11 \\
\rowcolor{green!7}\ours{}$_{\mathrm{SP}}$ & 0.860 & 0.658 & \textbf{23.00} & 0.945 & 0.712 & 43.66 \\
\rowcolor{green!7}\ours{}$_{\mathrm{LR}}$ & \textbf{0.867} & 0.653 & 22.71 & \textbf{0.949} & 0.716 & 44.70 \\
\rowcolor{green!7}\ours{}$_{\mathrm{ATOM}}$ & 0.837 & \textbf{0.666} & 22.94 & 0.941 & \textbf{0.744} & \textbf{50.77} \\
\bottomrule
\end{tabular}
\end{table}

\rev{Table~\ref{tab:dekois_ood} shows that the structured-interaction models improve AUROC and BEDROC$_{80.5}$ over both retrieval baselines under both shifts, while \ours{}$_{\mathrm{SP}}$ remains comparable to HypSeek at the strictest EF@1\% cutoff of the scaffold-level split. Together with the cross-dataset protocol of the main experiments, these results indicate that the learned interaction structures transfer under explicit target-level and scaffold-level distribution shifts. Under CD-HIT sequence clustering at 40\% identity, only one of the nine targets (ACE2) has its entire cluster absent from training, so these results should not be read as evidence of broad protein-family generalization.}

\subsection{\rev{Interpretability Case Study on THRB}}
\label{app:thrb_case}

\rev{To connect the learned concepts with biologically meaningful binding patterns, we analyze the aligned protein--ligand crystal complex of the Human Thyroid Hormone Receptor $\beta$ (THRB) from DUD-E, following the evidence chain from local structural contacts to learned concepts, to the predicted binding score, and finally to a random-control validation (Table~\ref{tab:thrb_case}, Figure~\ref{fig:thrb_case}). Observed ligand--pocket contacts are heavy-atom pairs within 4\,\AA. Concepts are selected solely by their attention coverage of these contacts, using mean-over-heads attention from the final Perceiver layer, without using the downstream score change. Among the contacting atoms, the four ligand atoms and eight pocket atoms most attended by the selected concepts are then excluded from the concept bottleneck while the Uni-Mol encoder outputs are kept fixed, and the final pocket--ligand score is recomputed. As a control, we repeat this procedure for 64 matched random sets of four ligand and eight pocket atoms.}

\begin{table}[h]
\caption{\rev{Local-to-global attribution on the DUD-E THRB complex.}}
\label{tab:thrb_case}
\centering\small
\footnotesize
\begin{tabularx}{\textwidth}{>{\raggedright\arraybackslash}p{1.6cm}>{\raggedright\arraybackslash}X>{\raggedright\arraybackslash}X>{\raggedright\arraybackslash}p{2.9cm}}
\toprule
Stage & Evidence & Observation & Purpose \\
\midrule
Local structure & Heavy-atom contacts ($\le 4$\,\AA) & 98 contacts involving 27 ligand and 47 pocket atoms & Identify the local binding interface \\
Learned concept & Contact concentration in the most contact-aligned concept pair & 0.0505 vs.\ 0.00881 uniform baseline (5.73$\times$ enrichment) & Test whether contacts concentrate in concepts \\
Predicted score & Mask the associated 4 ligand and 8 pocket atoms (encoder fixed) & Score decreases $1.878\to0.851$ ($-54.7\%$) & Test contribution to the prediction \\
Random control & Mask 64 matched random sets of 4 ligand and 8 pocket atoms & Mean decrease $0.026\pm0.073$; guided decrease $1.027$ (99.2nd percentile) & Exclude an arbitrary masking effect \\
\bottomrule
\end{tabularx}
\end{table}

\begin{figure}[!ht]
\centering
\includegraphics[width=\textwidth]{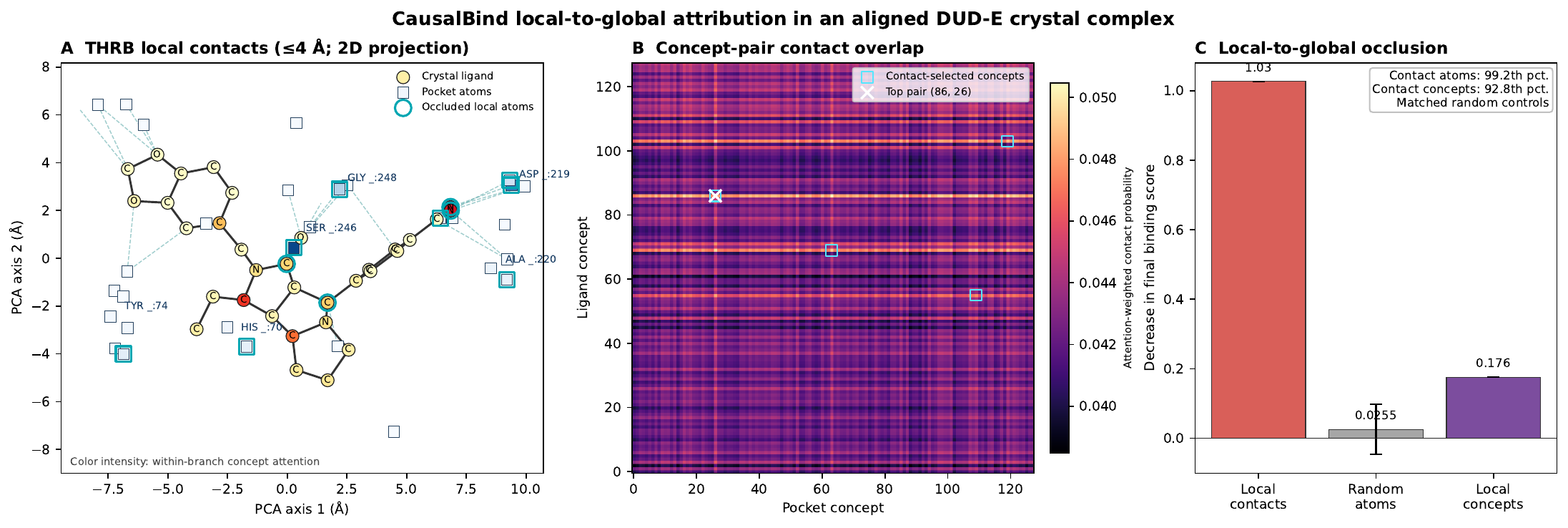}
\caption{\rev{Local-to-global attribution on the DUD-E THRB complex. (A) Ligand--pocket heavy-atom contacts within 4\,\AA\ in a 2D projection; highlighted atoms are those occluded at the concept bottleneck. (B) Attention-weighted contact probability over ligand--pocket concept pairs, with the contact-selected concepts and the top pair marked. (C) Decrease in the final binding score after masking the local contact atoms, matched random atoms, or the contact-selected concepts.}}
\label{fig:thrb_case}
\end{figure}

\rev{The highest-contributing contacts are chemically plausible close polar contacts: ligand N--Asp219 OD1 (2.71\,\AA), N--Asp219 OD2 (2.78\,\AA), N--Gly250 O (2.88\,\AA), and O--Gly248 N (3.04\,\AA). These results show that biologically meaningful local contacts concentrate within a learned concept pair, and that concept-guided masking causes a much larger score decrease than random masking, supporting the intended local-structure-to-concept-to-score mechanism rather than an arbitrary masking effect. This post-hoc diagnostic uses \ours{}$_{\mathrm{ATOM}}$, which exposes the full Uni-Mol atom sequence to the Perceiver; it does not claim that individual concepts carry one-to-one biochemical meanings. Attention and occlusion provide model-level attribution evidence rather than proof of a physical causal mechanism.}

\subsection{\texorpdfstring{\ours{}$_{\mathrm{LR}}$}{CausalBind-LR}: Low-Rank Concept-Basis Results}
\label{app:lr_results}

\rev{This subsection reports the explicit low-rank concept-basis implementation, denoted \ours{}$_{\mathrm{LR}}$ in the paper and recorded as V3 in the experiment log. It uses the same concept extractor and retrieval loss as \ours{}$_{\mathrm{SP}}$, but replaces the dense $K{\times}K$ concept-pair mask with a hard rank-$r$ factorization. Thus it is the most direct empirical test of the subspace-identifiability view in Theorem~\ref{thm:main:lowrank}.}

\subsubsection{Main Rank Scan and Three-Protocol Results}

\rev{For each branch, \ours{}$_{\mathrm{LR}}$ parameterizes the concept-basis mask as}
\[
\sigma(\mathcal{M}_{\mathrm{LR}})=\tanh(\mathbf U\mathbf V^\top)+1,\qquad
\mathcal{M}_{\mathrm{eff}}=H_{0.5}\!\left(\sigma(\mathcal{M}_{\mathrm{LR}})\right),\qquad
\mathbf U,\mathbf V\in\mathbb{R}^{K\times r},
\]
\rev{where $H_{0.5}$ denotes the same threshold gate used by \ours{}$_{\mathrm{SP}}$. Unless otherwise stated, the rank scan uses $K{=}128$, $d_c{=}1024$, $L{=}4$, seed 1, and $\lambda{=}10^{-4}$. The mask-parameter count per branch is $2Kr$, compared with $K^2$ for the dense \ours{}$_{\mathrm{SP}}$ mask.}

\begin{table}[t]
\caption{\rev{\ours{}$_{\mathrm{LR}}$ rank scan, seed=1, all three checkpoint protocols. All rows use $K{=}128$, $d_c{=}1024$, $L{=}4$, and $\lambda{=}10^{-4}$.}}
\label{tab:lr_rank_scan}
\centering\small
\resizebox{\textwidth}{!}{%
\begin{tabular}{c|c|l|ccc|ccc}
\toprule
$r$ & Params/branch & Protocol & \multicolumn{3}{c|}{DUD-E} & \multicolumn{3}{c}{LIT-PCBA} \\
 & & & AUROC & BEDROC$_{80.5}$ & EF@1\% & AUROC & BEDROC$_{80.5}$ & EF@1\% \\
\midrule
\multirow{3}{*}{1} & \multirow{3}{*}{256} & last & 0.935 & 0.720 & 45.86 & 0.630 & 0.087 & 8.41 \\
 & & best\_auc & \textbf{0.946} & \textbf{0.756} & \textbf{48.78} & 0.636 & 0.079 & 6.53 \\
 & & best\_bedroc & 0.935 & 0.724 & 46.28 & 0.632 & \textbf{0.089} & \textbf{8.56} \\
\midrule
\multirow{3}{*}{4} & \multirow{3}{*}{1,024} & last & 0.934 & 0.731 & 47.14 & 0.632 & 0.085 & 7.45 \\
 & & best\_auc & 0.933 & 0.710 & 45.14 & 0.633 & 0.080 & 7.67 \\
 & & best\_bedroc & 0.938 & 0.744 & 47.95 & 0.632 & 0.078 & 6.52 \\
\midrule
\multirow{3}{*}{8} & \multirow{3}{*}{2,048} & last & 0.941 & \textbf{0.756} & 48.67 & 0.626 & 0.085 & 7.60 \\
 & & best\_auc & 0.941 & \textbf{0.756} & 48.38 & 0.634 & 0.079 & 6.61 \\
 & & best\_bedroc & 0.940 & 0.750 & 48.29 & 0.627 & 0.083 & 7.32 \\
\midrule
\multirow{3}{*}{16} & \multirow{3}{*}{4,096} & last & 0.933 & 0.725 & 46.39 & 0.630 & 0.075 & 6.27 \\
 & & best\_auc & 0.928 & 0.693 & 44.07 & 0.616 & 0.075 & 6.01 \\
 & & best\_bedroc & 0.932 & 0.717 & 45.82 & 0.630 & 0.067 & 5.55 \\
\midrule
\multirow{3}{*}{32} & \multirow{3}{*}{8,192} & last & 0.939 & 0.738 & 47.13 & 0.622 & 0.085 & 7.25 \\
 & & best\_auc & 0.941 & 0.741 & 47.44 & 0.624 & 0.083 & 6.75 \\
 & & best\_bedroc & 0.933 & 0.705 & 45.11 & 0.611 & 0.075 & 5.88 \\
\midrule
\multirow{3}{*}{64} & \multirow{3}{*}{16,384} & last & 0.937 & 0.748 & 48.37 & 0.625 & 0.081 & 6.83 \\
 & & best\_auc & 0.936 & 0.731 & 46.88 & 0.622 & 0.074 & 5.58 \\
 & & best\_bedroc & 0.936 & 0.743 & 47.79 & 0.625 & 0.075 & 6.58 \\
\bottomrule
\end{tabular}}
\end{table}

\rev{Table~\ref{tab:lr_rank_scan} shows that very small ranks are already competitive. Rank $r{=}1$ is the most parameter-efficient setting and reaches the best LIT-PCBA EF@1\% under \texttt{best\_bedroc}, while $r{=}8$ gives the strongest DUD-E early-enrichment numbers. Increasing the nominal rank beyond this range does not provide a consistent gain, suggesting that the useful interaction signal is much lower-dimensional than the full $K{\times}K$ mask.}

\subsubsection{Multi-Seed Comparison and Ablation}

\rev{We next compare the most relevant \ours{}$_{\mathrm{LR}}$ settings across three seeds under the fixed-endpoint \texttt{last} protocol. The default row uses $r{=}1$ with $\lambda_{\mathrm{spa}}^{\mathrm{LR}}=0$; the $+L_{\mathrm{spa}}$ row adds the optional mask penalty. The table also includes the main \ours{}$_{\mathrm{SP}}$ multi-seed reference so that the low-rank implementation can be read against the main model directly.}

\begin{table}[t]
\caption{\rev{Multi-seed \texttt{last}-checkpoint comparison for \ours{}$_{\mathrm{LR}}$ and the main \ours{}$_{\mathrm{SP}}$ reference.}}
\label{tab:lr_multiseed}
\centering\small
\resizebox{\textwidth}{!}{%
\begin{tabular}{l|c|ccc|ccc}
\toprule
Method & $n$ & \multicolumn{3}{c|}{DUD-E} & \multicolumn{3}{c}{LIT-PCBA} \\
 & & AUROC & BEDROC$_{80.5}$ & EF@1\% & AUROC & BEDROC$_{80.5}$ & EF@1\% \\
\midrule
\ours{}$_{\mathrm{LR}}$ ($r{=}1$) & 3 & $0.937{\pm}0.004$ & $0.732{\pm}0.011$ & $46.87{\pm}0.91$ & $0.630{\pm}0.006$ & $0.085{\pm}0.006$ & $7.22{\pm}0.31$ \\
\ours{}$_{\mathrm{LR}}$ ($r{=}1$) $+L_{\mathrm{spa}}$ & 3 & $0.935{\pm}0.003$ & $0.727{\pm}0.007$ & $46.39{\pm}0.58$ & $0.623{\pm}0.007$ & $0.083{\pm}0.005$ & $7.51{\pm}0.94$ \\
\ours{}$_{\mathrm{LR}}$ ($r{=}8$) $+L_{\mathrm{spa}}$ & 3 & $0.940{\pm}0.001$ & $0.750{\pm}0.011$ & $48.30{\pm}0.97$ & $0.621{\pm}0.005$ & $0.087{\pm}0.003$ & $7.73{\pm}0.40$ \\
\midrule
\ours{}$_{\mathrm{SP}}$ main & 4 & $0.938{\pm}0.004$ & $0.746{\pm}0.004$ & $47.99{\pm}0.39$ & $0.631{\pm}0.006$ & $0.088{\pm}0.007$ & $7.55{\pm}0.73$ \\
\bottomrule
\end{tabular}}
\end{table}

\rev{The key observation is that \ours{}$_{\mathrm{LR}}$ with $r{=}1$ essentially matches the main dense-mask \ours{}$_{\mathrm{SP}}$ within one standard deviation while using only $256$ mask parameters per branch instead of $16{,}384$. Adding $L_{\mathrm{spa}}$ to the rank-one model changes little, indicating that the hard rank constraint already provides the structural compression that the $L_1$ term induces only indirectly in the dense-mask version. Rank $r{=}8$ improves DUD-E, but the LIT-PCBA early-enrichment metrics are already close to saturated at rank one.}

\subsubsection{Hyperparameter Robustness}

\rev{The rank-one implementation also changes how the model behaves as the concept count grows. In the dense \ours{}$_{\mathrm{SP}}$ mask, increasing $K$ expands the mask as $K^2$ and can over-parameterize the interaction layer. In \ours{}$_{\mathrm{LR}}$ with $r{=}1$, the mask has only $2K$ parameters per branch, so increasing $K$ mainly changes the number of available concepts rather than the rank of their interaction.}

\begin{table}[t]
\caption{\rev{\ours{}$_{\mathrm{LR}}$ $K$-scan with $r{=}1$, seed=1, all three protocols. All rows use $d_c{=}1024$, $L{=}4$, and $\lambda{=}10^{-4}$.}}
\label{tab:lr_k_scan}
\centering\small
\resizebox{\textwidth}{!}{%
\begin{tabular}{c|c|l|ccc|ccc}
\toprule
$K$ & Params/branch & Protocol & \multicolumn{3}{c|}{DUD-E} & \multicolumn{3}{c}{LIT-PCBA} \\
 & & & AUROC & BEDROC$_{80.5}$ & EF@1\% & AUROC & BEDROC$_{80.5}$ & EF@1\% \\
\midrule
\multirow{3}{*}{64} & \multirow{3}{*}{128} & last & 0.935 & 0.704 & 44.87 & 0.632 & \textbf{0.103} & 8.84 \\
 & & best\_auc & 0.934 & 0.702 & 44.63 & \textbf{0.638} & 0.099 & \textbf{9.03} \\
 & & best\_bedroc & 0.936 & 0.716 & 45.92 & 0.631 & 0.095 & 6.54 \\
\midrule
\multirow{3}{*}{128} & \multirow{3}{*}{256} & last & 0.935 & 0.720 & 45.86 & 0.630 & 0.087 & 8.41 \\
 & & best\_auc & \textbf{0.946} & \textbf{0.756} & \textbf{48.78} & 0.636 & 0.079 & 6.53 \\
 & & best\_bedroc & 0.935 & 0.724 & 46.28 & 0.632 & 0.089 & 8.56 \\
\midrule
\multirow{3}{*}{256} & \multirow{3}{*}{512} & last & 0.935 & 0.734 & 47.16 & 0.622 & 0.074 & 5.82 \\
 & & best\_auc & 0.936 & 0.744 & 47.91 & 0.617 & 0.077 & 6.79 \\
 & & best\_bedroc & 0.938 & 0.740 & 47.77 & 0.629 & 0.066 & 5.72 \\
\bottomrule
\end{tabular}}
\end{table}

\rev{Table~\ref{tab:lr_k_scan} shows that the rank-one constraint protects the model from the large-$K$ collapse observed in the dense concept-pair sweep: $K{=}256$ remains strong on DUD-E instead of failing under an oversized $K^2$ mask. At the other end, $K{=}64,r{=}1$ gives the best LIT-PCBA early-enrichment numbers in this scan, suggesting that the property-debiased benchmark can be served by a very small interaction subspace.}

\subsubsection{Mask Analysis}

\rev{Finally, we analyze whether the explicit rank capacity is actually used by the optimizer. The diagnostics use the same thresholded effective mask and singular-value energy criteria as App.~\ref{app:v2_mask_stats}.}

\begin{table}[t]
\caption{\rev{Mask diagnostics for the \ours{}$_{\mathrm{LR}}$ rank scan. The reported branch uses the same protocol as Table~\ref{tab:v2_mask_stats}.}}
\label{tab:lr_mask_rank}
\centering\small
\resizebox{0.86\textwidth}{!}{%
\begin{tabular}{c|cccc|l}
\toprule
$r$ capacity & $\overline{\sigma(\mathcal M)}$ & \%zeros & $r@95\%$ & $r@99\%$ & Interpretation \\
\midrule
1 & 0.86 & 0.0\% & 1 & 1 & rank-one by construction \\
4 & 0.63 & 0.0\% & 1 & 1 & trained solution collapses to rank one \\
8 & 0.54 & 8.9\% & 1 & 1 & some thresholded zeros, still rank one \\
16 & 0.51 & 23.3\% & 1 & 1 & more zeros, same effective rank \\
32 & 0.51 & 34.3\% & 1 & 1 & unused rank capacity \\
64 & 0.50 & 39.7\% & 1 & 1 & unused rank capacity \\
\midrule
\ours{}$_{\mathrm{SP}}$ dense ref. & 0.76 & 0.0\% & 1 & 9--14 & rank-one at 95\%, richer 99\% tail \\
\bottomrule
\end{tabular}}
\end{table}

\rev{Across all nominal capacities $r\in\{1,4,8,16,32,64\}$, the trained \ours{}$_{\mathrm{LR}}$ mask has $r@95\%{=}r@99\%{=}1$. This provides a useful check on the theory--practice story: the low-rank behavior observed in the dense \ours{}$_{\mathrm{SP}}$ mask is not merely an artifact of the $\tanh{+}L_1$ parameterization, because an explicitly rank-constrained model also converges to a rank-one interaction.}

\subsection{\texorpdfstring{\ours{}$_{\mathrm{EMB}}$}{CausalBind-EMB}: Experimental Results and Diagnostics}
\label{app:emb_results}

\rev{The implementation and parameter settings of \ours{}$_{\mathrm{EMB}}$ are given in App.~\ref{app:impl_emb}, and the three-way architectural comparison is summarized in App.~\ref{app:impl_compare}. Here we only report its retrieval behavior and mask diagnostics. This keeps A5 focused on implementation and keeps the present A6 section focused on evidence.}

\subsubsection{\texorpdfstring{Hyperparameter Study on $K$ and $d_c$}{Hyperparameter Study on K and dc}}
\label{app:emb_kd_3prot}

\rev{We first report the embedding-level $K/d_c$ sweep that determines the representative \ours{}$_{\mathrm{EMB}}$ configuration used below. Unlike \ours{}$_{\mathrm{SP}}$ and \ours{}$_{\mathrm{LR}}$, this implementation applies the mask after concept pooling, so these scans should be read as pooled-space sensitivity analyses rather than concept-pair support analyses.}

\begin{table}[h]
\caption{\rev{\ours{}$_{\mathrm{EMB}}$ $K$-scan (fix $d_c{=}256$), seed=1, all three protocols.}}
\label{tab:v1_Kscan_3prot}
\centering\small
\resizebox{0.95\textwidth}{!}{%
\begin{tabular}{c|l|ccc|ccc}
\toprule
$K$ & Protocol & \multicolumn{3}{c|}{DUD-E} & \multicolumn{3}{c}{LIT-PCBA} \\
 & & AUROC & BEDROC$_{80.5}$ & EF@1\% & AUROC & BEDROC$_{80.5}$ & EF@1\% \\
\midrule
\multirow{3}{*}{32} & last & 0.942 & 0.743 & 47.35 & 0.638 & 0.090 & 7.48 \\
 & best\_auc & 0.942 & 0.747 & 47.84 & 0.640 & 0.085 & 6.86 \\
 & best\_bedroc & 0.937 & 0.716 & 45.79 & 0.643 & 0.066 & 4.38 \\
\midrule
\multirow{3}{*}{\textbf{64}} & last & 0.939 & 0.754 & 48.56 & 0.636 & \textbf{0.102} & \textbf{9.60} \\
 & best\_auc & 0.941 & 0.765 & 49.47 & 0.635 & 0.097 & 8.18 \\
 & best\_bedroc & 0.939 & 0.751 & 48.33 & 0.636 & 0.101 & 9.25 \\
\midrule
\multirow{3}{*}{128} & last & 0.935 & 0.715 & 45.62 & 0.623 & 0.082 & 6.50 \\
 & best\_auc & 0.938 & 0.713 & 45.29 & 0.630 & 0.084 & 7.50 \\
 & best\_bedroc & 0.938 & 0.701 & 44.49 & 0.630 & 0.080 & 7.01 \\
\midrule
\multirow{3}{*}{256} & last & 0.936 & 0.737 & 47.55 & 0.624 & 0.080 & 7.08 \\
 & best\_auc & 0.937 & 0.739 & 47.69 & 0.623 & 0.080 & 7.39 \\
 & best\_bedroc & 0.934 & 0.728 & 46.92 & 0.624 & 0.081 & 7.04 \\
\bottomrule
\end{tabular}}
\end{table}

\begin{table}[h]
\caption{\rev{\ours{}$_{\mathrm{EMB}}$ $d_c$-scan (fix $K{=}64$), seed=1, all three protocols. $d_c{=}256$ is strongest on LIT-PCBA early enrichment.}}
\label{tab:v1_Dscan_3prot}
\centering\small
\resizebox{0.95\textwidth}{!}{%
\begin{tabular}{c|l|ccc|ccc}
\toprule
$d_c$ & Protocol & \multicolumn{3}{c|}{DUD-E} & \multicolumn{3}{c}{LIT-PCBA} \\
 & & AUROC & BEDROC$_{80.5}$ & EF@1\% & AUROC & BEDROC$_{80.5}$ & EF@1\% \\
\midrule
\multirow{3}{*}{128}  & last & 0.943 & 0.732 & 47.09 & 0.606 & 0.074 & 5.88 \\
 & best\_auc & 0.944 & 0.726 & 46.59 & 0.617 & 0.075 & 6.97 \\
 & best\_bedroc & 0.943 & 0.738 & 47.49 & 0.602 & 0.071 & 6.03 \\
\midrule
\multirow{3}{*}{\textbf{256}}  & last & 0.939 & 0.754 & 48.56 & 0.636 & \textbf{0.102} & \textbf{9.60} \\
 & best\_auc & 0.941 & 0.765 & 49.47 & 0.635 & 0.097 & 8.18 \\
 & best\_bedroc & 0.939 & 0.751 & 48.33 & 0.636 & 0.101 & 9.25 \\
\midrule
\multirow{3}{*}{512}  & last & 0.942 & 0.750 & 48.32 & 0.624 & 0.079 & 6.45 \\
 & best\_auc & 0.945 & 0.747 & 48.11 & 0.625 & 0.078 & 6.68 \\
 & best\_bedroc & 0.943 & 0.753 & 48.53 & 0.619 & 0.079 & 6.43 \\
\midrule
\multirow{3}{*}{1024} & last & 0.940 & 0.745 & 48.01 & 0.622 & 0.088 & 8.12 \\
 & best\_auc & 0.943 & 0.760 & 48.82 & 0.628 & 0.090 & 8.47 \\
 & best\_bedroc & 0.935 & 0.716 & 45.73 & 0.622 & 0.077 & 6.92 \\
\bottomrule
\end{tabular}}
\end{table}

\rev{The embedding-level sweeps are smoother than the concept-pair sweeps, but the best operating point is still concentrated around moderate concept count and intermediate channel width. We therefore use $K{=}64$, $d_c{=}256$, and $\lambda{=}10^{-2}$ as the representative \ours{}$_{\mathrm{EMB}}$ configuration in Table~\ref{tab:v1_appendix}.}

\subsubsection{Experimental Results}

\begin{table}[t]
\caption{Three-protocol comparison among HypSeek, \ours{}$_{\mathrm{SP}}$, and \ours{}$_{\mathrm{EMB}}$. \ours{}$_{\mathrm{SP}}$ uses $K{=}128$, $d_c{=}1024$, $\lambda{=}10^{-4}$, $L{=}4$, which is the main operating point analyzed in the paper. \ours{}$_{\mathrm{EMB}}$ uses $K{=}64$, $d_c{=}256$, $\lambda{=}10^{-2}$, selected as the strongest setting in the embedding-level sweep. All rows are seed-1 results.}
\label{tab:v1_appendix}
\centering\small
\resizebox{\textwidth}{!}{%
\begin{tabular}{l|l|ccc|ccc}
\toprule
\multirow{2}{*}{\textbf{Model}} & \multirow{2}{*}{Protocol} & \multicolumn{3}{c|}{\textbf{DUD-E}} & \multicolumn{3}{c}{\textbf{LIT-PCBA}} \\
&  & AUROC & BEDROC$_{80.5}$ & EF@1\% & AUROC & BEDROC$_{80.5}$ & EF@1\% \\
\midrule
\multirow{3}{*}{HypSeek~\citep{hypseek2025}} & last & 0.909 & 0.605 & 37.83 & 0.603 & 0.061 & 4.62 \\
 & best\_auc & 0.913 & 0.629 & 39.63 & 0.608 & 0.063 & 5.13 \\
 & best\_bedroc & 0.923 & 0.676 & 42.71 & 0.607 & 0.063 & 4.56 \\
\midrule
\multirow{3}{*}{\ours{}$_{\mathrm{SP}}$} & last & 0.935 & 0.744 & 47.69 & \textbf{0.639} & 0.097 & 8.56 \\
 & best\_auc & 0.934 & 0.722 & 46.41 & 0.628 & 0.074 & 5.84 \\
 & best\_bedroc & 0.935 & 0.747 & 48.24 & \textbf{0.643} & 0.094 & 8.42 \\
\midrule
\multirow{3}{*}{\ours{}$_{\mathrm{EMB}}$} & last & 0.939 & 0.754 & 48.56 & 0.636 & \textbf{0.102} & \textbf{9.60} \\
 & best\_auc & \textbf{0.941} & \textbf{0.765} & \textbf{49.47} & 0.635 & 0.097 & 8.18 \\
 & best\_bedroc & 0.939 & 0.751 & 48.33 & 0.636 & 0.101 & 9.25 \\
\bottomrule
\end{tabular}}
\end{table}

\paragraph{Main comparison.} Table~\ref{tab:v1_appendix} now puts the three relevant models under the same three checkpoint-selection protocols, so the comparison is directly aligned with the earlier protocol analyses in this appendix. The two \ours{} variant rows use the representative configurations selected earlier in their respective sweeps: \ours{}$_{\mathrm{SP}}$ is the paper's main operating point, while \ours{}$_{\mathrm{EMB}}$ is the strongest embedding-level setting identified by the $K/d_c$ scan. Two patterns stand out. First, both \ours{} variants remain clearly stronger than HypSeek across all three protocols, especially on LIT-PCBA early-enrichment metrics. This confirms that the gain is not tied to one specific checkpoint rule. Second, the sparse concept-pair and embedding-level variants emphasize slightly different strengths: \ours{}$_{\mathrm{SP}}$ gives the best LIT-PCBA AUROC under \texttt{last}/\texttt{best\_bedroc}, whereas \ours{}$_{\mathrm{EMB}}$ is especially strong on LIT-PCBA BEDROC$_{80.5}$ / EF@1\% under \texttt{last}, and also gives the best DUD-E numbers under \texttt{best\_auc}.

\paragraph{What this means empirically.} The purpose of this section is not to argue that \ours{}$_{\mathrm{EMB}}$ should replace the main model, but to show that the sparse/low-rank interaction prior is effective beyond a single implementation. \ours{}$_{\mathrm{SP}}$ remains the paper's primary architecture because it offers a cleaner causal interpretation and aligns more directly with the concept-pair theory. \ours{}$_{\mathrm{EMB}}$, however, demonstrates that a simpler pooled-space realization can also produce strong retrieval, which makes the overall design principle more credible.

\subsubsection{Theory Correspondence}
\label{app:v1_v2_stability}

The theory correspondence of the embedding-level variant is simpler, but also narrower, than for the main concept-level model. Because the mask acts on pooled embeddings rather than on explicit concept pairs, it does not naturally match the sparse-support picture behind Condition~2.i and Theorem~\ref{thm:main:ident}. In particular, there is no direct analogue of concept-pair support geometry or the antichain-style condition discussed for the concept-level mask. For that reason, this implementation should not be presented as evidence for the sparse identifiability theorem.

By contrast, the connection to the low-rank theorem is much more natural. The embedding-level mask is exactly a constrained interaction operator in a $d_c$-dimensional shared space, so asking whether useful solutions are low-rank is directly meaningful here. In that sense, the variant is much closer in spirit to Condition~3.i and Theorem~\ref{thm:main:lowrank}: it tests whether a small number of interaction modes in the pooled cross-modal space is already sufficient for strong retrieval.

\begin{table}[h]
\caption{Spectral diagnostics for the embedding-level mask ($K{=}64$, $d_c{=}256$, $\lambda{=}10^{-2}$, seed 1). We report the two branchwise product masks used in the paper's interaction view, evaluated under the \texttt{last} checkpoint.}
\label{tab:v1_mask_rank}
\centering\small
\begin{tabular}{l|ccc}
\toprule
Branch & \%zeros & $r@95\%$ & $r@99\%$ \\
\midrule
mol--pocket & 0.0\% & 1 & 46 \\
mol--sequence & 0.0\% & 1 & 46 \\
\bottomrule
\end{tabular}
\end{table}

Table~\ref{tab:v1_mask_rank} makes the interpretation more concrete. The embedding-level mask exhibits no exact zeros after gating, so it provides no support for the sparse-support assumptions behind Theorem~\ref{thm:main:ident}. At the same time, its spectral profile is clearly compressed: for both branches, a single dominant mode already explains $95\%$ of the squared singular-value energy, and only $46$ modes are needed to explain $99\%$ out of the full ambient dimension $d_c{=}256$. This is not a strict low-rank matrix in the algebraic sense, but it is a strongly \emph{effectively} low-rank interaction operator. We also checked nearby late checkpoints (epochs 40, 45, and 50), and the same $r@95\%{=}1$, $r@99\%{=}46$ pattern persists, so this is not a single-checkpoint artifact.

Across all scanned configurations, the embedding-level variant is also somewhat smoother than the concept-level one: its DUD-E AUROC varies within a relatively narrow band ($0.935$--$0.943$, $n{=}7$), whereas the concept-level sweeps in Tables~\ref{tab:v2_Kscan_3prot} and~\ref{tab:v2_Dscan_3prot} span a wider range. This difference is plausible architecturally. The embedding-level mask acts after pooling and therefore modulates a coarser representation; the concept-level mask, in contrast, directly shapes the fine-grained interaction structure before pooling. Taken together, these observations suggest a clean division of roles: the concept-level model is the better vehicle for the paper's causal and sparse-support story, while the embedding-level alternative provides additional empirical support for the low-rank interaction view.

\subsection{Overall Observations Across Three Implementations}
\label{app:full_sweep}

\paragraph{What remains stable across the sweep.} \rev{The most important observation is that the main qualitative conclusion does not depend on a single hyperparameter choice, checkpoint protocol, or implementation variant. Across the main \ours{}$_{\mathrm{SP}}$ comparison (Table~\ref{tab:three_protocol_eval}), the explicit low-rank results (Tables~\ref{tab:lr_rank_scan}--\ref{tab:lr_mask_rank}), and the embedding-level alternative (Tables~\ref{tab:v1_Kscan_3prot}--\ref{tab:v1_appendix}), constrained cross-modal interaction consistently improves over dense holistic alignment on the retrieval metrics emphasized in the main text, especially early enrichment on LIT-PCBA.}

\paragraph{How the three implementations differ.} \rev{\ours{}$_{\mathrm{SP}}$ is the main paper model: it keeps explicit concept-pair structure and is therefore the cleanest implementation for studying sparse support, mask calibration, and concept-level interpretation. \ours{}$_{\mathrm{LR}}$ is the direct low-rank concept-basis implementation: it matches the main dense-mask model with far fewer mask parameters and shows that rank-one or near-rank-one interaction is often sufficient. \ours{}$_{\mathrm{EMB}}$ is a pooled-space alternative: it can be very strong empirically, especially on LIT-PCBA early enrichment, but its mask acts after concept pooling and therefore has a weaker connection to component-wise concept identifiability.}

\paragraph{What the hyperparameter studies collectively suggest.} \rev{The \ours{}$_{\mathrm{SP}}$ $K/d_c/\lambda/L$ scans point to a coherent operating regime: moderate concept granularity, adequate channel width, and moderate interaction depth work best, while overly large dense concept-pair masks can become brittle. The \ours{}$_{\mathrm{LR}}$ scans sharpen this lesson by showing that explicit rank control protects against the large-$K$ collapse and that increasing rank beyond a small value rarely helps. The \ours{}$_{\mathrm{EMB}}$ scans are smoother, but the strongest setting is still a moderate pooled representation ($K{=}64$, $d_c{=}256$).}

\paragraph{How the sweep connects back to theory.} \rev{Taken together, the three implementations separate form from empirical regime. \ours{}$_{\mathrm{SP}}$ has the sparse concept-pair form of Theorem~\ref{thm:main:ident}, but its strongest $\lambda{=}10^{-4}$ setting behaves more like a low-rank solution; at $\lambda{=}10^{-3}$ it moves closer to sparse support. \ours{}$_{\mathrm{LR}}$ directly realizes the subspace-identifiability form of Theorem~\ref{thm:main:lowrank} and confirms that the learned interaction can be rank one in practice. \ours{}$_{\mathrm{EMB}}$ supports the broader low-rank interaction principle but should be interpreted as a pooled-space variant rather than a component-wise concept-identifiability result.}

\subsection{Discussion on Theory--Practice Gap}
\label{app:theory:gap}

The theoretical results in Section~\ref{sec:theory} are intentionally sharp: Theorem~\ref{thm:main:ident} describes what becomes identifiable under the sparse-support assumptions, while \rev{Theorem~\ref{thm:main:lowrank} gives a subspace-identifiability statement under low-rank interaction.} The experiments in App.~\ref{app:analysis} show that these two theorems are both useful for interpretation, but not in the same way. This subsection makes that distinction explicit by connecting the compact mask statistics in Table~\ref{tab:lambda_mask} with the branchwise diagnostics in Table~\ref{tab:v2_mask_stats}.

\paragraph{\rev{Main takeaway.}} \rev{The central lesson is not that the learned interaction becomes exactly element-wise sparse. Rather, the causal prior pushes training toward structurally simpler cross-modal interactions: stronger sparsification moves \ours{}$_{\mathrm{SP}}$ toward the sparse-support picture, while milder sparsification and explicit \ours{}$_{\mathrm{LR}}$ reveal a low-rank regime. Thus low-rank subspace identifiability is not merely a secondary theoretical case, but an important description of the regime that modern neural training naturally reaches.}

\paragraph{What the experiments do support.} \rev{The experiments support both theorem regimes, but through different implementations and operating points. In Table~\ref{tab:lambda_mask}, `$L_1$ avg' is the average absolute mask entry, \%zeros is the fraction of entries thresholded to exact zero, and $r@95\%$/$r@99\%$ are effective ranks explaining $95\%$/$99\%$ of squared singular-value energy. For \ours{}$_{\mathrm{SP}}$, $\lambda{=}10^{-3}$ yields strong DUD-E retrieval together with meaningful exact sparsity ($20.5\%$ zeros in the compact table), matching the sparse-support picture of Theorem~\ref{thm:main:ident}; the default $\lambda{=}10^{-4}$ setting remains dense entry-wise ($0\%$ zeros) but is strongly low-rank in the dominant energy component ($r@95\%=1$), matching Theorem~\ref{thm:main:lowrank}. For $\lambda\geq 10^{-2}$, the mask enters an excessive hard-zero regime and performance degrades sharply, so sparsity alone is not the desired behavior. In \ours{}$_{\mathrm{LR}}$, the rank constraint is imposed directly and the learned masks still collapse to effective rank one, which independently supports the low-rank/subspace interpretation. In \ours{}$_{\mathrm{EMB}}$, the same low-rank tendency appears after pooling, but this evidence is architectural rather than component-wise.}

\paragraph{Why this is not a contradiction.} \rev{The theorem statements are intentionally sharp, whereas optimization can realize more than one useful structural regime. A dense $K{\times}K$ mask with an $L_1$ penalty can behave as sparse support under stronger regularization or as a low-rank operator under milder regularization; an explicit rank-factorized mask enforces the latter by construction. The empirical message is therefore not that one theorem replaces the other, but that the two theorem regimes identify different useful limits of the same causal interaction prior.}

\paragraph{\rev{Which theorem covers the implementation.}} \rev{One gap between the theory and the implementation deserves to be stated plainly. For simplicity, all three variants use an inner product between concept tokens as the compatibility function $\phi$, which is \emph{bilinear}. Condition~2.ii therefore fails, and by Theorem~\ref{lem:non-id} the latent concepts are not component-wise identifiable in this model class, whatever structure is imposed on $\mathcal M$: there exist reparameterizations that mix concepts while leaving both the mask and the selected distribution unchanged. Theorem~\ref{thm:identifiability} consequently does not certify that the concepts \ours{} learns are the true ones, and an antichain support observed in a learned mask is a structural diagnostic rather than a certificate.}

\rev{The result that does apply is Theorem~\ref{thm:lowrank-recovery}, which places no non-bilinearity requirement on $\phi$: for $\phi(x,y)=xy$ the centred kernel has rank $\mathrm{rank}(\mathcal M)$, so a low-rank mask satisfies Condition~3.i, and the identified objects are the molecule-side and protein-side interaction subspaces. This matches what we observe: at the default $\lambda_{\mathrm{spa}}{=}10^{-4}$ the learned masks have no hard zeros but $r@95\%{=}1$ (Table~\ref{tab:lambda_mask}), and \ours{}$_{\mathrm{LR}}$ imposes the same structure by construction. A second, smaller mismatch remains even there: the theory treats scalar concepts ($d_c{=}1$), whereas the implementation uses $d_c{>}1$ concept tokens, which raises the kernel rank to roughly $\mathrm{rank}(\mathcal M)\cdot d_c$, so Condition~3.i holds for the theoretical model rather than verbatim for the implemented one.}

\rev{We therefore position \ours{} as a causally motivated structural prior rather than as an estimator with a component-wise identifiability certificate. The theory says why unconstrained bilinear interaction cannot identify concepts, which structural restrictions would make them identifiable, and which weaker object (the interaction subspace) remains identifiable in the class we implement. Closing the gap is a concrete next step: replacing the inner product by a non-bilinear compatibility function, for instance $\tanh(\langle\mathbf c_i,\mathbf c_j\rangle)$ or a small shared network, would place the model inside the class of Theorem~\ref{thm:identifiability}, and we leave this to future work.}

\section{Limitations, Broader Impact, and Usage of LLMs.}
\label{app:limitations}

\rev{\paragraph{Summary (moved from the main text).} Our evidence is centered primarily on retrieval-based virtual screening under cross-dataset generalization, especially the two main benchmarks, DUD-E and LIT-PCBA, with an additional Merck-FEP-8 affinity-ranking transfer analysis in App.~\ref{app:sp_merck_affinity} \rev{and target- and scaffold-level out-of-distribution evaluation on DEKOIS 2.0 in App.~\ref{app:dekois_ood}}. This is the right, deliberately focused scope for the main question of the paper, but it does not yet establish uniform gains across \rev{sequence-family or time-stratified splits} or larger benchmark collections\rev{, and our interpretability analysis provides model-level attribution rather than proof of a physical causal mechanism} (see App.~\ref{app:limitations}). As with any virtual-screening method, our model could be misused to prioritize ligands for harmful targets. Although it is a scoring tool rather than a generator, downstream deployment should still be paired with safety and experimental validation filters before any wet-lab follow-up. All source code is publicly available in our GitHub repository to support reproducibility.}

\paragraph{Limitation on scope of the current evidence.} This paper should be read primarily as a study of \emph{retrieval-based virtual screening under cross-dataset generalization}, not as a complete evaluation of protein-molecule modeling in every downstream regime. Accordingly, our main evidence primarily comes from two retrieval benchmarks, DUD-E and LIT-PCBA, while App.~\ref{app:sp_merck_affinity} provides a supplementary Merck-FEP-8 affinity-ranking transfer check; the strongest claims still concern early enrichment rather than absolute affinity calibration. This scope is deliberate: the central question of the paper is whether replacing dense holistic alignment with sparse or low-rank cross-modal interaction improves the high-throughput screening stage. It also means that broader claims should be made cautiously. The current experiments do not yet establish how uniformly the gain transfers across protein families, scaffold splits, time-stratified settings, or much larger benchmark collections.

\paragraph{Broader impact and misuse clarification.} Better virtual screening can reduce the cost of early discovery by improving which candidates are advanced to docking and experimental follow-up, potentially accelerating the identification of useful therapeutics. At the same time, any method that improves prioritization of protein-molecule interactions could in principle be misused to search for ligands against harmful targets. Our model is less directly dual-use than a generative molecular-design system, because it is a scoring and ranking tool rather than a molecule generator, but this does not remove the need for safeguards. In practical use, its predictions should be treated as one stage in a controlled pipeline and paired with downstream safety, ADMET, toxicity, and experimental validation filters before any wet-lab deployment.

\paragraph{Usage of LLMs.} LLMs were used to assist with writing and editing, but they are not the source of any scientific claims or empirical results in this paper. All technical content is derived from our own theoretical analysis and experimental work, and all key numbers are directly traceable to our codebase and training logs. We did not use LLMs to generate scientific hypotheses. Their role was limited to language polishing and formatting assistance, which we carefully reviewed and edited for accuracy and clarity.

\section{Reproducibility and Asset Provenance}
\label{app:reproducibility}

\paragraph{Code basis and development path.} The implementation of \ours{} was developed by extending the retrieval pipeline used in HypSeek and cross-checking evaluation behavior against LigUnity. Concretely, our code keeps the same broad three-stream retrieval setting as HypSeek, while replacing the dense shared-space interaction with the concept extractor and sparse/low-rank cross-modal masking modules described in Section~\ref{sec:method}. The implementation is thus a method-development layer on top of existing retrieval infrastructure, not a from-scratch rewrite.

\paragraph{Main software assets.} The most important inherited software assets are the HypSeek code base, the LigUnity code base used for comparison and protocol alignment, the Uni-Mol/Uni-Core stack used by the molecular and pocket branches, and the ESM-2 protein language model used by the sequence branch. Many inherited Uni-Mol-style source files carry MIT license headers, and the HypSeek repository is MIT-licensed. The LigUnity repository releases code under Apache~2.0 and its data under CC BY-NC~4.0. These upstream assets are credited throughout the paper via citations to HypSeek, LigUnity, Uni-Mol, Uni-Core, and ESM-2.

\paragraph{Data assets.} The training and evaluation data are not newly created by this paper. As described in Section~\ref{sec:experiments} and Appendix~\ref{app:implementation}, training follows the same assay-level setup used by LigUnity and HypSeek, drawing from ChEMBL, BindingDB, and PDBBind\rev{, where ChEMBL/BindingDB assays whose target UniProt IDs appear in the DUD-E, LIT-PCBA, or DEKOIS 2.0 benchmarks are removed}. The main evaluation benchmarks are DUD-E and LIT-PCBA, and the appendix additionally reports transfer results on Merck-FEP-8. All are existing public resources and remain subject to the terms and redistribution rules of their original providers; we do not re-license them.

\paragraph{What is provided for reproducibility.} \rev{The code is publicly available at \url{https://github.com/lokali/CausalBind}. The repository contains the \ours{} implementation (the core model, loss, and task files for all variants), scripts that download the public training data, benchmarks, and pretrained backbones from their original sources, training and evaluation entry points with the exact configurations reported in the paper, and the tools for the mask-structure analysis, the DEKOIS 2.0 out-of-distribution evaluation, and the interpretability case study. We also release the trained checkpoints of the reported models together with a one-command script that reproduces the DUD-E and LIT-PCBA results. Processed datasets are not redistributed; they remain subject to the terms of their original providers.}

\paragraph{Scope of the current license discussion.} This section makes the provenance of the main assets explicit and clarifies which parts are inherited versus newly added by \ours{}. We do not re-publish upstream assets under a unified license: the new code is layered on existing open-source assets, and each inherited component remains under its original license.

\end{document}